\documentclass{article}
\usepackage[table]{xcolor}
\usepackage{microtype}
\usepackage{graphicx}
\usepackage{subfigure}
\usepackage{booktabs} 

\usepackage{hyperref}

\usepackage[accepted]{icml2024}

\usepackage{amsmath}
\usepackage{amssymb}
\usepackage{mathtools}
\usepackage{amsthm}

\usepackage[capitalize,noabbrev]{cleveref}

\definecolor{blue}{HTML}{006699}
\definecolor{green}{HTML}{009966}
\definecolor{orange}{HTML}{f97306}

\theoremstyle{plain}
\newtheorem{theorem}{Theorem}[section]

\theoremstyle{definition}

\theoremstyle{remark}

\usepackage[textsize=tiny]{todonotes}

\usepackage{amsfonts}
\usepackage{wrapfig}

\DeclareMathOperator*{\argmax}{arg\,max}

\newcommand{\norm}[1]{\left\lVert#1\right\rVert}
\usepackage{multirow}

\icmltitlerunning{Mastering Zero-Shot Interactions in Cooperative and Competitive Simultaneous Games}

\begin{document}

\twocolumn[
\icmltitle{Mastering Zero-Shot Interactions in \\Cooperative and Competitive Simultaneous Games}



\icmlsetsymbol{equal}{*}

\begin{icmlauthorlist}
\icmlauthor{Yannik Mahlau}{luh}
\icmlauthor{Frederik Schubert}{luh}
\icmlauthor{Bodo Rosenhahn}{luh}
\end{icmlauthorlist}

\icmlaffiliation{luh}{Department for Information Processing, Leibniz University Hannover, Germany}

\icmlcorrespondingauthor{Yannik Mahlau}{mahlau@tnt.uni-hannover.de}

\icmlkeywords{Machine Learning, ICML}

\vskip 0.3in
]



\printAffiliationsAndNotice{}  

\begin{abstract}
The combination of self-play and planning has achieved great successes in sequential games, for instance in Chess and Go.
However, adapting algorithms such as AlphaZero to simultaneous games poses a new challenge.
In these games, missing information about concurrent actions of other agents is a limiting factor as they may select different Nash equilibria or do not play optimally at all.
Thus, it is vital to model the behavior of the other agents when interacting with them in simultaneous games.
To this end, we propose \textbf{Albatross}: \textbf{A}lphaZero for \textbf{L}earning \textbf{B}ounded-rational \textbf{A}gents and \textbf{T}emperature-based \textbf{R}esponse \textbf{O}ptimization using \textbf{S}imulated \textbf{S}elf-play.
\mbox{Albatross} learns to play the novel equilibrium concept of a \emph{Smooth Best Response Logit Equilibrium (SBRLE)}, which enables cooperation and competition with agents of any playing strength.
We perform an extensive evaluation of Albatross on a set of cooperative and competitive simultaneous perfect-information games.
In contrast to \mbox{AlphaZero}, Albatross is able to exploit weak agents in the competitive game of Battlesnake.
Additionally, it yields an improvement of 37.6\% compared to previous state of the art in the cooperative Overcooked benchmark.

\end{abstract}

\section{Introduction}

\begin{figure}[t]
\vskip 0.2in
\begin{center}
\centerline{\includegraphics[width=\columnwidth]{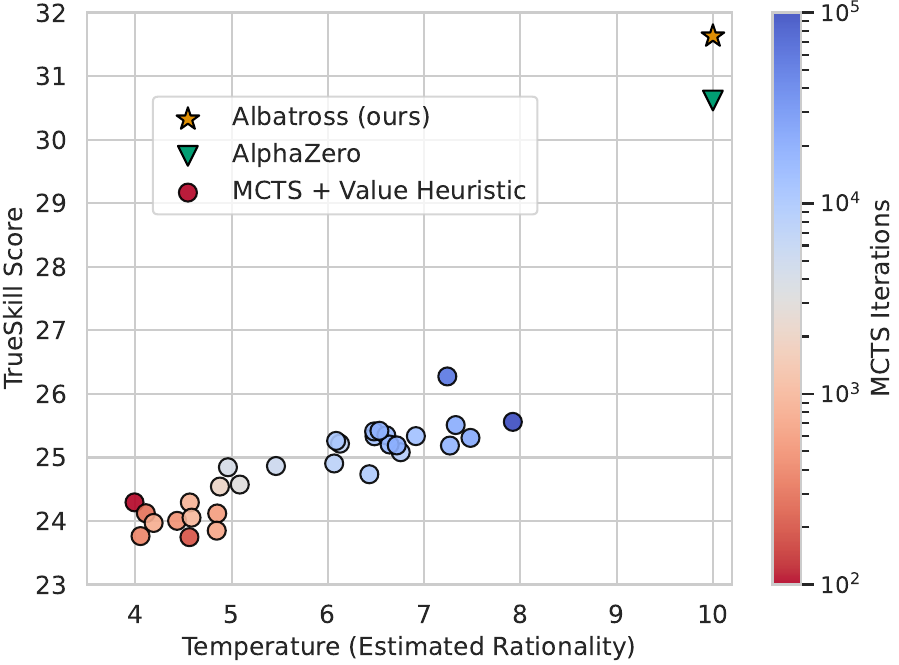}}
\caption{
TrueSkill scores \cite{trueskill} of a tournament consisting of an Albatross agent, Monte-Carlo-Tree-Search (MCTS) baseline agents and an AlphaZero baseline.
Each game takes place in a free for all setting of four agents in the stochastic simultaneous game of Battlesnake.
Albatross estimates the temperature, i.e. rationality, of the baseline agents online using only data from the current game.
A temperature of 0 corresponds to random play and 10 to optimal play if all other agents play optimally as well.
AlphaZero achieves optimal play given that all other agents play optimally (temperature of 10), but fails to adapt to subrational agents.
In contrast, Albatross is able to respond optimally against any combination of weak and strong agents due to its rationality estimation, resulting in a higher TrueSkill tournament score.
}
\label{fig:4nd7_tnmt}
\end{center}
\vskip -0.2in
\end{figure}

Games have been played by humans for centuries, some of the earliest dating back more than 4000 years \cite{senet_game}.
They enable us to measure skill, either in cooperation or competition with other agents.
When facing unseen agents, one has to adapt to their playing style, which is called zero-shot interaction \cite{other_play}.
In sequential games, this only entails finding the best response after observing the action of the other agent.
However, in simultaneous games, where all agents make their move at the same time, it is also necessary predict the next concurrent move.
Therefore, opponent modelling is an important factor for zero-shot interactions in simultaneous games.
Most existing methods learn a policy that performs well with or against as many agents as possible \cite{NEURIPS2021_797134c3, lupu21a, mep, DocKru2017a}.
However, a more scalable approach is the prediction of the other agent's behavior, based on the immediate interactions within a single episode.
Lou et al. \yrcite{pecan} apply this idea and classify other agents into groups of low, medium or high skill.
In contrast, we model their strength as a continuous scalar temperature parameter.
The continuous parametrization is able to accurately model any playing strength.
Additionally, its sparsity enables efficient maximum likelihood estimation (MLE) using only interactions from the current episode.
For this continuous opponent model, we develop the novel concept of a Smooth Best Response Logit equilibrium (SBRLE).
Since the SBRLE is intractable in all but very small games, our method \textbf{Albatross} (\textbf{A}lphaZero for \textbf{L}earning \textbf{B}ounded-rational \textbf{A}gents and \textbf{T}emperature-based \textbf{R}esponse \textbf{O}ptimization using \textbf{S}imulated \textbf{S}elf-play) learns the equilibrium through a combination of self-play and planning akin to AlphaZero \cite{silver_alphazero}.

The adaptive behavior of Albatross allows cooperation with unknown agents, which might not act optimally.
To this end, we evaluate our method in the game of Overcooked \cite{overcooked}, where Albatross is paired with a human behavior cloning agent.
Encoding the estimated rationality via a scalar enables us to dissect the different strategies of Albatross when acting with agents of different strengths.
In competitive games, the opponent modelling allows Albatross to exploit weaker agents as well as compete with strong agents.
In \cref{fig:4nd7_tnmt}, we demonstrate these capabilities in a tournament of Battlesnake \cite{Chung2020BattlesnakeCA}, for which we publish an efficient implementation.
Battlesnake is an extension of the well-studied game Tron \cite{samothrakis_uct_2010, tron_game_length, lanctot13tron, knegt_opponent_2018, rl_tron}, offering additional stochastic environment dynamics for either two or four agents.
In summary, our \textbf{contributions} are the following:
\begin{itemize}
    \item We introduce the novel equilibrium concept of a Smooth Best Response Logit equilibrium for modeling asymmetric bounded rationality with a single rational and an arbitrary number of weak agents.
    \item Our method Albatross learns to approximate an SBRLE using a mixture of self-play and planning, adapting AlphaZero in a principled way based on game theory to zero-shot interactions of simultaneous games.
    \item We empirically evaluate Albatross in several cooperative and competitive games and perform an extensive hyperparameter analysis. Additionally, we qualitatively demonstrate the adaptive behavior of Albatross.
    \item To support reproducibility, all of our code as well as the trained models are open source\footnote{\url{https://github.com/ymahlau/albatross}}.
\end{itemize}









\section{Related Work}
AlphaGo \cite{silver_mastering_2016} was one of the first successful applications of deep reinforcement learning to a complex multi-agent game, namely Go.
It used an actor-critic neural network, which was pre-trained on human games and fine-tuned using self-play.
The pre-training phase was eliminated in its successor AlphaGoZero \cite{silver_mastering_2016} to avoid learning a suboptimal bias from human play.
Independently of AlphaGoZero, a similar system named Expert Iteration was developed for the game of Hex \cite{NIPS2017_d8e1344e}.
Both AlphaGo and AlphaGoZero exploited knowledge about symmetries of the games, which prevented its application to other games.
AlphaZero \cite{silver_alphazero} excludes all game-specific knowledge and therefore is applicable to other games, e.g. Go, Chess and Shogi.
MuZero \cite{schrittwieser_mastering_2020} did not only learn perfect play, but also the environment dynamics through self-play, which makes it suitable for environments with unknown dynamics.
However, for that purpose, MuZero requires much larger compute resources than AlphaZero.


While the mentioned methods are state of the art in sequential competitive games, they do not necessarily work well in zero-shot coordination tasks.
That is, because they do not perform well on all game situations, but rather only on game situations that would arise through self-play \cite{az_attack}.
To achieve good coordination capabilities with a lot of different teammates, Fictitious Co-Play \cite{NEURIPS2021_797134c3} trains against past training checkpoints taken at different time points from multiple self-play agents.
In contrast, Trajectory Diversity \cite{lupu21a} also aims to train a diverse population of agents by regularizing the loss function with Jensen-Shannon Divergence \cite{lin1991divergence}.
Similarly, Maximum Entropy Population-Based training (MEP) \cite{mep} uses population entropy as regularization.
Lou et al. \yrcite{pecan} extend MEP in their framework called Policy Ensemble for Context-Aware Zero-Shot Human-AI Coordination (PECAN) by using a randomly weighted policy ensemble given three groups of agents, ranked based on their self-play performance.
Hidden-Utility Self-Play (HSP) \cite{hsp} trains with agents maximizing a hidden reward function and achieves population diversity by filtering with an event-based metric.
All of the mentioned techniques rely on the idea of training a diverse population to learn to cooperate with as many agents as possible.
However, we believe that explicit opponent modelling based on game theory can produce better cooperation capabilities.

Substantial research regarding opponent modeling is conducted in the field of security games and adversarial domains \cite{Nashed2022ASO}.
Some approaches are based on Variational Autoencoders \cite{vae_opp_model}, Switching tables \cite{Everett2018LearningAN}, model-based approaches using neural networks \cite{knegt_opponent_2018}, recursive reasoning \cite{yu2022modelbased} or expert imitation \cite{DocDoe2017}.
Applications of opponent modeling include real-world situations like wildlife protection against poachers \cite{paws} or protection of security resources \cite{Yang2012ComputingOS}.
However, these methods fail to accurately model behavior in zero-shot interactions, because they require a lot of data to train the opponent model.

\section{Game-Theoretic Background}

\newcommand{\ponecolor}{orange!50}
\newcommand{\ptwocolor}{blue!50}

\begin{figure}[t]
\vskip 0.2in
\begin{center}
\centering
\subfigure{
    \resizebox{0.41\columnwidth}{!}{
        \begin{tabular}{|c|c|c|c|c|c|}
            \hline
             & \multicolumn{5}{c|}{\cellcolor{\ptwocolor}Player 2} \\
            \hline
            \cellcolor{\ponecolor}& Action & \multicolumn{2}{c|}{\cellcolor{\ptwocolor}c} & \multicolumn{2}{c|}{\cellcolor{\ptwocolor}d} \\
            \cline{2-6}
            \cellcolor{\ponecolor}&\cellcolor{\ponecolor}a & \cellcolor{\ponecolor}-4 & \cellcolor{\ptwocolor}4 & \cellcolor{\ponecolor}-7 & \cellcolor{\ptwocolor}7 \\
            \cline{2-6}
            \multirow{-3}{*}{\cellcolor{\ponecolor}\rotatebox[origin=c]{90}{Player 1}}&\cellcolor{\ponecolor}b & \cellcolor{\ponecolor}-6 & \cellcolor{\ptwocolor}6 & \cellcolor{\ponecolor}2 & \cellcolor{\ptwocolor}-2 \\
            \cline{2-5}
            \hline
        \end{tabular}
    }
} \\
\vskip -0.05in
\subfigure{
    \includegraphics[width=0.7\columnwidth]{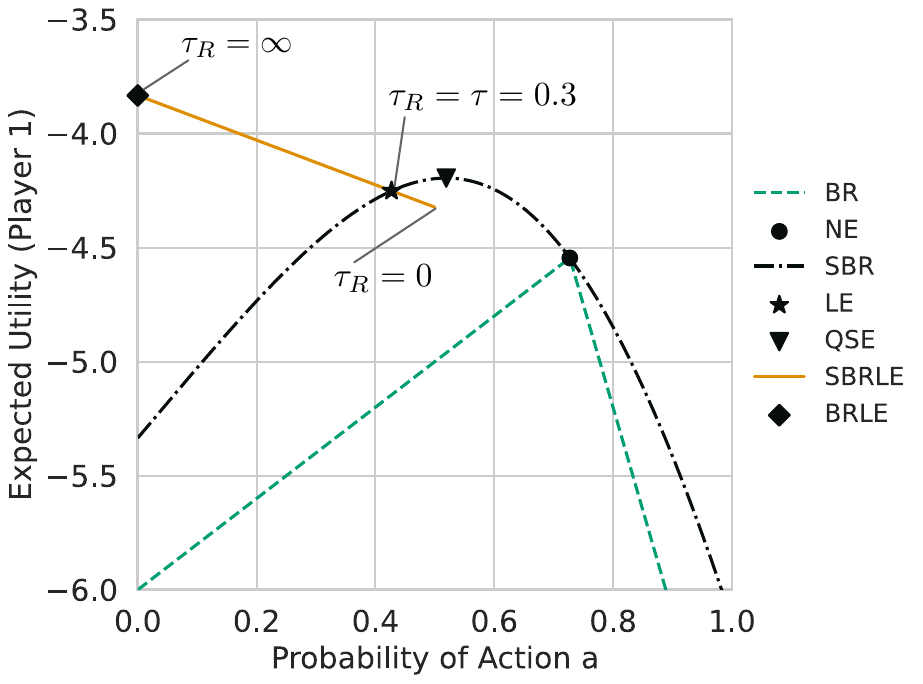}
}
\caption{
Visualization of equilibria in a zero-sum NFG.
Assuming that player 2 plays a best response (BR) to the policy of player 1, the expected utility is lower than under the assumption of a subrational smooth best response $\mathit{SBR}(\cdot, \tau = 0.3)$.
The dotted gray lines denote the expected utility of playing actions a or b against an SBR respectively.
The NE maximizes the expected utility assuming player 2 plays a BR, while the QSE maximizes under the assumption of an SBR.
The SBRLE starts with response temperature $\tau_R$ at a uniform distribution over actions a and b, and ends with $\tau_R \rightarrow \infty$ at the BRLE.
The SBRLE is equal to the Logit equilibrium (LE) if the response temperature $\tau_R$ is equal to the temperature $\tau$ of the LE.
}
\label{fig:equilibria}
\end{center}
\vskip -0.2in
\end{figure}

For the game theoretic background, we adapt the notation of Leyton-Brown and Shoham \yrcite{leyton_brown_essentials_2008} as well as Milec et al. \yrcite{milec_complexity_2021}.
A Normal-form game (NFG) is a tuple $(n, A, u)$, where $n \in \mathbb{N}$ denotes the number of agents, $A = A_1 \times \cdots \times A_n$ the joint action set and $u = (u_1, \ldots, u_n)$ the utility functions.
We call $a = (a_1, \ldots, a_n) \in A$ a joint action and $a_i$ is the action of agent $i$.
Agents are indexed by $i \in \{1, \ldots, n\}$ and $-i$ denotes the set of all agents except $i$.
We abuse notation and use $-i$ as notation for the other agent in games of two players and the set of all agents except $i$ otherwise.
A game of two agents is called zero-sum, if the utility function has the property $u_i(a) = -u_{-i}(a) \, \forall a \in A$.
Similarly, a game is fully cooperative if all agents have the same utility function.
The set of policies $\Delta_i$ (also called mixed strategies) is the set of probability distributions over the action space $A_i$.
A joint policy (also called strategy profile) is a tuple of policies $\pi = (\pi_1, \ldots, \pi_N), \pi_i \in \Delta_i$.
Its utility for agent $i$ is defined as the expected outcome $u_i(\pi) = \sum_{a \in A} u_i(a) \prod_{j=1}^{n} \pi_j(a_j)$.
The best response (BR) of agent $i$ to the policies of other agents $\pi_{-i}$ is a policy $\pi_i \in \mathit{BR}(\pi_{-i})$, where the best response function is defined as the set of policies achieving maximum utility: $\mathit{BR}(\pi_{-i}) = \{\pi^{*}_i \, | \, u_i(\pi^{*}_{i}, \pi_{-i}) \geq u_i(\pi_i, \pi_{-i})\, \forall \pi_i \in \Delta_i \}.$
If all agents play a BR, then their joint policy is called a Nash equilibrium \cite{nash_thesis}.

Nash equilibria inhibit the assumption that all players act rational, which is not realistic in most real-world scenarios.
Following Hofbauer and Sandholm \yrcite{global_convergence}, one can incorporate an error probability by transforming the utility functions $\tilde{u}_i(\pi) = u_i(\pi) + \frac{1}{\tau} \psi(\pi_i)$ using Shannon Entropy as a concave smoothing function $\psi(\pi_i) = \sum_{a_i \in A_i} \pi_i(a_i) \log(\pi_i(a_i))$.
The temperature parameter $\tau$ controls the inverse strength of regularization, i.e. the bounded rationality.
Agents can maximize the transformed utilities by playing a smooth best response (SBR), which is a softmax over the original utilities $\mathit{SBR}(\pi_{-i}, \tau) \propto \exp (\tau \, u_i(\,\cdot\,, \pi_{-i}))$.
To exemplify the effect of the temperature, a temperature of $\tau=0$ corresponds to uniform random play and the SBR approaches a BR with $\tau \rightarrow \infty$.
Akin to BR and Nash equilibria, if all agents play an SBR, then their joint policy is called a Logit equilibrium (LE), which is a subset of the more general Quantal Response equilibrium \cite{qre}.
In some literature, the Logit equilibrium is also called Stochastic equilibrium \cite{MAHER1998539, akamatsu96}.
We compute the LE by Stochastic Fictitious Play \cite{global_convergence}.
In detail, each agent starts from a random policy and iteratively computes the SBR to the other policies.
This process is repeated with a step size annealing according to Nagurney and Zhang \yrcite{nagurney} until a fixed point is reached (see \cref{subsec:le_details} for details).

In the LE, both players play with the same rationality.
But, it is also useful to model asymmetric rationality, i.e. one rational agent playing with one or multiple other imperfect agents.
An equilibrium that models such asymmetry is the Quantal Stackelberg Equilibrium (QSE) \cite{milec_complexity_2021}, which is restricted to games of two agents.
One rational agent maximizes their reward given that the other weak agent plays an SBR: $\pi_i = \argmax u_i(\pi_i, \mathit{SBR}(\pi_i, \tau))$.
However, the weak agent can only play an SBR if they know the optimal policy of the rational agent.
This assumption is true in repeated interactions as one can observe the policy of the other agent, but violated in zero-shot interactions.







\begin{figure*}[!ht]
\vskip 0.2in
\begin{center}
    \subfigure[
    The \textbf{proxy} model evaluates the leaf nodes of its fixed-depth search tree using a learned value function $v_{\theta_{P}}$.
    It approximates a Logit equilibrium (LE) given the temperature $\tau$ with its policy $\pi_{\theta_{P}}$.
    The backup operator is a solver for the LE of a Normal-form game constructed from sibling nodes.
    The temperature $\tau$ is sampled uniformly transformed by a cosine (\cref{subsec:temp_dist}). 
    ]{%
        \includegraphics[width=0.45\textwidth]{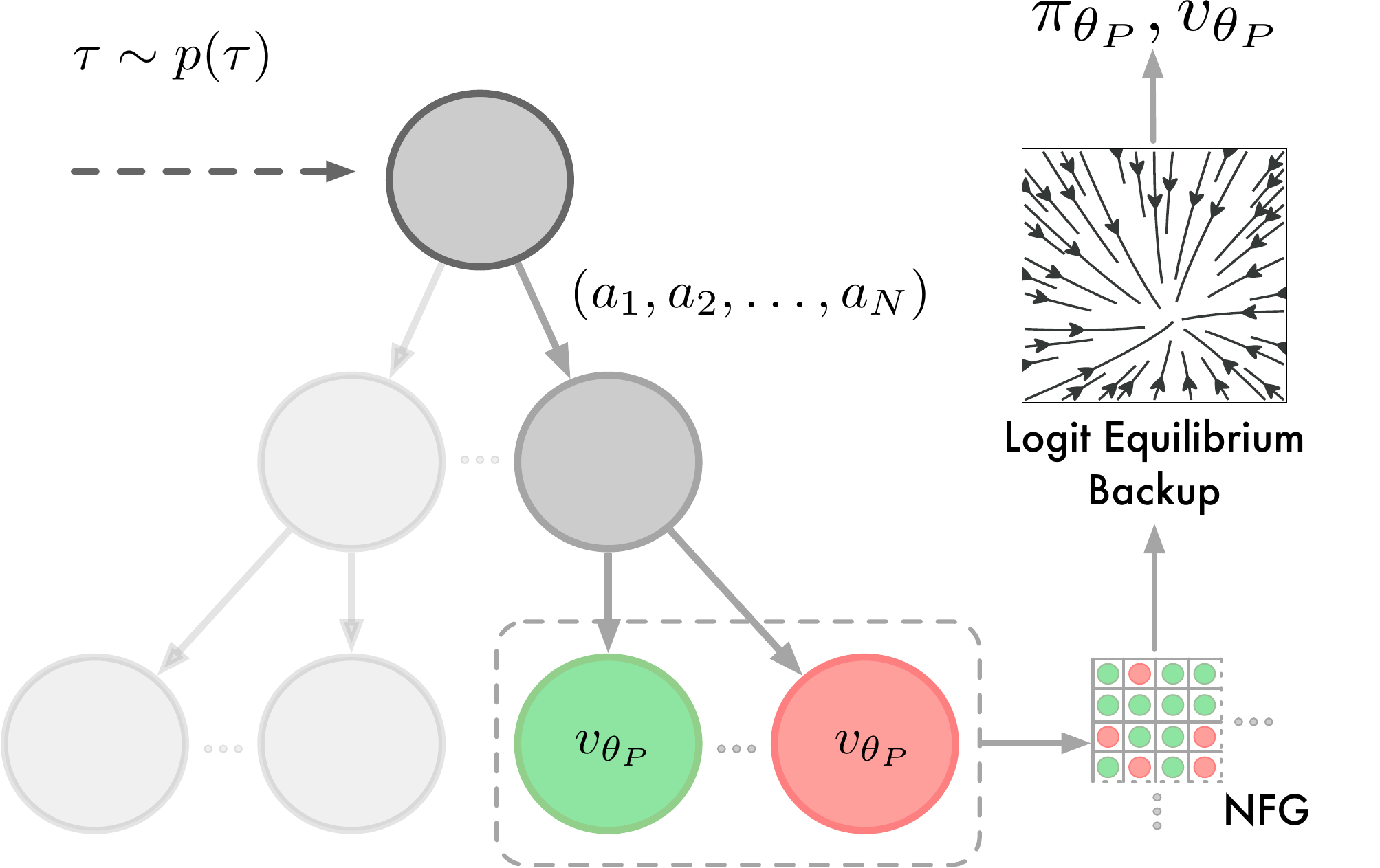}
    }
    \hspace{10mm}
    \subfigure[
    The \textbf{response} model evaluates the leaf nodes of its fixed-depth search tree using a learned value function $v_{\theta_{R}}$.
    It computes the action utilities $u_i(\cdot, \pi_{-i})$ induced by the trained proxy policy $\pi_{\theta_{P}}$ given temperatures $\tau_{1},\tau_{2},\dots,\tau_{N}$.
    The policy $\pi_{\theta_{R}}$ approximates a Smooth Best Response (SBR) with fixed response temperature $\tau_R$ to the action utilities.
    ]{%
        \includegraphics[width=0.45\textwidth]{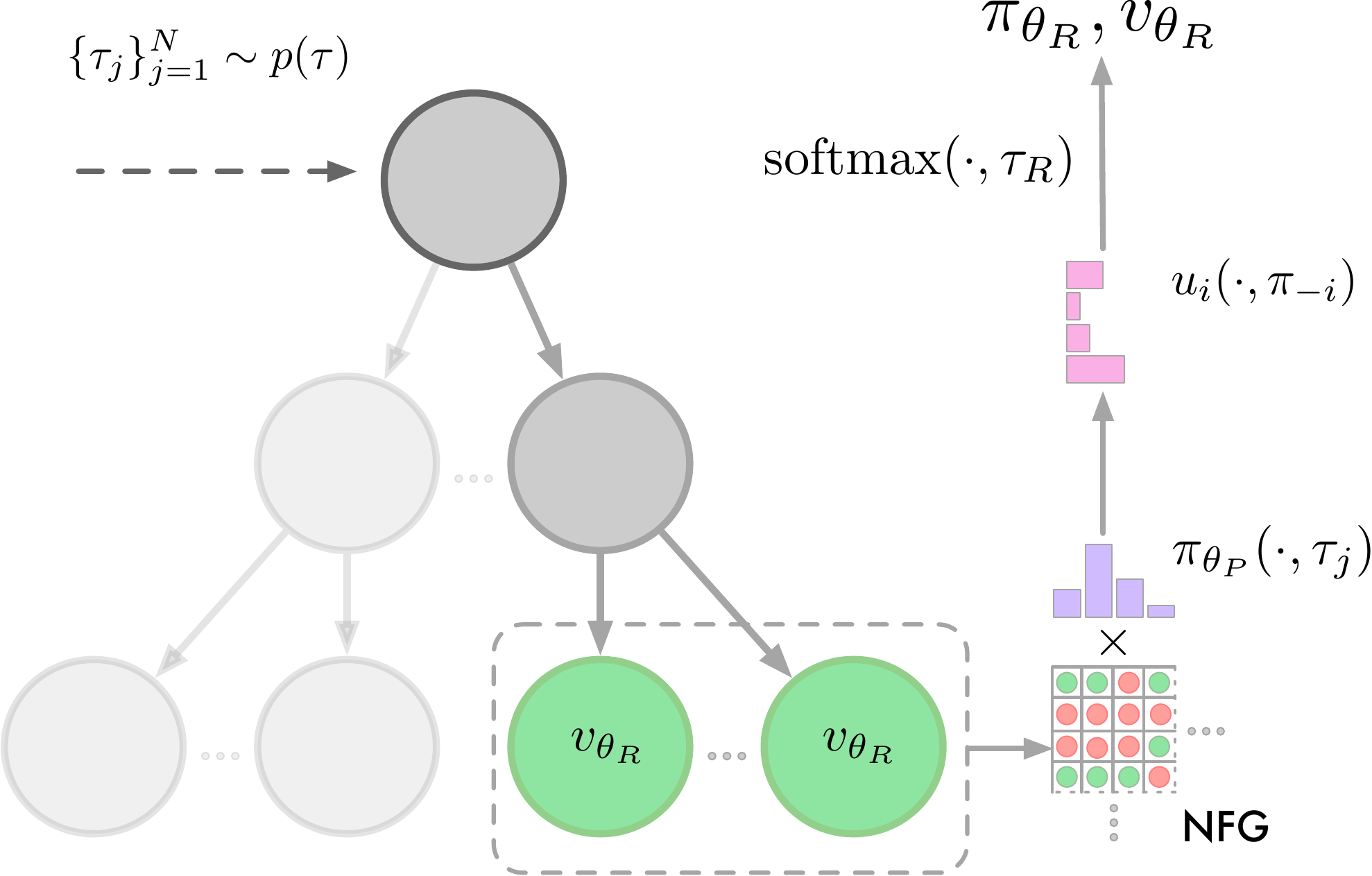}
    }
    \caption{{\label{fig:architecture}Training architecture of the proxy and response models of Albatross. Both models are trained via planning-augmented self-play using fixed-depth search and are conditioned on one (proxy model) or multiple (response model) temperatures $\tau$ that are drawn from a distribution $p(\tau)$. The response model uses the trained proxy model to compute the Smooth Best Response Logit Equilibrium (SBRLE).}}
\end{center}
\vskip -0.2in
\end{figure*}

\section{Method}
In contrast to existing methods, we do not use an ensemble to train with diverse agents, but rather focus on training against agents of different playing strengths.
This strength is parameterized by a scalar temperature $\tau$.
We consider the following setup: a perfectly rational agent $i$ plays a game with one or multiple weak agents, who only inherit a bounded rationality. 
Therefore, each of the weak agents $j \in -i$ is modeled by a scalar $\tau_j$.

\subsection{Smooth Best Response Logit Equilibrium (SBRLE)}
Existing concepts of asymmetric equilibria, i.e. QNE and QSE, are a model for the interaction of a single weak agent and a perfectly rational agent.
In detail, the weak agent plays an SBR to the respective optimal policy of the rational agent.
However, this implies that the optimal policy of the rational agent is known to the weak agent.
In other words, their bounded rationality only prevents them from computing their own optimal strategy, but not of the rational player.
This assumption is valid in repeated interactions where the other agents policy can simply be observed, but violated in our setup of zero-shot interactions.

To model the symmetric bounded rationality in a setup with multiple weak agents, we utilize Logit equilibria (LE).
The LE does not need to be the same for all bounded rational agents, i.e. they could play according to different LE with different temperatures.
Given the weak agents play according LEs, we can compute a (smooth) best response to their policies.
We call the joint policy of LE and the rational agents response a \emph{Smooth Best Response Logit Equilibrium (SBRLE)}:
$(SBR(\pi_{-i}, \tau_R), \pi_{-i})$ is a SBLRE, iff $\forall\, j \in -i\,$ $\pi_{j}$ is a LE policy with temperature $\tau_{j}$, where $\tau_R$ is the response temperature of the rational player.
If $\tau_R = \infty$, then the rational player plays a best response to the weak agents and we call the corresponding joint policy a \emph{Best Response Logit equilibrium (BRLE)}.

To highlight the difference between QSE and SBRLE: In the QSE, the weak agent plays an SBR to the optimal policy of the rational agent, but in the SBRLE to the LE policy, if the rational agent were to play according to LE.
The different equilibria are visualized in \cref{fig:equilibria} for a zero-sum NFG.
Even though the NE achieves maximal utility against perfectly rational opponents, it does not gain any utility from playing a weak opponent (in the example, NE lies on the intersection of BR and SBR).
Depending on the NFG and rationality of agents, the quantal equilibria yield higher expected utility.
In the SBRLE, the rational agent can freely choose the response temperature $\tau_R$.
A high response temperature directly corresponds to higher utility, but a response temperature of $\tau_R = \infty$ may impair the training process, since a BR is not necessarily unique.
Consequently, the SBRLE is preferable to BRLE as it yields a unique learning target.
For a detailed discussion on the effect of a unique learning target, see Appendix \ref{subsec:nash_instability}.

\subsection{Albatross}


The computation of an SBRLE requires a full traversal of the game tree for every agent, initially to compute all Logit equilibria and afterwards the BR of the rational player.
Additionally, a complete re-computation is required for every new temperature estimate of the weak agents.
Since this is infeasible for most games, we present Albatross, which approximates an SBRLE using neural networks.
The training of Albatross consists of two stages.
Firstly, a proxy model approximates Logit equilibria at different temperatures.
Then, a response model is trained to exploit the proxy model.
Both models are trained using an adaptation of AlphaZero.
We briefly outline the original AlphaZero algorithm, but refer to \cref{sec:az_alb_training} and the original paper \cite{silver_alphazero, silver_mastering_2017} for a detailed explanation.
Using AlphaZero, an actor-critic network predicting policy $\pi_{\theta}(o_i)$ and value $v_{\theta}(o_i)$ of agent $i$ is trained, where $o_i$ is an observation of the current game state from the perspective of agent $i$.
During training, Monte-Carlo tree search (MCTS) \cite{metropolis_monte_1949} is used as policy improvement operator, i.e. the policy of the root node is used as target for gradient updates.
Targets for the gradient updates of the value function are the cumulative rewards of an episode, i.e. the result of a Monte-Carlo policy evaluation.
During MCTS, the value function is used as an evaluation heuristic in the leaf nodes and the policy as guidance for exploration.

Similar to the original AlphaZero algorithm, the proxy model predicts policy $\pi_{\theta_{P}}(o_i, \tau)$ and value $v_{\theta_{P}}(o_i, \tau)$ of agent $i$ based on observation $o_i$, but is also conditioned on the temperature $\tau$.
This conditioning allows the prediction of policy and value for LE of any temperature in the training distribution.
During training, the temperature $\tau$ is sampled at the beginning of an episode from a distribution $p(\tau)$ in the interval $[\tau_{\mathit{min}}, \tau_{\mathit{max}}]$ (see \cref{subsec:temp_dist} for details regarding the training distribution).
For the policy- and value improvement operator, we utilize fixed depth search instead of MCTS.
In detail, we start at the leaf nodes and construct an NFG from sibling nodes.
Then, we use a solver to compute the Logit equilibrium of the NFG and propagate the expected utility of the equilibrium to the parent node.
This process is repeated until the LE at the root node is known.
Then, the policy and value at the root node are used as targets for the gradient updates.



After the proxy model is trained, we start training of the response model.
In contrast to the single scalar temperature $\tau$ of the proxy model, the response model is conditioned on the temperature of every agent except itself, which we define as $\tau_{-i} =  (\tau_1, \ldots, \tau_{i-1}, \tau_{i+1}, \ldots, \tau_N)$.
Therefore, the response model predicts policy $\pi_{\theta_{R}}(o_i, \tau_{-i})$ and value $v_{\theta_{R}}(o_i, \tau_{-i})$.
Again, we utilize fixed depth search and an adapted backup function for training.
During backup, we approximate the policies $\pi_{-i}$ of other agents with the policy $\pi_{\theta_P}$ of the proxy model.
Given the proxy policies, we compute the smooth best response $SBR(\pi_{-i, \theta_{P}}, \tau_R)$, where the response temperature $\tau_R$ is a fixed hyperparameter.
The policy at the root node as well as its expected utility are used as targets for gradient updates.
The policy $\pi_{\theta_{R}}$ of the response model is used for evaluation.
We visualize the training scheme of proxy and response model in \cref{fig:architecture}.
Detailed pseudocode for training AlphaZero and Albatross can be found in \cref{sec:az_alb_training}.


\subsection{Online Temperature Estimation}
The previous two sections presented the methodology for training a policy conditioned on the temperature of other agents.
At test-time, the rational agent has to estimate the temperature of the weak agents to input these temperatures into the response model.
As a first option, the temperature can originate from insights about the agent.
For example, a uniformly random agent always plays with temperature of $\tau = 0$.
In zero-shot interactions, no previous knowledge about the other agent exists, but it is possible to estimate their temperature online using \emph{Maximum Likelihood Estimation (MLE)}.
We consider the following setup: a rational agent $i$ intents to estimate the temperatures $\tau_j$ of weak agents $j \in -i$.
Given $K$ observations of actions $(a_j^1, \ldots, a_j^K)$ and corresponding optimal policies $(\pi_{-j}^1, \ldots, \pi_{-j}^K)$ of the other agents, one can estimate the temperature $\tau_j$ of weak agent $j$.
The log-likelihood $l(\tau_j)$ of player $j$ exhibiting temperature $\tau_j$ can be computed \cite{Reverdy2015ParameterEI} as
\begin{align*}
    l(\tau_j) &= \sum_{k=1}^K \Big[ \tau_j\, u_j^k(a_j^k, \pi_{-j}^k) \!-\! \ln \!\!\!{\sum_{a_j \in A_j} \!\!\exp\big(\tau_j\, u^k_j(a_j, \pi_{-j}^k})\big) \Big], \\
    \frac{\partial l}{\partial \tau_j} &= \sum_{k=1}^K \Bigg[  u_j^k(a_j^k, \pi_{-j}^k) \\
    &\qquad- \frac{\sum_{a_j \in A_j}  u^k_j(a_j, \pi_{-j}^k) \exp\big(\tau_j\, u^k_j(a_j, \pi_{-j}^k)\big)}{\sum_{a_j \in A_j}  \exp\big(\tau_j\, u^k_j(a_j, \pi_{-j}^k)\big)} \Bigg].
\end{align*}
In order to find the maximum likelihood, one can utilize the gradient of the likelihood function $\frac{\partial l}{\partial \tau_j}$.
The global optimum can be computed using simple line search over the temperature $\tau_j$, because the likelihood function is concave, which we prove in \cref{subsec:concavity_proof}.
Note that globally optimal policies are not well defined if multiple Nash equilibria exist due to the equilibrium selection problem \cite{nash_selection}.
Therefore, we define optimal play in relation to the learned equilibrium of the proxy model.
Specifically, we use the policy $\pi_{\theta_{P}}(\cdot, \tau_{\mathit{max}})$ of the proxy model with the highest temperature $\tau_{\mathit{max}}$.

\begin{figure}[!t]
\vskip 0.2in
\begin{center}
\centerline{\includegraphics[width=0.9\columnwidth]{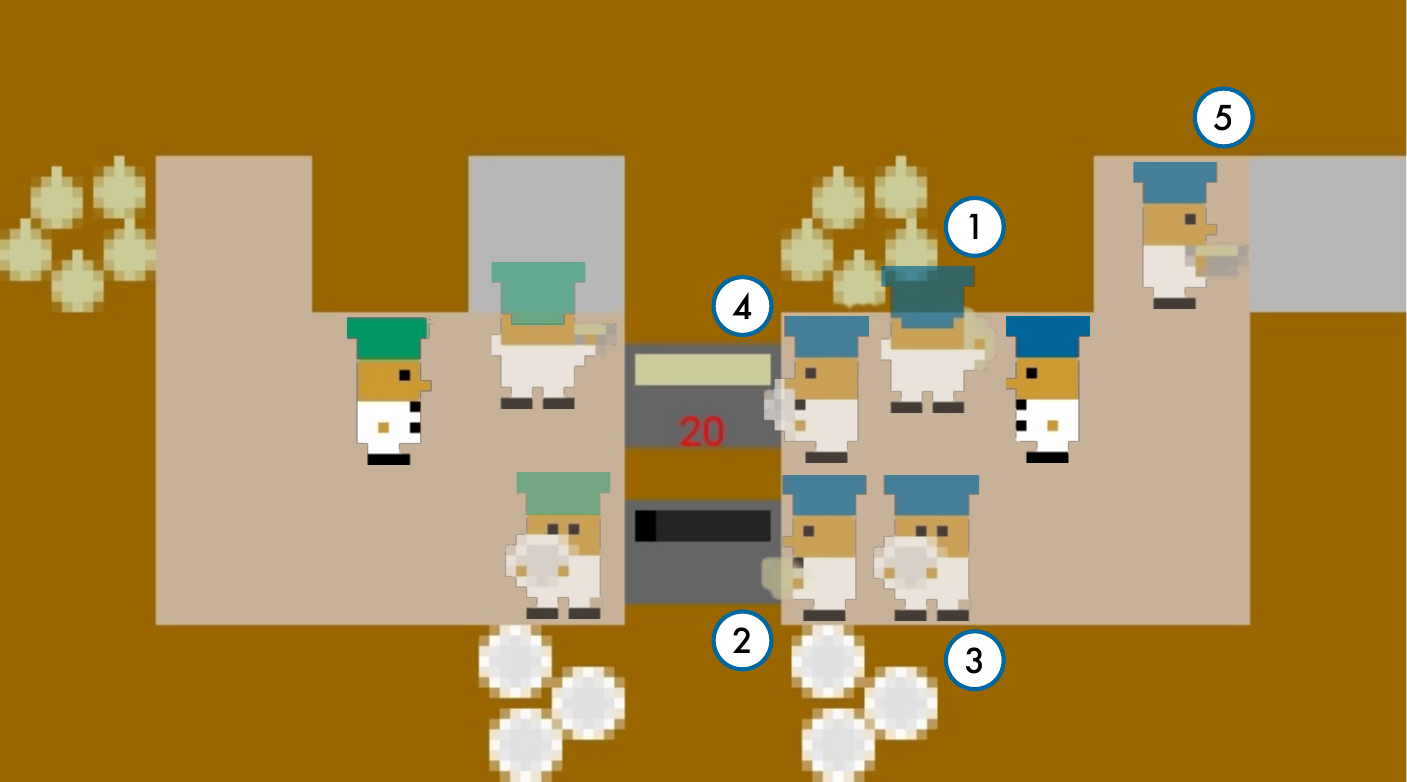}}
\caption{
Albatross agent (\textcolor{blue}{right}) and a possibly weak agent (\textcolor{green}{left}) in the Asymmetric Advantage layout of Overcooked.
If the left agent plays rationally, they should realize that they have a shorter path to their serving location (gray tile).
They would move down, retrieve a dish and deliver the soup.
Having a strong estimation of rationality (e.g. $\tau = 10$), Albatross trusts them to deliver the soup and moves up to pick up an onion \includegraphics[width=8pt]{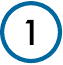} and prepare the next soup in the other pot \includegraphics[width=8pt]{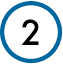}.
If Albatross has an estimation of weak rationality for the left agent (e.g. $\tau = 0$), then Albatross moves down to retrieve a dish \includegraphics[width=8pt]{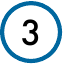}, collects the soup \includegraphics[width=8pt]{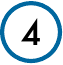} and serves it themselves \includegraphics[width=8pt]{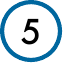}.
}
\label{fig:oc_situation}
\end{center}
\vskip -0.2in
\end{figure}

\section{Empirical Evaluation}
\label{sec:eval}
In our experiments, we want to research the following questions:
(\textbf{Q1}) How does Albatross cooperate with unknown agents and adapt its behavior?
(\textbf{Q2}) What effect has the temperature on its behavior? 
(\textbf{Q3}) Is Albatross able to estimate rationality within a single episode?
(\textbf{Q4}) Can Albatross exploit weak enemies in the competitive domain?
We answer these questions in the cooperative game of Overcooked \cite{overcooked} and the competitive game of Battlesnake \cite{Chung2020BattlesnakeCA}.
In all experiments, we evaluate on five different seeds.

\subsection{Cooperative Overcooked}

For cooperative tasks, we evaluate Albatross in the Overcooked benchmark.
In this game, two agents are placed in a kitchen and tasked with cooking as many soups as possible in a given time frame.
The agents may perform six actions: move up, down, left, right, stay in place or interact with the environment.
To cook a soup, an agent firstly needs to fetch and place an onion in a pot three times.
Then, they have to start the cooking process, wait for 20 steps and retrieve a dish from a dish dispenser.
Lastly, they have to put the soup on the dish and serve the soup at a counter.
There exist five different kitchen layouts (see \cref{sec:overcooked}).

To answer (\textbf{Q1}) qualitatively, we show an example of the adaptive behavior of Albatross in \cref{fig:oc_situation}.
Albatross performs different action sequences depending on the rationality estimation of the other agent.
For a quantitative answer, we simulate the cooperation with humans by evaluating with a behavior cloning agent trained on human data.
We compare Albatross to Proximal Policy-Optimization \cite{Schulman2017ProximalPO}, Population-based Training \cite{Jaderberg2017PopulationBT}, Fictitious Co-Play \cite{NEURIPS2021_797134c3}, Trajectory Diversity \cite{lupu21a}, Maximum Entropy Population-Based training \cite{mep} and PECAN \cite{pecan}.
In \cref{fig:oc_main}, the reward in cooperation with the behavior cloning agent is displayed.
In all layouts, except Forced Coordination, Albatross yields higher cooperation rewards than all baseline methods.
On average, Albatross outperforms PECAN by 37.6\%.
In the Forced Coordination layout, good behavior is difficult to learn as rewards highly depend on the actions of the other agents.

\begin{figure}[!t]
\vskip 0.2in
\begin{center}
    \centerline{\includegraphics[width=\columnwidth]{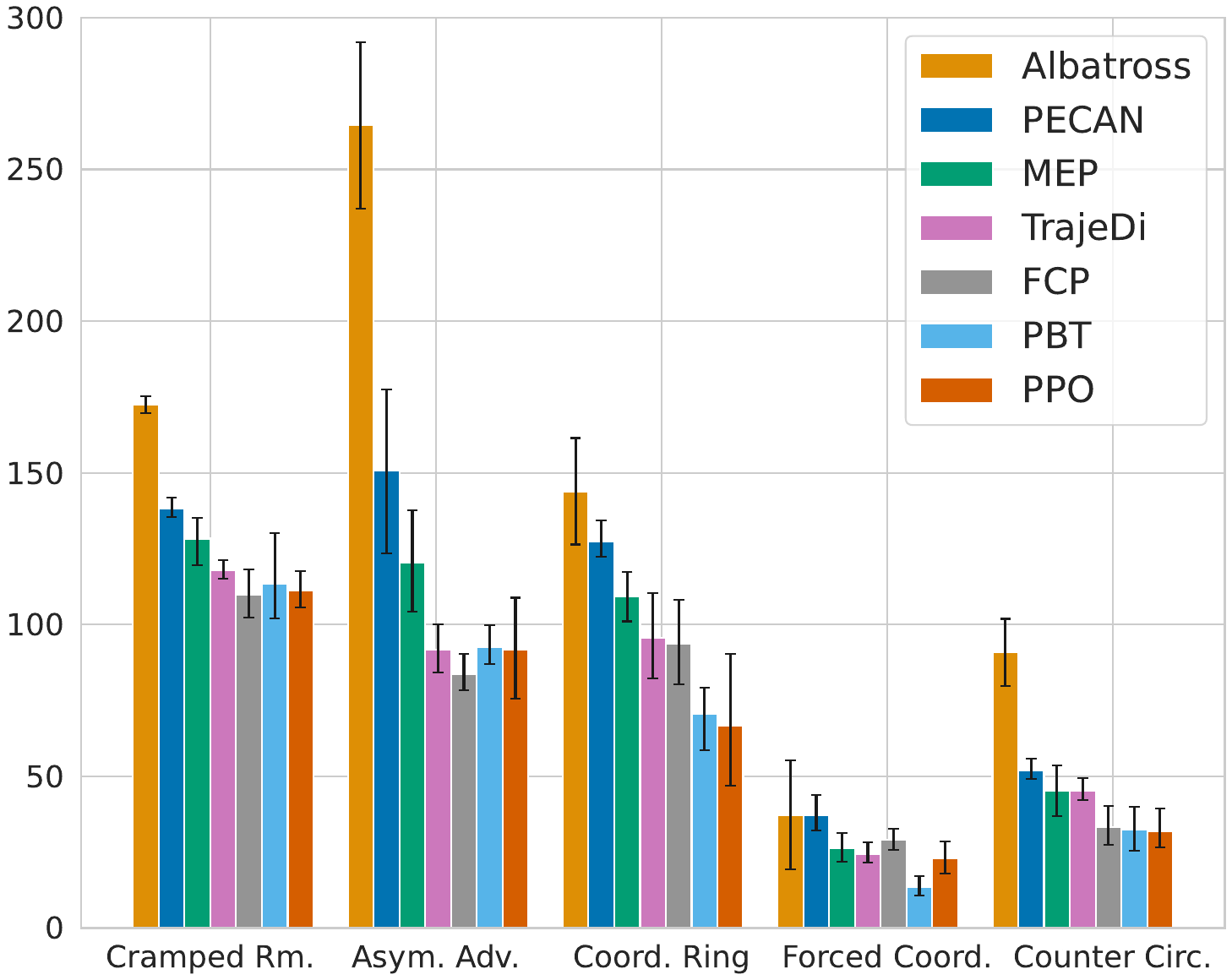}}
    \caption{
    Cooperation performance with a behavior cloning agent trained on a dataset of human play \cite{overcooked} in all five layouts of Overcooked.
    Episodes last 400 time steps and agents receive a common reward of 20 for delivering a soup.
    }
    \label{fig:oc_main}
\end{center}
\vskip -0.2in
\end{figure}

\begin{table*}[t]
\caption{Cooperation performance with scripted agents exhibiting specific behaviors, e.g. always placing onions in the cooking pot. None of the behavior patterns were used during training, such that the learned policy has to adapt in zero-shot coordination.}
\label{tbl:eval_scripted}
\vskip 0.15in
\begin{center}
\begin{small}
\begin{sc}
\begin{tabular}{llrrrrr}
\toprule
Environment & Script                & FCP              & MEP            & TrajDiv       & HSP         & Albatross \\
\midrule
Asym. Adv.  & Onion Placement       & 334.8$\pm$13.0   & 330.5$\pm$14.2& 323.6$\pm$17.0& \textbf{376.8$\pm$9.9}  & 342.5$\pm$11.3 \\
Asym. Adv.   & Onion Plac.+Delivery & 297.7$\pm$3.4    & 298.5$\pm$3.4 & 290.0$\pm$4.7 & 300.1$\pm$4.1  & \textbf{309.2$\pm$11.1} \\
Coord. Ring   & Onion Everywhere    & 109.1$\pm$7.9    & 124.0$\pm$3.4 & 116.9$\pm$8.9 & 121.2$\pm$12.6 & \textbf{143.9$\pm$2.4}  \\
Coord. Ring & Dish Everywhere       & 94.4$\pm$3.8     & 100.2$\pm$5.3 & 107.3$\pm$5.3 & 115.4$\pm$7.4  & \textbf{117.4$\pm$0.7} \\
Counter Circ. & Onion Everywhere    & 63.7$\pm$9.2     & 88.9$\pm$5.1  & 82.0$\pm$12.8 & 107.5$\pm$3.5  & \textbf{119.6$\pm$1.2} \\
Counter Circ. & Dish Everywhere     & 57.0$\pm$5.3     & 53.0$\pm$1.8  & 57.2$\pm$2.2  & \textbf{78.5$\pm$4.1}   & \textbf{78.9$\pm$4.3} \\
\bottomrule
\end{tabular}
\end{sc}
\end{small}
\end{center}
\vskip -0.1in
\end{table*}

Additionally, we evaluate the zero-shot cooperation performance of Albatross with different scripted agents, that exhibit specific behavioral pattern \cite{hsp}.
For example, we pair the learned policy with agents that only place onions in a pot, or place dishes everywhere in the kitchen.
Results of the evaluation are displayed in \cref{tbl:eval_scripted}.
The scripted policies are out-of-distribution regarding the training of Albatross, since they exhibit specific irrational behavior regardless of the reward.
In contrast, Albatross only trained with Logit equilibrium policies, which exhibit an error probability proportional to the expected reward per action.
Nevertheless, in most settings Albatross achieves cooperation performance greater or equal to Hidden-Utility Self-Play (HSP) \cite{hsp}, whose event-based training models such biased policies.
This indicates that Albatross learns best responses which are robust against violations of the modelling assumptions.

\begin{figure}[!b]
\begin{center}
    \subfigure[Proxy Policy Entropy]{%
        \includegraphics[width=0.43\columnwidth]{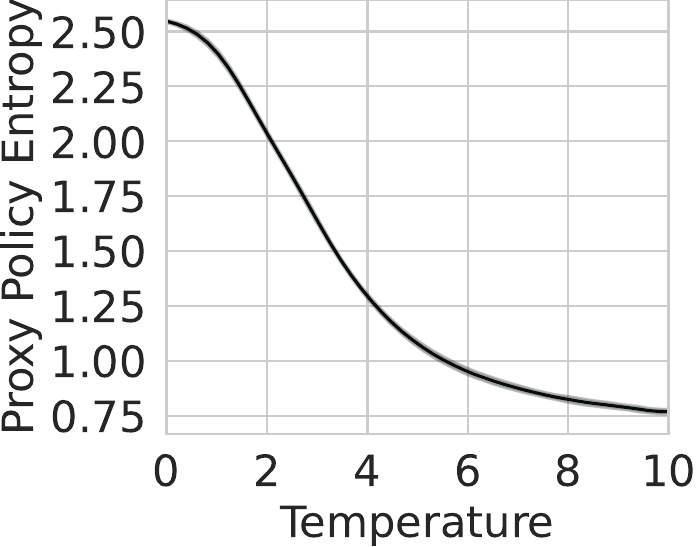}
    }
    \subfigure[MI Proxy \& Response]{%
        \includegraphics[width=0.43\columnwidth]{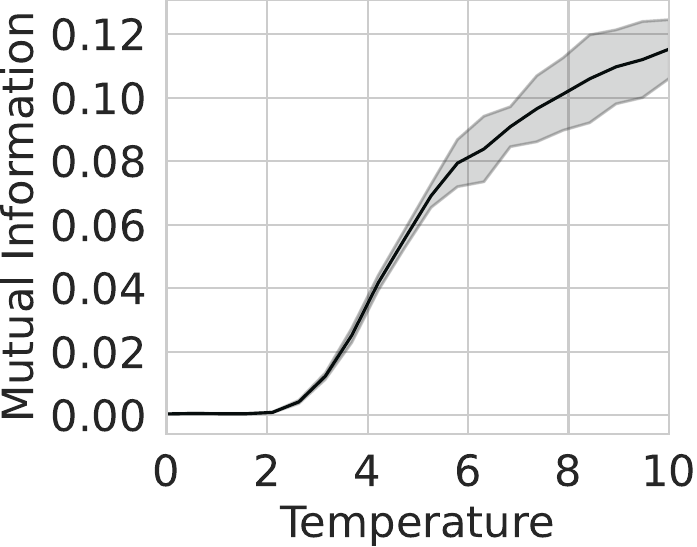}
    }
    \caption{
    Effects of temperature on the behavior of Albatross in the Counter Circuit layout.
    (a) The entropy of the proxy model decreases with rising temperature and (b) the mutual information between response and proxy policy increases.
    }
    \label{fig:entropy_mi}
\end{center}
\end{figure}

\begin{figure*}[t]
\vskip 0.2in
\begin{center}
    \subfigure[Cramped Room]{%
        \includegraphics[height=0.19\textwidth]{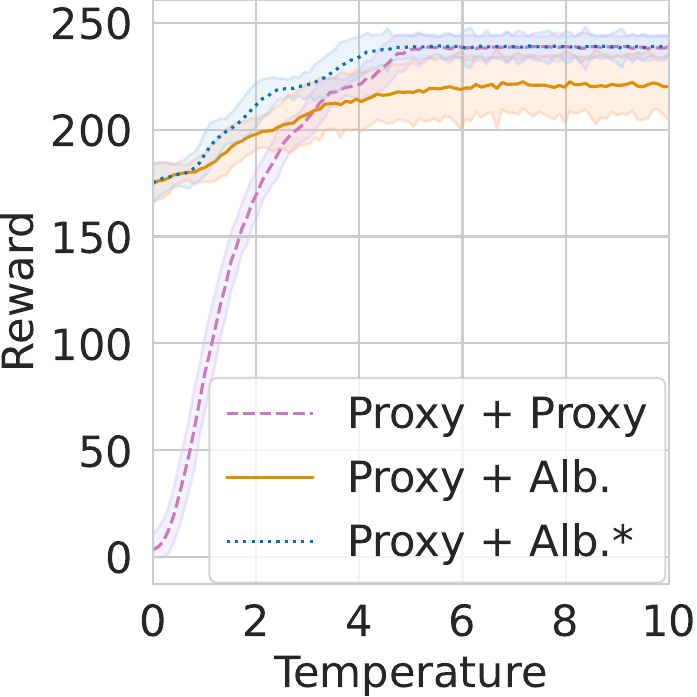}
    }
    \subfigure[Asym. Advantage]{%
        \includegraphics[height=0.19\textwidth]{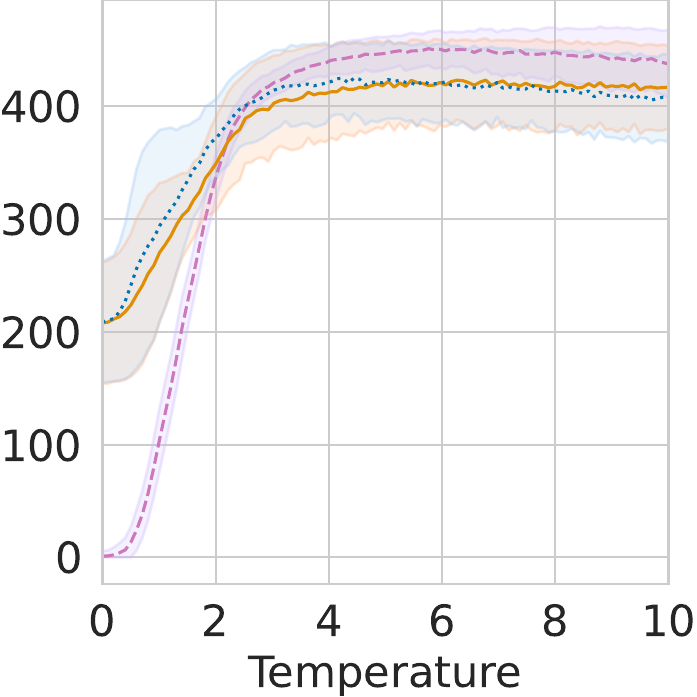}   
    }
    \subfigure[Coord. Ring]{%
        \includegraphics[height=0.19\textwidth]{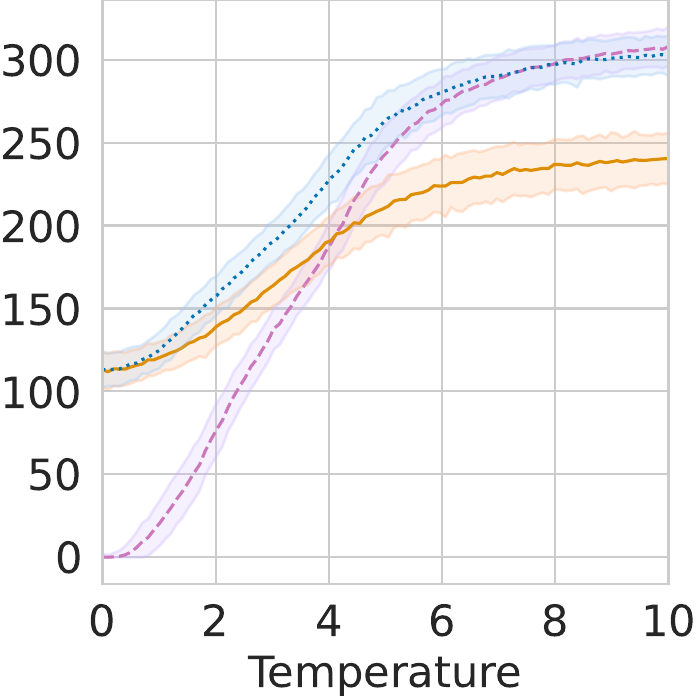}
    }
    \subfigure[Forced Coord.]{%
        \includegraphics[height=0.19\textwidth]{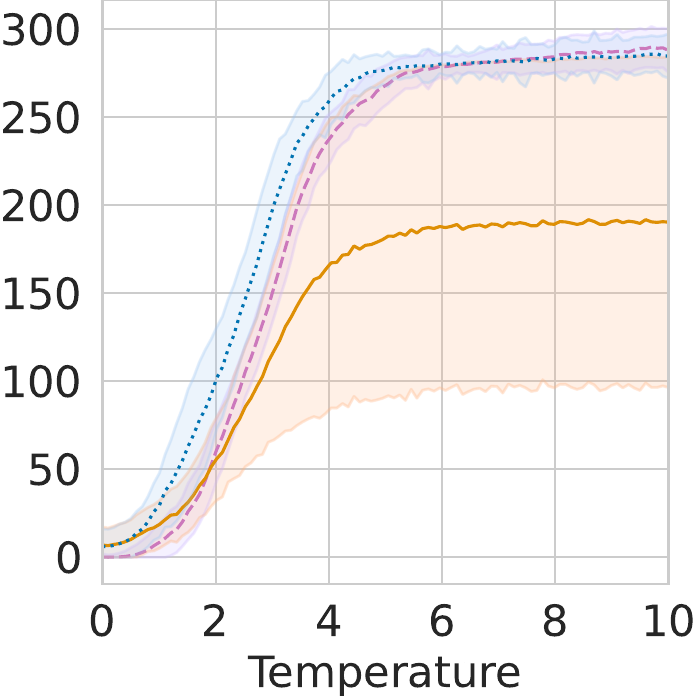}
    }
    \subfigure[Counter Circuit]{%
        \includegraphics[height=0.19\textwidth]{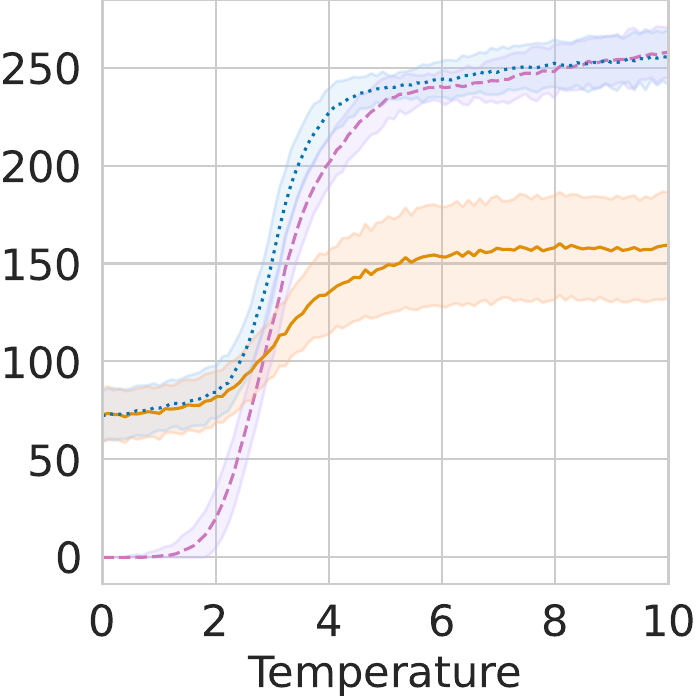}
    }
    \caption{
    Self-play performance of the proxy model as well as cooperation performance between Albatross and the proxy model at different temperatures.
    Albatross* denotes the cooperation capability of Albatross without Maximum Likelihood Estimation, i.e. if the the true fixed temperature of the proxy model is given.
    }
    \label{fig:proxy_temps}
\end{center}
\vskip -0.2in
\end{figure*}

To analyze the effect of temperature on the behavior of Albatross (\textbf{Q2}), we firstly analyze proxy and response model directly.
In \cref{fig:entropy_mi}, we show that the entropy of the proxy policy decreases with rising temperature.
This corresponds to the lower error probability of agents with higher rationality.
To further highlight the adaptive behavior of Albatross, we measure the mutual information \cite{Li2022PMICIM,raceMARL2023} between its response- and proxy policy:
\begin{equation*}
    I(\pi_{\theta_R};\pi_{\theta_P}) = H(\pi_{\theta_P}) - H(\pi_{\theta_P} \mid \pi_{\theta_R}).
    \label{eq:mutual_information}
\end{equation*}
With rising temperatures, we observe a drop in the entropy of the proxy policy with an increase in mutual information, estimated via the conditional action frequencies.
This captures the level of cooperation between both policies as it implies a decrease in the conditional entropy $H(\pi_{\theta_P} \mid \pi_{\theta_R})$.
Consequently, Albatross cooperates with rational agents and acts self-reliant if the other agent does not cooperate.

\begin{figure}[!b]
\begin{center}
    \subfigure[MLE during Episode]{%
        \includegraphics[width=0.43\columnwidth]{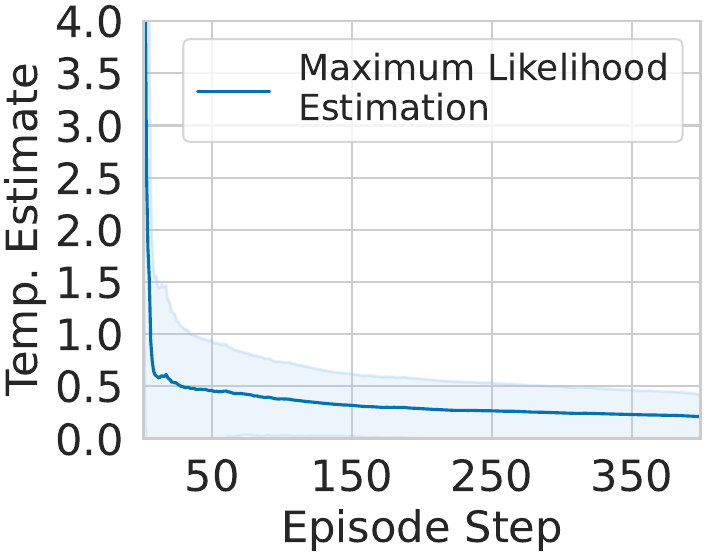}
    }
    \subfigure[Temp. Misspecification]{%
        \includegraphics[width=0.43\columnwidth]{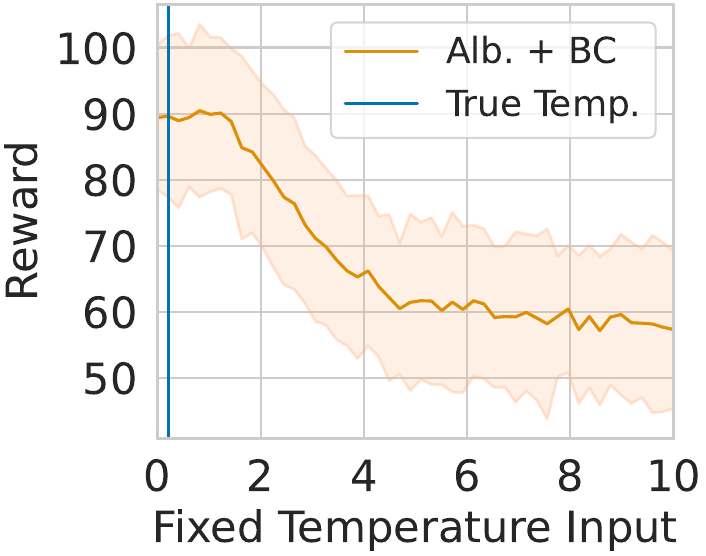}
    }
    \caption{
    Analysis of temperature estimation of Albatross in the Counter Circuit layout.
    (a) During an episode, temperature estimation of MLE quickly converges to the true value. 
    (b) We analyze the the expected reward using a fixed temperature input.
    }
    \label{fig:alb_fix}
\end{center}
\end{figure}

\begin{figure*}[!t]
\vskip 0.2in
\begin{center}
    \subfigure[Tron]{%
        \includegraphics[width=0.65\columnwidth]{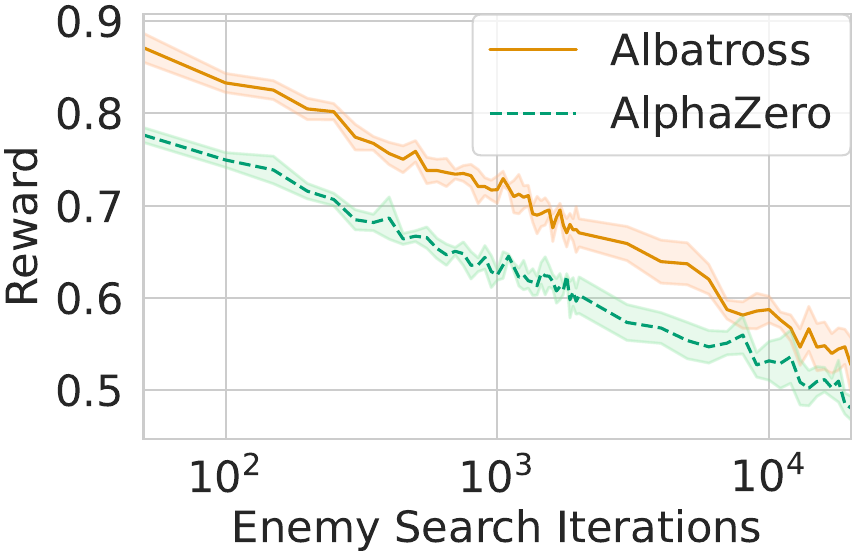}
    }
    \subfigure[Stochastic 2 Player]{%
        \includegraphics[width=0.65\columnwidth]{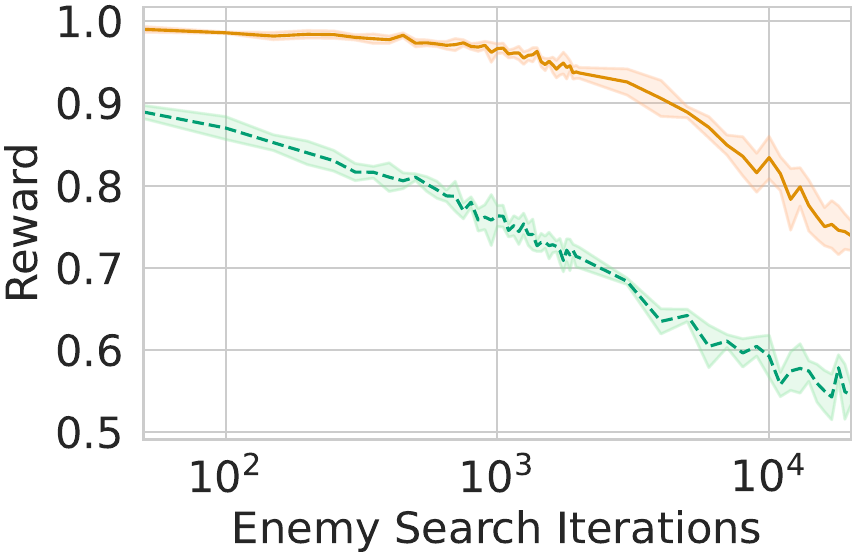}
    }
    \subfigure[Stochastic 4 Player]{%
        \includegraphics[width=0.65\columnwidth]{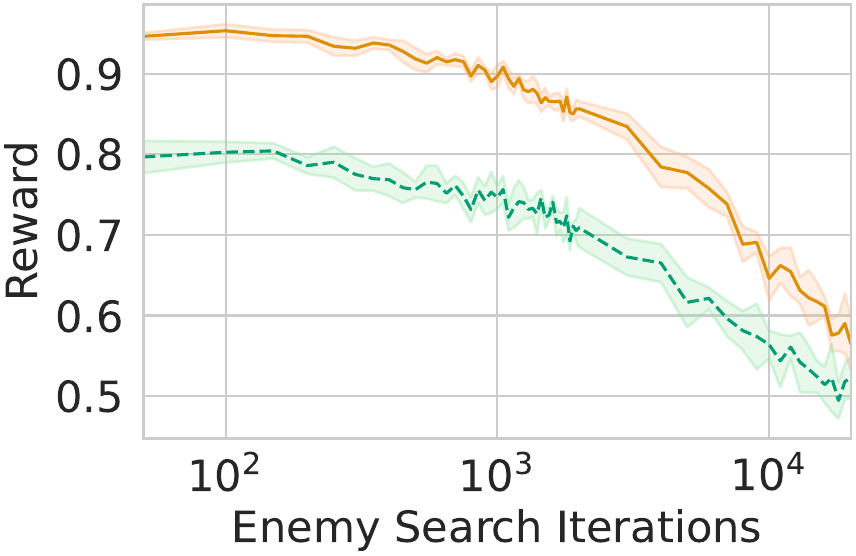}
    }
    \caption{
    Evaluation of AlphaZero and Albatross against baseline MCTS (+ value function heuristic) agents with different budgets of search iterations in the game modes of Battlesnake.
    Albatross is able to exploit weak agents better than AlphaZero.
    }
    \label{fig:az_depth_eval}
\end{center}
\vskip -0.2in
\end{figure*}

Next, we analyze the effect of temperature on the expected reward in \cref{fig:proxy_temps}.
For temperature $\tau = 0$, the proxy model learned a uniformly random policy and does not achieve any reward in self-play.
At higher temperatures, the achieved reward directly corresponds to the temperature.
Analyzing Albatross with the proxy model shows the cooperation with agents of different playing strength.
For example, the reward attainable when cooperating with a uniformly random agent is visible at $\tau=0$.
Given an exact temperature estimation (denoted as Albatross*), we expect the reward of Albatross* with the proxy model to converge to the proxy self-play performance at high temperatures as both play optimally.
We can observe this effect in all layouts except Asymmetric Advantage, where the training of the response model did not perfectly converge.

To show that Albatross is able to estimate rationality within a single episode (\textbf{Q3}), we can observe the difference in reward between Albatross and Albatross* in \cref{fig:proxy_temps}.
At high temperatures, the reward obtained by Albatross is lower than the self-play performance of the proxy model due to the aleatoric uncertainty of the MLE.
The extend of this uncertainty depends on the layout.
In \cref{fig:alb_fix}, we show the result of MLE at different time steps during an episode in cooperation with the human behavior cloning agent.
After only few time steps, MLE converges towards the true temperature estimate.
An evaluation of Albatross with a fixed temperature input reveals that minor estimation errors have little effect on the achieved reward.
However, major overestimation of the other agents rationality leads to a significant drop in performance.


In \cref{fig:alb_fix}, after only few episode steps in an episode, the MLE already estimates a temperature close to the true value.
An evaluation of Albatross with a fixed temperature input reveals that minor estimation errors have little effect on the achieved reward.
However, major overestimation of the other agents rationality leads to a significant drop in performance.

\subsection{Competitive Battlesnake}

To show that Albatross is able to exploit weak agents in the competitive domain (\textbf{Q4}), we evaluate in the game of Battlesnake.
The game takes place on a grid, where agents have to survive as long as possible.
They die, if they collide either with a wall or the body of a snake.
If two agents collide head-to-head, the longer snake survives.
In the stochastic extensions, food spawns randomly on the map which agents have to eat to prevent starving and grow their body.
In contrast, there exist no food in the game of Tron and the body of each snake is elongated in every turn.
All three game modes are visualized in \cref{fig:game_modes}.
Agents receive a reward of $+1$ for winning and $-1$ for dying in the game modes with two agents.
In the mode with four agents, a reward of $+1$ is distributed among the living agents if another agent dies.

We compare Albatross against AlphaZero \cite{silver_alphazero}.
Since AlphaZero was developed for sequential games, we perform an extensive analysis on the adaptation of AlphaZero to simultaneous games (see Appendix \ref{sec:az_sim}).
For a fair comparison, Albatross and AlphaZero are trained on the same time budget and hardware.
Since training Albatross is a two-step process, proxy and response model are trained for half as long as AlphaZero.




\begin{figure}[!b]
\begin{center}
    \subfigure[Tron]{%
        \includegraphics[width=0.31\columnwidth]{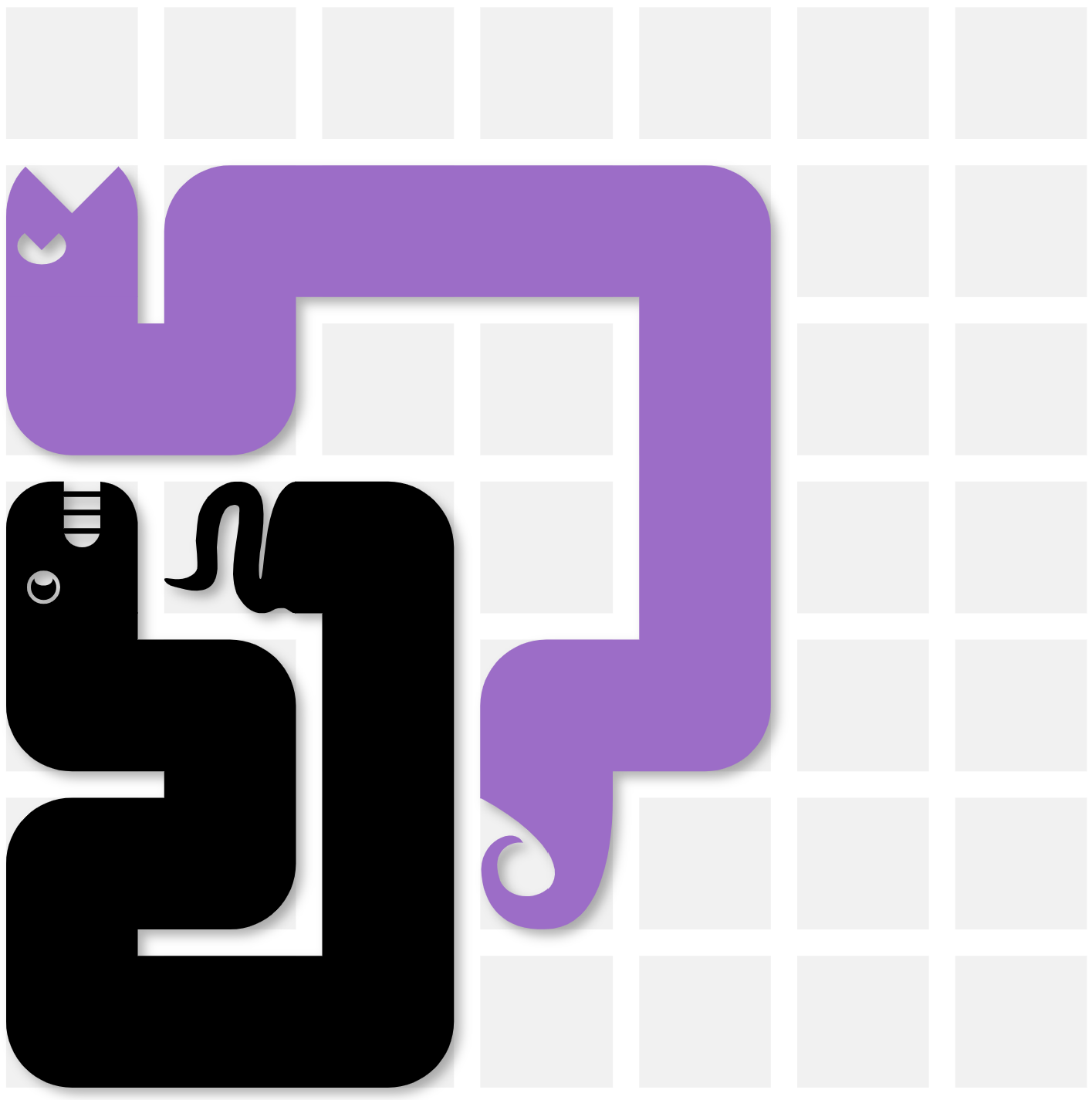}
    }
    \subfigure[Stoch. 2 Player]{%
        \includegraphics[width=0.31\columnwidth]{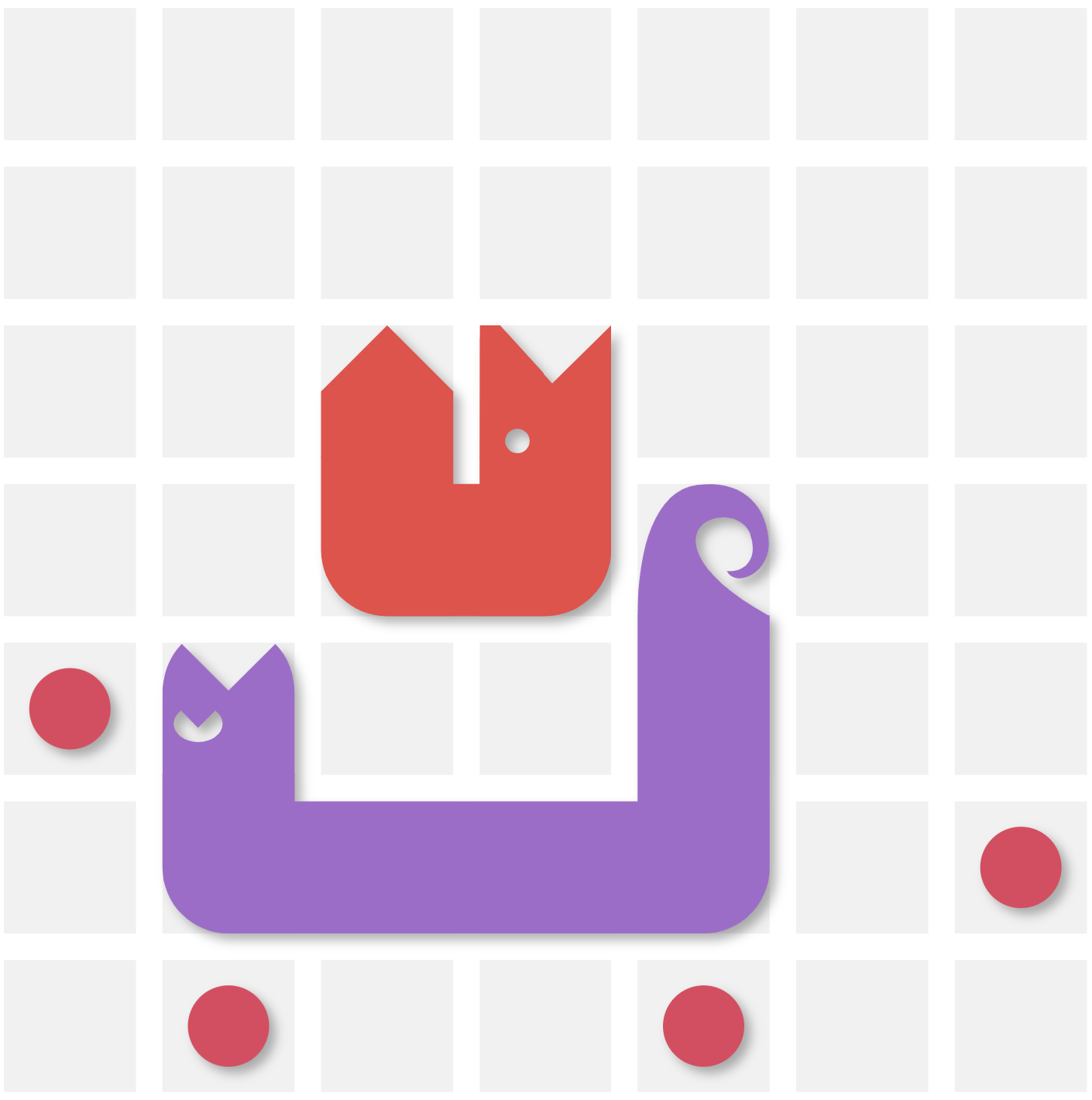}
    }
    \subfigure[Stoch. 4 Player]{%
        \includegraphics[width=0.31\columnwidth]{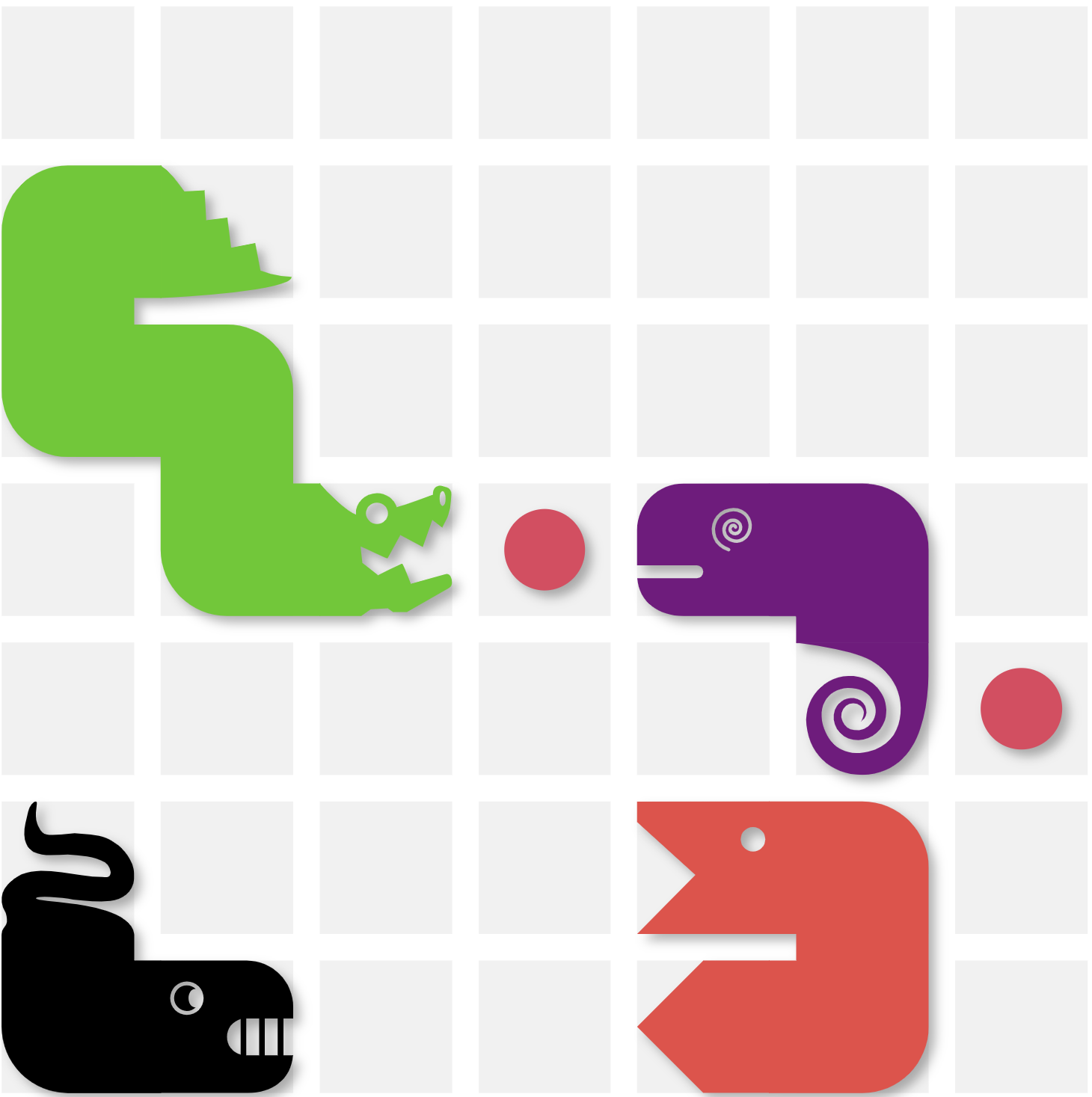}
    } 
    \caption{
    Game modes of Battlesnake: the standard game of Tron as well as the stochastic extensions of Battlesnake for either two or four agents. Food is visualized by red circles.
    }
    \label{fig:game_modes}
\end{center}
\end{figure}

In \cref{fig:az_depth_eval}, both play against baseline agents, which utilize simultaneous-move MCTS \cite{sm_search} with a handcrafted value heuristic adapted from Schier and Wüstenbecker \yrcite{niedersnake}.
Rationality of the baseline agents is modulated by the compute budget, i.e. the number of tree search iterations.
Albatross consistently outperforms AlphaZero and the reward difference is highest against weaker enemies.
That is, because they violate the perfect rationality assumption of AlphaZero.
In contrast, Albatross is able to identify and exploit their weak rationality.
Strong enemies are not exploitable, which leads to a convergence of the reward achieved by Albatross and AlphaZero.
Interestingly, in the stochastic 2-player mode, the reward difference between Albatross and AlphaZero is highest for medium rational agents with about $10^3$ iterations.
Those agents are able to surprise AlphaZero by playing suboptimal, but still good enough to win some games.
This effect does not occur in the game of Tron as games are shorter and mistakes result in a quick death.

Additionally, we play a tournament between Albatross, AlphaZero and baseline agents of different strengths in the stochastic mode with four agents.
In \cref{fig:4nd7_tnmt}, the results of this tournament are displayed.
We plot the estimated temperature of the agents and the TrueSkill scores achieved in the tournament.
For the rationality of Albatross we use the response temperature $\tau_R = 10$.
AlphaZero achieves nearly maximum temperature of $\tau = 9.99$, which corresponds to optimal play assuming all other agents play optimal as well.
However, the baseline agents do not play optimally, which Albatross is able to detect and exploit.
Consequently, it achieves a higher TrueSkill score than AlphaZero.

\subsection{Cooperative Battlesnake}
Lastly, we evaluate the capabilities of Albatross to cooperate with more than two agents.
To this end, we devise a game variant of Tron with cooperative rewards, i.e. the objective of all agents is to stay alive as a group for as long as possible.
Since the previously used board size of 7$\times$7 is too small for four players in the deterministic game mode of Tron, we increase the board size to 11$\times$11.
The results are shown in \cref{fig:coop_tron}.
Due to the deterministic game dynamics, the board fills up quickly and games are short.
Therefore, the possible variation in discounted reward is also small.
However, Albatross still outperforms AlphaZero by a small margin when playing with weak partners, i.e. few search iterations.
Again, this difference diminishes with partners of higher rationality since AlphaZero assumes optimal play.
These results verify that Albatross is able to cooperate well with more than two agents of different rationality.

\section{Limitations and Future Work}
\label{sec:limits}
To accurately estimate the rationality of an agent, Albatross requires observations of their behavior.
In \cref{fig:alb_fix}, we demonstrate that this estimation quickly converges within 20 to 30 time steps.
This is exemplified in the game of Tron, which has a maximum game length of 24.
However, Albatross may not be applicable to games with even shorter episodes.
In future work, Albatross could incorporate a prior temperature likelihood and perform maximum a posteriori estimation to be applicable to very short interactions.
For example, a prior likelihood could be obtained from an online leaderboard or other sparse knowledge about the other agent.
Additionally, Unsupervised Reinforcement Learning could be used to obtain a prior policy \cite{polter}.
Another current limitation of Albatross is the dependency on planning in the joint action space of all agents.
The size of the joint action space grows exponentially with the number of agents and number of actions per agent.
Therefore, the tree search becomes a weak improvement operator if only a fraction of nodes can be evaluated.
This prevents the application in domains with large joint action spaces or domains where the environment dynamics are unknown.
For example, Albatross is not applicable to the MeltingPot environment \cite{Agapiou2022MeltingP2} as it does not allow for planning.
In future work, Albatross can be enhanced with a learned environment model akin to MuZero \cite{schrittwieser_mastering_2020} to address these limitations.

\begin{figure}[tb]
\vskip 0.2in
\begin{center}
    \centerline{\includegraphics[width=0.8\columnwidth]{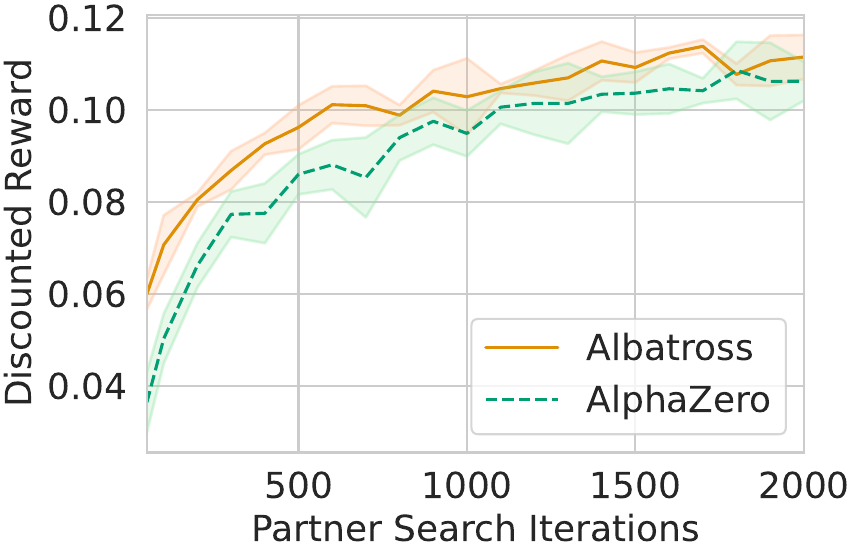}}
    \caption{
    Evaluation of AlphaZero and Albatross n a four player cooperative game variant of Tron with baseline MCTS (+ value function heuristic) agents that have different budgets of search iterations i.
    }
    \label{fig:coop_tron}
\end{center}
\vskip -0.2in
\end{figure}

\section{Conclusion}
We developed the novel equilibrium concept of an SBRLE for modeling the interaction between a single rational and multiple weak agents in zero-shot interactions.
Since the SBRLE is infeasible in most games, we proposed Albatross, which is capable of learning the SBRLE through a combination of self-play and planning.
Using Albatross, we are able to reach state of the art in the Overcooked benchmark for cooperative tasks.
We showed that Albatross is able to estimate rationality of unknown agents within a single episode and cooperate with them by adapting its behavior.
Analyzing the effect of temperature estimation, we find Albatross cooperates with rational partners and behaves self-reliant with weak partners.
Moreover, we showed that Albatross is able to exploit weak enemies in the competitive domain.




\section*{Impact Statement}
This paper represents work contributing to the advancement of Human-AI collaboration.
A potential societal consequence of our work is the adoption of adaptive AI-Assistants, increasing quality of life on the human side.
However, we are aware that rationality is an intricate concept, which cannot completely be modeled using a scalar value.
For real-world applications, aspects of fairness, privacy and safety must be examined.
Moreover, for a given application, ethical aspects must be considered to determine if estimating the rationality of a human is appropriate.


\section*{Acknowledgements}
This work was supported by the Federal Ministry of Education and Research (BMBF), Germany under the AI service center KISSKI (grant no. 01IS22093C), the Lower Saxony Ministry of Science and Culture (MWK) through the zukunft.niedersachsen program of the Volkswagen Foundation and the Deutsche Forschungsgemeinschaft (DFG) under Germany's Excellence Strategy within the Cluster of Excellence PhoenixD (EXC 2122) and  (RO 2497/17-1). 

\bibliographystyle{icml2024}
\bibliography{references}

\begin{thebibliography}{68}
\providecommand{\natexlab}[1]{#1}
\providecommand{\url}[1]{\texttt{#1}}
\expandafter\ifx\csname urlstyle\endcsname\relax
  \providecommand{\doi}[1]{doi: #1}\else
  \providecommand{\doi}{doi: \begingroup \urlstyle{rm}\Url}\fi

\bibitem[Agapiou et~al.(2022)Agapiou, Vezhnevets, Du{\'e}{\~n}ez-Guzm{\'a}n, Matyas, Mao, Sunehag, Koster, Madhushani, Kopparapu, Comanescu, Strouse, Johanson, Singh, Haas, Mordatch, Mobbs, and Leibo]{Agapiou2022MeltingP2}
Agapiou, J.~P., Vezhnevets, A.~S., Du{\'e}{\~n}ez-Guzm{\'a}n, E.~A., Matyas, J., Mao, Y., Sunehag, P., Koster, R., Madhushani, U., Kopparapu, K., Comanescu, R., Strouse, D., Johanson, M.~B., Singh, S., Haas, J., Mordatch, I., Mobbs, D., and Leibo, J.~Z.
\newblock {Melting Pot 2.0}.
\newblock \emph{ArXiv}, abs/2211.13746, 2022.

\bibitem[Akamatsu(1996)]{akamatsu96}
Akamatsu, T.
\newblock {Cyclic flows, Markov process and stochastic traffic assignment}.
\newblock \emph{Transportation Research Part B: Methodological}, 30\penalty0 (5):\penalty0 369--386, October 1996.

\bibitem[Anthony et~al.(2017)Anthony, Tian, and Barber]{NIPS2017_d8e1344e}
Anthony, T., Tian, Z., and Barber, D.
\newblock Thinking {F}ast and {S}low with {D}eep {L}earning and {T}ree {S}earch.
\newblock In Guyon, I., Luxburg, U.~V., Bengio, S., Wallach, H., Fergus, R., Vishwanathan, S., and Garnett, R. (eds.), \emph{Advances in Neural Information Processing Systems}, volume~30. Curran Associates, Inc., 2017.

\bibitem[Auer et~al.(2002)Auer, Cesa-Bianchi, Freund, and Schapire]{Auer2002TheNM}
Auer, P., Cesa-Bianchi, N., Freund, Y., and Schapire, R.~E.
\newblock The {N}onstochastic {M}ultiarmed {B}andit {P}roblem.
\newblock \emph{SIAM J. Comput.}, 32:\penalty0 48--77, 2002.

\bibitem[Bošanský et~al.(2016)Bošanský, Lisý, Lanctot, Čermák, and Winands]{sm_search}
Bošanský, B., Lisý, V., Lanctot, M., Čermák, J., and Winands, M.~H.
\newblock Algorithms for computing strategies in two-player simultaneous move games.
\newblock \emph{Artificial Intelligence}, 237:\penalty0 1--40, 2016.

\bibitem[Brown et~al.(2019)Brown, Lerer, Gross, and Sandholm]{deep_cfr}
Brown, N., Lerer, A., Gross, S., and Sandholm, T.
\newblock Deep {C}ounterfactual {R}egret {M}inimization.
\newblock In Chaudhuri, K. and Salakhutdinov, R. (eds.), \emph{Proceedings of the 36th International Conference on Machine Learning}, volume~97 of \emph{Proceedings of Machine Learning Research}, pp.\  793--802. PMLR, 09--15 Jun 2019.

\bibitem[Carroll et~al.(2019)Carroll, Shah, Ho, Griffiths, Seshia, Abbeel, and Dragan]{overcooked}
Carroll, M., Shah, R., Ho, M.~K., Griffiths, T.~L., Seshia, S.~A., Abbeel, P., and Dragan, A.
\newblock On the {U}tility of {L}earning about {H}umans for {H}uman-{A}{I} {C}oordination.
\newblock In \emph{Proceedings of the 33rd International Conference on Neural Information Processing Systems}, pp.\  5174–5185, Red Hook, NY, USA, 2019. Curran Associates Inc.

\bibitem[Cerny et~al.(2020)Cerny, Lisý, Bošanský, and An]{dinkelbach}
Cerny, J., Lisý, V., Bošanský, B., and An, B.
\newblock Dinkelbach-{T}ype {A}lgorithm for {C}omputing {Q}uantal {S}tackelberg {E}quilibrium.
\newblock In Bessiere, C. (ed.), \emph{Proceedings of the Twenty-Ninth International Joint Conference on Artificial Intelligence, {IJCAI-20}}, pp.\  246--253. International Joint Conferences on Artificial Intelligence Organization, 7 2020.

\bibitem[Cesa-Bianchi et~al.(2017)Cesa-Bianchi, Gentile, Neu, and Lugosi]{CesaBianchi2017BoltzmannED}
Cesa-Bianchi, N., Gentile, C., Neu, G., and Lugosi, G.
\newblock Boltzmann {E}xploration {D}one {R}ight.
\newblock In \emph{NIPS}, 2017.

\bibitem[Chung et~al.(2020)Chung, Luo, Raffin, and Perry]{Chung2020BattlesnakeCA}
Chung, J., Luo, A., Raffin, X., and Perry, S.
\newblock Battlesnake {C}hallenge: {A} {M}ulti-agent {R}einforcement {L}earning {P}layground with {H}uman-in-the-loop.
\newblock \emph{ArXiv}, abs/2007.10504, 2020.

\bibitem[Dinkelbach(1967)]{fractional_prog}
Dinkelbach, W.
\newblock On {N}onlinear {F}ractional {P}rogramming.
\newblock \emph{Management Science}, 13\penalty0 (7):\penalty0 492--498, 1967.

\bibitem[Dockhorn \& Kruse(2017)Dockhorn and Kruse]{DocKru2017a}
Dockhorn, A. and Kruse, R.
\newblock {Combining cooperative and adversarial coevolution in the context of Pac-Man}.
\newblock In \emph{2017 IEEE Conference on Computational Intelligence and Games, CIG 2017}, pp.\  60--67, 2017.
\newblock ISBN 9781538632338.
\newblock \doi{10.1109/CIG.2017.8080416}.

\bibitem[Dockhorn et~al.(2017)Dockhorn, Doell, Hewelt, and Kruse]{DocDoe2017}
Dockhorn, A., Doell, C., Hewelt, M., and Kruse, R.
\newblock {A decision heuristic for Monte Carlo tree search doppelkopf agents}.
\newblock In \emph{2017 IEEE Symposium Series on Computational Intelligence (SSCI)}, pp.\  1--8, November 2017.
\newblock ISBN 978-1-5386-2726-6.
\newblock \doi{10.1109/SSCI.2017.8285181}.

\bibitem[Everett \& Roberts(2018)Everett and Roberts]{Everett2018LearningAN}
Everett, R. and Roberts, S.~J.
\newblock Learning {A}gainst {N}on-{S}tationary {A}gents with {O}pponent {M}odelling and {D}eep {R}einforcement {L}earning.
\newblock In \emph{AAAI Spring Symposia}, 2018.

\bibitem[Fang et~al.(2017)Fang, Nguyen, Pickles, Lam, Clements, An, Singh, Schwedock, Tambe, and Lemieux]{paws}
Fang, F., Nguyen, T.~H., Pickles, R., Lam, W.~Y., Clements, G.~R., An, B., Singh, A., Schwedock, B.~C., Tambe, M., and Lemieux, A.
\newblock {PAWS} — {A} {D}eployed {G}ame‐{T}heoretic {A}pplication to {C}ombat {P}oaching.
\newblock \emph{AI Mag.}, 38\penalty0 (1):\penalty0 23–36, mar 2017.

\bibitem[Fudenberg \& Levine(1998)Fudenberg and Levine]{fudenberg1998theory}
Fudenberg, D. and Levine, D.
\newblock \emph{The Theory of Learning in Games}.
\newblock Economics Learning and Social Evolution Series. MIT Press, 1998.

\bibitem[Herbrich et~al.(2006)Herbrich, Minka, and Graepel]{trueskill}
Herbrich, R., Minka, T., and Graepel, T.
\newblock {TrueSkill\texttrademark : A Bayesian Skill Rating System}.
\newblock In Sch\"{o}lkopf, B., Platt, J., and Hoffman, T. (eds.), \emph{Advances in Neural Information Processing Systems}, volume~19. MIT Press, 2006.

\bibitem[Hofbauer \& Sandholm(2002)Hofbauer and Sandholm]{global_convergence}
Hofbauer, J. and Sandholm, W.~H.
\newblock {On the Global Convergence of Stochastic Fictitious Play}.
\newblock \emph{Econometrica}, 70\penalty0 (6):\penalty0 2265--2294, 2002.

\bibitem[Howard et~al.(2019)Howard, Sandler, Chu, Chen, Chen, Tan, Wang, Zhu, Pang, Vasudevan, Le, and Adam]{Howard2019SearchingFM}
Howard, A.~G., Sandler, M., Chu, G., Chen, L.-C., Chen, B., Tan, M., Wang, W., Zhu, Y., Pang, R., Vasudevan, V., Le, Q.~V., and Adam, H.
\newblock Searching for mobilenetv3.
\newblock \emph{2019 IEEE/CVF International Conference on Computer Vision (ICCV)}, pp.\  1314--1324, 2019.

\bibitem[Hu et~al.(2020)Hu, Lerer, Peysakhovich, and Foerster]{other_play}
Hu, H., Lerer, A., Peysakhovich, A., and Foerster, J.
\newblock {“{O}ther-Play” for Zero-Shot Coordination}.
\newblock In III, H.~D. and Singh, A. (eds.), \emph{Proceedings of the 37th International Conference on Machine Learning}, volume 119 of \emph{Proceedings of Machine Learning Research}, pp.\  4399--4410. PMLR, 13--18 Jul 2020.

\bibitem[Jaderberg et~al.(2017)Jaderberg, Dalibard, Osindero, Czarnecki, Donahue, Razavi, Vinyals, Green, Dunning, Simonyan, Fernando, and Kavukcuoglu]{Jaderberg2017PopulationBT}
Jaderberg, M., Dalibard, V., Osindero, S., Czarnecki, W.~M., Donahue, J., Razavi, A., Vinyals, O., Green, T., Dunning, I., Simonyan, K., Fernando, C., and Kavukcuoglu, K.
\newblock {Population Based Training of Neural Networks}.
\newblock \emph{ArXiv}, abs/1711.09846, 2017.

\bibitem[Jeon et~al.(2022)Jeon, Lee, and Ko]{rl_tron}
Jeon, M., Lee, J., and Ko, S.-K.
\newblock Modular reinforcement learning for playing the game of tron.
\newblock \emph{IEEE Access}, 10:\penalty0 63394--63402, 2022.

\bibitem[Knegt et~al.(2018)Knegt, M.~Drugan, and A.~Wiering]{knegt_opponent_2018}
Knegt, S. J.~L., M.~Drugan, M., and A.~Wiering, M.
\newblock Opponent {Modelling} in the {Game} of {Tron} using {Reinforcement} {Learning}:.
\newblock In \emph{Proceedings of the 10th {International} {Conference} on {Agents} and {Artificial} {Intelligence}}, pp.\  29--40, Funchal, Madeira, Portugal, 2018. SCITEPRESS - Science and Technology Publications.

\bibitem[Lan et~al.(2022)Lan, Zhang, Wu, Tsai, Wu, and Hsieh]{az_attack}
Lan, L.-C., Zhang, H., Wu, T.-R., Tsai, M.-Y., Wu, I.-C., and Hsieh, C.-J.
\newblock {Are AlphaZero-like Agents Robust to Adversarial Perturbations?}
\newblock In Koyejo, S., Mohamed, S., Agarwal, A., Belgrave, D., Cho, K., and Oh, A. (eds.), \emph{Advances in Neural Information Processing Systems}, volume~35, pp.\  11229--11240. Curran Associates, Inc., 2022.

\bibitem[Lanctot et~al.(2013)Lanctot, Wittlinger, Den~Teuling, and Winands]{lanctot13tron}
Lanctot, M., Wittlinger, C., Den~Teuling, N., and Winands, M.
\newblock {Monte Carlo tree search for simultaneous move games: A case study in the game of Tron}.
\newblock \emph{BNAlC 2013}, pp.\  104--111, 01 2013.

\bibitem[Leyton-Brown \& Shoham(2008)Leyton-Brown and Shoham]{leyton_brown_essentials_2008}
Leyton-Brown, K. and Shoham, Y.
\newblock \emph{Essentials of {Game} {Theory}: {A} {Concise}, {Multidisciplinary} {Introduction}}.
\newblock Synthesis {Lectures} on {Artificial} {Intelligence} and {Machine} {Learning}. Springer International Publishing, Cham, 2008.

\bibitem[Li et~al.(2022)Li, Tang, Yang, Hao, Sang, Zheng, Hao, Taylor, and Wang]{Li2022PMICIM}
Li, P., Tang, H., Yang, T., Hao, X., Sang, T., Zheng, Y., Hao, J., Taylor, M.~E., and Wang, Z.
\newblock {PMIC: Improving Multi-Agent Reinforcement Learning with Progressive Mutual Information Collaboration}.
\newblock In Chaudhuri, K., Jegelka, S., Song, L., Szepesvari, C., Niu, G., and Sabato, S. (eds.), \emph{Proceedings of the 39th International Conference on Machine Learning}, volume 162 of \emph{Proceedings of Machine Learning Research}, pp.\  12979--12997. PMLR, 17--23 Jul 2022.

\bibitem[Li et~al.(2023)Li, Hao, Tang, Zheng, and Fu]{raceMARL2023}
Li, P., Hao, J., Tang, H., Zheng, Y., and Fu, X.
\newblock {RACE}: Improve multi-agent reinforcement learning with representation asymmetry and collaborative evolution.
\newblock In \emph{Proceedings of the 40th International Conference on Machine Learning}, volume 202, pp.\  19490--19503. PMLR, 2023.

\bibitem[Lin(1991)]{lin1991divergence}
Lin, J.
\newblock {Divergence measures based on the Shannon entropy}.
\newblock \emph{IEEE Transactions on Information theory}, 37\penalty0 (1):\penalty0 145--151, 1991.

\bibitem[Lisy et~al.(2013)Lisy, Kovarik, Lanctot, and Bosansky]{lisy_convergence_2013}
Lisy, V., Kovarik, V., Lanctot, M., and Bosansky, B.
\newblock Convergence of {Monte} {Carlo} {Tree} {Search} in {Simultaneous} {Move} {Games}.
\newblock In \emph{Advances in {Neural} {Information} {Processing} {Systems}}, volume~26. Curran Associates, Inc., 2013.

\bibitem[Liu et~al.(2009)Liu, He, and He]{mswa}
Liu, H.~X., He, X., and He, B.
\newblock {Method of Successive Weighted Averages (MSWA) and Self-Regulated Averaging Schemes for Solving Stochastic User Equilibrium Problem}.
\newblock \emph{Networks and Spatial Economics}, 9:\penalty0 485--503, 2009.
\newblock \doi{10.1007/s11067-007-9023-x}.

\bibitem[Loshchilov \& Hutter(2017)Loshchilov and Hutter]{adamw}
Loshchilov, I. and Hutter, F.
\newblock {Decoupled Weight Decay Regularization}.
\newblock In \emph{International Conference on Learning Representations}, 2017.

\bibitem[Lou et~al.(2023)Lou, Guo, Zhang, Wang, Huang, and Du]{pecan}
Lou, X., Guo, J., Zhang, J., Wang, J., Huang, K., and Du, Y.
\newblock {PECAN: Leveraging Policy Ensemble for Context-Aware Zero-Shot Human-AI Coordination}.
\newblock In \emph{Proceedings of the 2023 International Conference on Autonomous Agents and Multiagent Systems}, AAMAS '23, pp.\  679–688, Richland, SC, 2023. International Foundation for Autonomous Agents and Multiagent Systems.
\newblock ISBN 9781450394321.

\bibitem[Lupu et~al.(2021)Lupu, Cui, Hu, and Foerster]{lupu21a}
Lupu, A., Cui, B., Hu, H., and Foerster, J.
\newblock {Trajectory Diversity for Zero-Shot Coordination}.
\newblock In Meila, M. and Zhang, T. (eds.), \emph{Proceedings of the 38th International Conference on Machine Learning}, volume 139 of \emph{Proceedings of Machine Learning Research}, pp.\  7204--7213. PMLR, 18--24 Jul 2021.

\bibitem[Maher(1998)]{MAHER1998539}
Maher, M.
\newblock {Algorithms for logit-based stochastic user equilibrium assignment}.
\newblock \emph{Transportation Research Part B: Methodological}, 32\penalty0 (8):\penalty0 539--549, 1998.
\newblock ISSN 0191-2615.
\newblock \doi{https://doi.org/10.1016/S0191-2615(98)00015-0}.

\bibitem[McFadden(1973)]{McFa73}
McFadden, D.
\newblock {Conditional Logit Analysis of Qualitative Choice Behaviour}.
\newblock In Zarembka, P. (ed.), \emph{Frontiers in Econometrics}, pp.\  105--142. Academic Press New York, New York, NY, USA, 1973.

\bibitem[McKelvey \& Palfrey(1995)McKelvey and Palfrey]{qre}
McKelvey, R.~D. and Palfrey, T.~R.
\newblock {Quantal Response Equilibria for Normal Form Games}.
\newblock \emph{Games and Economic Behavior}, 10\penalty0 (1):\penalty0 6--38, 1995.

\bibitem[Metropolis \& Ulam(1949)Metropolis and Ulam]{metropolis_monte_1949}
Metropolis, N. and Ulam, S.
\newblock The {Monte} {Carlo} {Method}.
\newblock \emph{Journal of the American Statistical Association}, 44\penalty0 (247):\penalty0 335--341, September 1949.

\bibitem[Milec et~al.(2021)Milec, {\v C}ern{\'y}, Lis{\'y}, and An]{milec_complexity_2021}
Milec, D., {\v C}ern{\'y}, J., Lis{\'y}, V., and An, B.
\newblock {Complexity and {Algorithms} for {Exploiting} {Quantal} {Opponents} in {Large} {Two}-{Player} {Games}}.
\newblock In \emph{Proceedings of the {AAAI} {Conference} on {Artificial} {Intelligence}}, volume~35, pp.\  5575--5583, May 2021.

\bibitem[Nagurney \& Zhang(1996)Nagurney and Zhang]{nagurney}
Nagurney, A. and Zhang, D.
\newblock \emph{Projected Dynamical Systems and Variational Inequalities with Applications}, volume~2.
\newblock Springer New York, 1996.
\newblock \doi{https://doi.org/10.1007/978-1-4615-2301-7}.

\bibitem[Nash(1951)]{nash_thesis}
Nash, J.
\newblock {Non-Cooperative Games}.
\newblock \emph{Annals of Mathematics}, 54\penalty0 (2):\penalty0 286--295, 1951.

\bibitem[Nashed \& Zilberstein(2022)Nashed and Zilberstein]{Nashed2022ASO}
Nashed, S.~B. and Zilberstein, S.
\newblock {A Survey of Opponent Modeling in Adversarial Domains}.
\newblock \emph{J. Artif. Intell. Res.}, 73:\penalty0 277--327, 2022.

\bibitem[Papoudakis \& Albrecht(2020)Papoudakis and Albrecht]{vae_opp_model}
Papoudakis, G. and Albrecht, S.~V.
\newblock {Variational Autoencoders for Opponent Modeling in Multi-Agent Systems}.
\newblock \emph{CoRR}, abs/2001.10829, 2020.

\bibitem[Pavlidis(1982)]{pavlidis}
Pavlidis, T.
\newblock \emph{Algorithms for Graphics and Image Processing}.
\newblock Springer, 1982.

\bibitem[Polyak(1990)]{polyak}
Polyak, B.~T.
\newblock {New method of stochastic approximation type}.
\newblock \emph{Autom. Remote Control}, 51:\penalty0 937--946, 1990.

\bibitem[Porter et~al.(2008)Porter, Nudelman, and Shoham]{porter_simple_2008}
Porter, R., Nudelman, E., and Shoham, Y.
\newblock {Simple search methods for finding a {Nash} equilibrium}.
\newblock \emph{Games and Economic Behavior}, 63\penalty0 (2):\penalty0 642--662, July 2008.

\bibitem[Punniyamoorthy et~al.(2023)Punniyamoorthy, Abraham, and Thoppan]{nash_selection}
Punniyamoorthy, M., Abraham, S., and Thoppan, J.~J.
\newblock {A Method to Select Best Among Multi-Nash Equilibria}.
\newblock \emph{Studies in Microeconomics}, 11\penalty0 (1):\penalty0 101--127, 2023.
\newblock \doi{10.1177/23210222211024388}.

\bibitem[Reverdy \& Leonard(2015)Reverdy and Leonard]{Reverdy2015ParameterEI}
Reverdy, P.~B. and Leonard, N.~E.
\newblock {Parameter Estimation in Softmax Decision-Making Models With Linear Objective Functions}.
\newblock \emph{IEEE Transactions on Automation Science and Engineering}, 13:\penalty0 54--67, 2015.

\bibitem[Robbins \& Monro(1951)Robbins and Monro]{sfp_convergence}
Robbins, H. and Monro, S.
\newblock {A Stochastic Approximation Method}.
\newblock \emph{The Annals of Mathematical Statistics}, 22\penalty0 (3):\penalty0 400--407, 1951.
\newblock ISSN 00034851.

\bibitem[Rust(2018)]{Rust2018}
Rust, J.
\newblock \emph{Dynamic Programming}, pp.\  3133--3158.
\newblock Palgrave Macmillan UK, London, 2018.

\bibitem[Samothrakis et~al.(2010)Samothrakis, Robles, and Lucas]{samothrakis_uct_2010}
Samothrakis, S., Robles, D., and Lucas, S.~M.
\newblock A {UCT} agent for {Tron}: {Initial} investigations.
\newblock In \emph{Proceedings of the 2010 {IEEE} {Conference} on {Computational} {Intelligence} and {Games}}, pp.\  365--371, August 2010.

\bibitem[Saverino(2011)]{tron_game_length}
Saverino, B.
\newblock {A Monte-Carlo Tree Search for playing Tron}.
\newblock Master's thesis, Montefiore, Department of Electrical Engineering and Computer Science, 2011.

\bibitem[Schier \& Wüstenbecker(2019)Schier and Wüstenbecker]{niedersnake}
Schier, M.~B. and Wüstenbecker, N.
\newblock {Adversarial N-player Search using Locality for the Game of Battlesnake}.
\newblock In \emph{SKILL 2019 - Studierendenkonferenz Informatik}, pp.\  109--120. Gesellschaft für Informatik e.V., Bonn, 2019.

\bibitem[Schrittwieser et~al.(2020)Schrittwieser, Antonoglou, Hubert, Simonyan, Sifre, Schmitt, Guez, Lockhart, Hassabis, Graepel, Lillicrap, and Silver]{schrittwieser_mastering_2020}
Schrittwieser, J., Antonoglou, I., Hubert, T., Simonyan, K., Sifre, L., Schmitt, S., Guez, A., Lockhart, E., Hassabis, D., Graepel, T., Lillicrap, T., and Silver, D.
\newblock Mastering {Atari}, {Go}, chess and shogi by planning with a learned model.
\newblock \emph{Nature}, 588\penalty0 (7839):\penalty0 604--609, December 2020.

\bibitem[Schubert et~al.(2023)Schubert, Benjamins, D{\"o}hler, Rosenhahn, and Lindauer]{polter}
Schubert, F., Benjamins, C., D{\"o}hler, S., Rosenhahn, B., and Lindauer, M.
\newblock {POLTER: Policy Trajectory Ensemble Regularization for Unsupervised Reinforcement Learning}.
\newblock \emph{Transactions on Machine Learning Research}, April 2023.
\newblock ISSN 2835-8856.

\bibitem[Schulman et~al.(2017)Schulman, Wolski, Dhariwal, Radford, and Klimov]{Schulman2017ProximalPO}
Schulman, J., Wolski, F., Dhariwal, P., Radford, A., and Klimov, O.
\newblock Proximal policy optimization algorithms.
\newblock \emph{ArXiv}, abs/1707.06347, 2017.

\bibitem[Sebbane(2001)]{senet_game}
Sebbane, M.
\newblock {Board Games from Canaan in the Early and Intermediate Bronze Ages and the Origin of the Egyptian Senet Game}.
\newblock \emph{Tel Aviv}, 28\penalty0 (2):\penalty0 213--230, 2001.
\newblock \doi{10.1179/tav.2001.2001.2.213}.

\bibitem[Shafiei et~al.(2009)Shafiei, Sturtevant, and Schaeffer]{Shafiei2009ComparingUV}
Shafiei, M., Sturtevant, N.~R., and Schaeffer, J.
\newblock {Comparing UCT versus CFR in Simultaneous Games}, 2009.

\bibitem[Shapley(1964)]{shapley_game}
Shapley, L.~S.
\newblock \emph{1. Some Topics in Two-Person Games}, pp.\  1--28.
\newblock Princeton University Press, Princeton, 1964.
\newblock ISBN 9781400882014.
\newblock \doi{doi:10.1515/9781400882014-002}.

\bibitem[Silver et~al.(2016)Silver, Huang, Maddison, Guez, Sifre, van~den Driessche, Schrittwieser, Antonoglou, Panneershelvam, Lanctot, Dieleman, Grewe, Nham, Kalchbrenner, Sutskever, Lillicrap, Leach, Kavukcuoglu, Graepel, and Hassabis]{silver_mastering_2016}
Silver, D., Huang, A., Maddison, C.~J., Guez, A., Sifre, L., van~den Driessche, G., Schrittwieser, J., Antonoglou, I., Panneershelvam, V., Lanctot, M., Dieleman, S., Grewe, D., Nham, J., Kalchbrenner, N., Sutskever, I., Lillicrap, T., Leach, M., Kavukcuoglu, K., Graepel, T., and Hassabis, D.
\newblock Mastering the game of {Go} with deep neural networks and tree search.
\newblock \emph{Nature}, 529\penalty0 (7587):\penalty0 484--489, January 2016.

\bibitem[Silver et~al.(2017)Silver, Schrittwieser, Simonyan, Antonoglou, Huang, Guez, Hubert, Baker, Lai, Bolton, Chen, Lillicrap, Hui, Sifre, van~den Driessche, Graepel, and Hassabis]{silver_mastering_2017}
Silver, D., Schrittwieser, J., Simonyan, K., Antonoglou, I., Huang, A., Guez, A., Hubert, T., Baker, L., Lai, M., Bolton, A., Chen, Y., Lillicrap, T., Hui, F., Sifre, L., van~den Driessche, G., Graepel, T., and Hassabis, D.
\newblock Mastering the game of {Go} without human knowledge.
\newblock \emph{Nature}, 550\penalty0 (7676):\penalty0 354--359, October 2017.

\bibitem[Silver et~al.(2018)Silver, Hubert, Schrittwieser, Antonoglou, Lai, Guez, Lanctot, Sifre, Kumaran, Graepel, Lillicrap, Simonyan, and Hassabis]{silver_alphazero}
Silver, D., Hubert, T., Schrittwieser, J., Antonoglou, I., Lai, M., Guez, A., Lanctot, M., Sifre, L., Kumaran, D., Graepel, T., Lillicrap, T., Simonyan, K., and Hassabis, D.
\newblock A general reinforcement learning algorithm that masters chess, shogi, and {Go} through self-play.
\newblock \emph{Science}, 362\penalty0 (6419):\penalty0 1140--1144, 2018.

\bibitem[Strouse et~al.(2021)Strouse, McKee, Botvinick, Hughes, and Everett]{NEURIPS2021_797134c3}
Strouse, D., McKee, K., Botvinick, M., Hughes, E., and Everett, R.
\newblock Collaborating with humans without human data.
\newblock In Ranzato, M., Beygelzimer, A., Dauphin, Y., Liang, P., and Vaughan, J.~W. (eds.), \emph{Advances in Neural Information Processing Systems}, volume~34, pp.\  14502--14515. Curran Associates, Inc., 2021.

\bibitem[Tak et~al.(2014)Tak, Lanctot, and Winands]{sm_mcts_variants}
Tak, M. J.~W., Lanctot, M., and Winands, M. H.~M.
\newblock {Monte Carlo Tree Search variants for simultaneous move games}.
\newblock In \emph{2014 IEEE Conference on Computational Intelligence and Games}, pp.\  1--8, 2014.

\bibitem[Yang et~al.(2012)Yang, Ord{\'o}{\~n}ez, and Tambe]{Yang2012ComputingOS}
Yang, R., Ord{\'o}{\~n}ez, F., and Tambe, M.
\newblock {Computing optimal strategy against quantal response in security games}.
\newblock In \emph{Adaptive Agents and Multi-Agent Systems}, 2012.

\bibitem[Yu et~al.(2023)Yu, Gao, Liu, Xu, Tang, Yang, Wang, and Wu]{hsp}
Yu, C., Gao, J., Liu, W., Xu, B., Tang, H., Yang, J., Wang, Y., and Wu, Y.
\newblock Learning zero-shot cooperation with humans, assuming humans are biased.
\newblock In \emph{The Eleventh International Conference on Learning Representations}, 2023.

\bibitem[Yu et~al.(2022)Yu, Jiang, Zhang, Jiang, and Lu]{yu2022modelbased}
Yu, X., Jiang, J., Zhang, W., Jiang, H., and Lu, Z.
\newblock {Model-Based Opponent Modeling}.
\newblock In Oh, A.~H., Agarwal, A., Belgrave, D., and Cho, K. (eds.), \emph{Advances in Neural Information Processing Systems}, 2022.

\bibitem[Zhao et~al.(2023)Zhao, Song, Yuan, Hu, Gao, Wu, Sun, and Yang]{mep}
Zhao, R., Song, J., Yuan, Y., Hu, H., Gao, Y., Wu, Y., Sun, Z., and Yang, W.
\newblock {Maximum Entropy Population-Based Training for Zero-Shot Human-AI Coordination}.
\newblock \emph{Proceedings of the AAAI Conference on Artificial Intelligence}, 37\penalty0 (5):\penalty0 6145--6153, Jun. 2023.

\end{thebibliography}

\newpage
\appendix
\onecolumn

\appendix

\section{Training Procedures of AlphaZero and Albatross}
\label{sec:az_alb_training}
In this section, we outline the training procedure of Albatross and highlight the differences to the original AlphaZero algorithm.
As an intuition, the proxy model learns to play like agents of different rationality, but only through self-play without external supervision (e.g. imitation learning).
Afterwards, the response model learns to play optimally against the proxy model at all rationalities.
In contrast, AlphaZero only performs self-play, only learning to play optimally against the own highly rational policy.

Therefore, the training of AlphaZero is very similar to the training of the proxy model, but the neural networks are not conditioned on the temperature. 
The main differences in the algorithms are highlighted in bold.

\subsection{AlphaZero}
In the original AlphaZero algorithm \cite{silver_alphazero}, Monte-Carlo tree search (MCTS) is used as a policy improvement operator.
For sequential games this is a good choice as it allows for deep searches by only focusing on promising regions in the search space.
However, in simultaneous games, the best action is highly dependent on the actions of the other players, which nullifies this advantage.
This is reflected in our hyperparameter search (see \cref{sec:az_sim}), where we found that replacing the standard MCTS with fixed depth search and Logit equilibrium backup yields the best performance.

\begin{algorithm}[ht]
   \caption{Training of AlphaZero for simultaneous games}
   \label{alg:training_alphazero}
\begin{algorithmic}[1]
\STATE \textbf{Input:} Observation space $o \in \mathcal{O}$, maximum rationality $\tau_{max}$
\STATE \textbf{Output:} Trained models $\pi_{\theta}(o_i)$ and $v_{\theta}(o_i)$
\STATE Initialize models $\pi_{\theta}$ and $v_{\theta}$ with random weights
\STATE Initialize replay buffer $B$
\FOR{each training episode}
    \FOR{each step of the episode}
        \STATE Perform fixed depth search to construct NFGs and compute LEs:
        \STATE \quad 1. At each leaf node, evaluate states via $v_{\theta}(o_{\text{leaf},i})$ and construct NFG from sibling nodes.
        \STATE \quad 2. Use a solver to compute the Logit equilibrium with temperature $\tau_{max}$ for the NFG
        \STATE \quad 3. Propagate the expected utility of the LE to the parent node
        \STATE \quad 4. Repeat until the LE at the root node is computed
        \STATE Let $o_i$ be the state observation and $\pi_{\text{LE}}(o_i)$, $v_{\text{LE}}(o_i)$ be policy and value of the LE at the root node for each player $i$
        \STATE Collect experiences $(o_i, \pi_{\text{LE}}(o_i), v_{\text{LE}}(o_i))$ at this step and add to $B$
        \STATE Perform environment step by sampling actions $a_i$ from $\pi_{\text{LE}}(o_i)$ for each player $i$
    \ENDFOR
    \FOR{each minibatch sampled from $B$}
        \STATE Update $\pi_{\theta}$ and $v_{\theta}$ via CrossEntropy and MSE respectively
    \ENDFOR
\ENDFOR
\end{algorithmic}
\end{algorithm}

The detailed training procedure is given in \cref{alg:training_alphazero}.
During each training step, a complete search tree is constructed up to a fixed depth.
Then, the leaf nodes are evaluated by the critic network $v_{\theta}$, which represent the utilities of Normal-form games (NFG) constructed from the direct sibling nodes.
The Logit equilibrium of these NFGs is computed using an equilibrium solver (for details see \cref{sec:game_theory_alg}).
The equilibria yield expected utilities for each player in the parent nodes, which again are used to create NFGs from sibling nodes.
This process is repeated until the root node is reached.
The Logit equilibrium at the root node yields an improved policy and value estimate for each player $i$ at the current game state compared to the original policy and value prediction at this game state.
These target values are added to the replay buffer and used later for gradient updates.

\subsection{Albatross}

The training of Albatross is a two-stage process.
Firstly, a proxy model learns to imitate the behavior of agents at different rationality.
Afterwards, a response model learns to play a smooth best response to the proxy model.
The training process of the proxy model is very similar to the training of AlphaZero outlined above.
This procedure is visualized in \cref{alg:training_proxy}.
The major distinction is that policy and value networks are conditioned on a scalar temperature parameter, which controls the rationality of the proxy model.
To train the proxy model on a range of different temperatures, a temperature is sampled at the beginning of each training episode. 
During each step in the episode, again Logit equilibria are computed during fixed depth search for policy and value improvement.
In contrast to the training of AlphaZero, the temperature sampled at the beginning of the episode is used for computing the Logit equilibria.
Additionally, the temperature is also added to the replay buffer as they are necessary for the gradient updates.

\begin{algorithm}[ht]
   \caption{Training of Proxy Model in Albatross (Differences to AlphaZero are highlighted in \textcolor{blue}{blue})}
   \label{alg:training_proxy}
\begin{algorithmic}[1]
\STATE \textbf{Input:} Observation space $o \in \mathcal{O}$, \textcolor{blue}{temperature range $[\tau_{min}, \tau_{max}$, temperature distribution $p(\tau)$}
\STATE \textbf{Output:} \textcolor{blue}{Trained proxy models $\pi_{\theta_P}(o_i, \tau)$ and $v_{\theta_P}(o_i, \tau)$} 
\STATE Initialize proxy models $\pi_{\theta_P}$ and $v_{\theta_P}$ with random weights
\STATE Initialize replay buffer $B$
\FOR{each training episode}
    \STATE \textcolor{blue}{Sample temperature $\tau$ from $p(\tau)$ within $[\tau_{min}, \tau_{max}]$}
    \FOR{each step of the episode}
        \STATE Perform fixed depth search to construct NFGs and compute LEs:
        \STATE \quad 1. At each leaf node, evaluate states via $v_{\theta_P}(o_{\text{leaf},i}, \tau)$ and construct NFG from sibling nodes.
        \STATE \quad 2. Use a solver to compute the Logit equilibrium with \textcolor{blue}{temperature $\tau$} for the NFG
        \STATE \quad 3. Propagate the expected utility of the LE to the parent node
        \STATE \quad 4. Repeat until the LE at the root node is computed
        \STATE Let $o_i$ be the state observation and $\pi_{\text{LE}}(o_i, \tau)$, $v_{\text{LE}}(o_i, \tau)$ be policy and value of the LE at the root node for each player $i$
        \STATE Collect experiences $(o_i, \pi_{\text{LE}}(o_i, \tau), v_{\text{LE}}(o_i, \tau), \text{\textcolor{blue}{temperature $\tau$}})$ at this step and add to $B$
        \STATE Perform environment step by sampling actions $a_i$ from $\pi_{\text{LE}}(o_i, \tau)$ for each player $i$
    \ENDFOR
    \FOR{each minibatch sampled from $B$}
        \STATE Update $\pi_{\theta_P}$ and $v_{\theta_P}$ via CrossEntropy and MSE respectively
    \ENDFOR
\ENDFOR
\end{algorithmic}
\end{algorithm}

\begin{algorithm}[ht]
   \caption{Training of Response Model in Albatross (Differences to Proxy are highlighted in \textcolor{blue}{blue})}
   \label{alg:training_response}
\begin{algorithmic}[1]
\STATE \textbf{Input:} Observation space $o \in \mathcal{O}$, temperature range $[\tau_{min}, \tau_{max}]$, temperature distribution $p(\tau)$, \textcolor{blue}{fixed response temperature $\tau_R$, proxy policy $\pi_{\theta_P}(o_i, \tau)$}
\STATE \textbf{Output:} \textcolor{blue}{Trained response models $\pi_{\theta_R}(o_i, \tau_{-i})$ and $v_{\theta_R}(o_i, \tau_{-i})$} 
\STATE Initialize response models $\pi_{\theta_R}$ and $v_{\theta_R}$ with random weights
\STATE Initialize replay buffer $B$
\FOR{each training episode}
    \STATE \textcolor{blue}{Sample temperatures $\tau_{-i}$ from $p(\tau)$ within $[\tau_{min}, \tau_{max}]$ for each player except $i$}
    \FOR{each step of the episode}
        \STATE Perform fixed depth search to approximate an \textcolor{blue}{SBRLE}:
        \STATE \quad 1. At each leaf node, evaluate states via $v_{\theta_R}(o_{\text{leaf},i}, \tau_{-i})$ and construct NFG from sibling nodes.
        \STATE \quad 2. \textcolor{blue}{At the parent node evaluate the proxy policy for every other agent j as $\pi_{\theta_P}(o_{\text{parent}, j}, \tau_j)$}
        \STATE \quad 3. \textcolor{blue}{Compute the Smooth Best Response (SBR) to LE policies with response temperature $\tau_R$}
        \STATE \quad 4. Repeat until the \textcolor{blue}{SBR to the LE of the other player} at the root node is computed
        \STATE Let $o_i$ be the state observation and $\pi_{\text{SBRLE}}(o_i, \tau_{-i})$, $v_{\text{SBRLE}}(o_i, \tau_{-i})$ be policy and value of the SBRLE at the root node for player $i$
        \STATE Collect experiences $(o_i, \pi_{\text{SBRLE}}(o_i, \tau_{-i}), v_{\text{SBRLE}}(o_i, \tau_{-i}), \text{\textcolor{blue}{temperatures $\tau_{-i}$}})$ at this step and add to $B$
        \STATE Perform environment step by sampling actions $a_i$ from $\pi_{\text{SBRLE}}(o_i, \tau)$ for player $i$ \textcolor{blue}{and from the proxy policy $\pi_{\theta_P}(o_j, \tau_j)$ for every other agent $j$}
    \ENDFOR
    \FOR{each minibatch sampled from $B$}
        \STATE Update $\pi_{\theta_R}$ and $v_{\theta_R}$ via CrossEntropy and MSE respectively
    \ENDFOR
\ENDFOR
\end{algorithmic}
\end{algorithm}

After training the proxy model, the response model is trained using the policy of the proxy model.
The policy and value network of the response model are not only conditioned on a single temperature, but a temperature for every player except itself.
Therefore, in a game with $n$ players, the response model is conditioned on $n - 1$ scalar temperatures.
The response model approximates the Smooth Best Response (SBR) to the policy of the proxy model.
Since the response model should be trained on all combination of rationalities for the other agent, $n-1$ temperatures are sampled at the beginning of an episode.
At each training step, a policy and value improvement is achieved by computing the SBRLE with tree search up to a fixed depth.
At first, the leaf nodes of the search tree are evaluated using the value network $v_{\theta_R}$, which form NFGs.
The policy of the weak players in the NFG can be computed by using the proxy policy $\pi_{\theta_P}$ in the parent node with the respective temperature sampled at the beginning of the episode.
Then, the SBR is computed by a softmax transformation with fixed response temperature $\tau_R$ on the expected utilities for player $i$.
Similar to the training of AlphaZero and the Proxy model, the SBRLE is propagated to the root node and used for gradient updates.
In contrast to the training of the proxy model, all $n-1$ temperatures are added to the replay buffer.
Another difference to AlphaZero and the proxy model is the policy used for advancing the environment state between training steps.
Only the action for the agent controlled by the response model is sampled from the SBRLE computed at the root node.
All other actions are sampled from the proxy policy to accurately represent the state distribution when playing weak agents.

\section{Algorithms for equilibrium computation}
\label{sec:game_theory_alg}
There exist a number of algorithms for computing the equilibria presented.
We give a brief overview of the algorithms used in this work.
All algorithms are implemented in C++ and available open source along with our code.
To the best of our knowledge, this is the first open-source implementation of solvers for Quantal Stackelberg equilibria.

\subsection{Nash Equilibrium}
For the computation of Nash equilibria, we use the algorithm of Porter et al. \yrcite{porter_simple_2008}, which is based on support enumeration.
A support is the set of actions receiving a non-zero probability in the Nash equilibrium.
Until an equilibrium is found, the supports are iterated and a linear program is solved to determine if a Nash equilibrium exists for the current support.
Supports are ordered based on a heuristic prioritizing small and balanced supports.
For games of more than two players, a non-linear program is solved.
For details regarding the formulation of the (non-) linear program, we refer to the original work.

\subsection{Logit Equilibrium}
\label{subsec:le_details}
A Logit equilibrium can be computed by using the smooth best response dynamics.
In detail, starting from a uniform policy, all players compute the smooth best response given the other players policy.
These new policies are the basis for the computation of smooth best responses in the next iteration.
This process is called \emph{Stochastic Fictitious Play (SFP)} \cite{global_convergence} (sometimes also titled Smooth Fictitious Play \cite{fudenberg1998theory}).
Simply following the dynamics yields a fixed point in two-player zero-sum games, but can form a cycle in other games \cite{shapley_game}.
Therefore, one has to anneal the step size for updating the policy of each player.
Formally, in iteration $t$ the policies are updated as $\pi_i^{t+1} = \pi_i^t + \alpha_t(\mathit{SBR}(\pi_{-i}^t, \tau) - \pi_i^t)$.
Robbins and Monro \yrcite{sfp_convergence} proved that SFP converges almost surely to the equilibrium point for the step sizes $(\alpha_1, \alpha_2, \ldots)$ if the conditions $\lim_{t\to\infty} \alpha_t = 0$ and $\sum_{t=1}^{\infty} \alpha_t = \infty$ hold.
We compare multiple step size schedules, which all fulfill the two mentioned conditions.
\begin{itemize}
    \item The method of successive averages (MSA) \cite{sfp_convergence} updates the policy as an average of all previous policies, which is equivalent to using a step size of $\alpha_t = 1 / t$.
    \item Polyak \yrcite{polyak} proposed to use a step size of $\alpha_t = t^{-2/3}$.
    \item Nagurney and Zhang \yrcite{nagurney} proposed a schedule of learning rates which converges to zero at a slow rate: $(1, \frac{1}{2}, \frac{1}{2}, \frac{1}{3}, \frac{1}{3}, \frac{1}{3}, \ldots, t \text{ times } \frac{1}{t})$. 
    \item Self-regulating average (SRA) \cite{mswa} is a case-based schedule accelerating the decay if learning diverges and decelerating the decay when converging. The learning rate factors are $\alpha_t = \frac{1}{\beta_t}$ with
    \begin{align*}
        \beta_t = \begin{cases}
            \beta_{t-1} + \gamma & \norm{\pi_i^{t} - \mathit{SBR}(\pi_{-i}^t, \tau)} \geq \norm{\pi_i^{t-1} - \mathit{SBR}(\pi_{-i}^{t-1}, \tau)}\\
            \beta_{t-1} + \Gamma & \textit{ else }
        \end{cases},
    \end{align*}
    where $\gamma < 1$ and $\Gamma > 1$. For our experiments, we adopted the hyperparameters $\gamma = 0.3$ and $\Gamma = 1.8$ from Liu et al. \yrcite{mswa}.
\end{itemize}

\begin{figure}[ht]
\begin{center}
    \subfigure[General Sum]{%
        \includegraphics[width=0.24\textwidth]{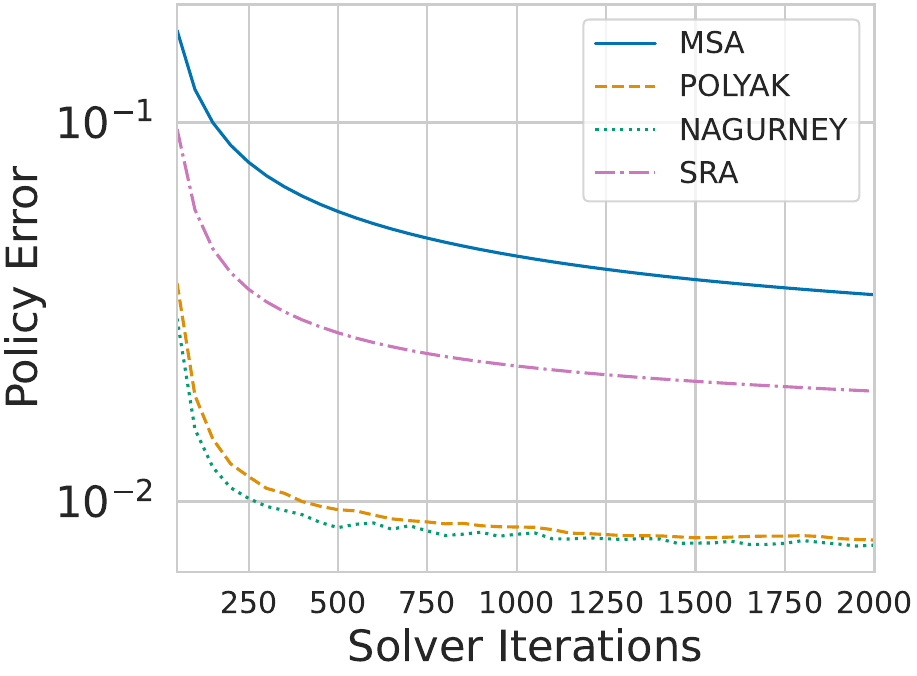}
    }
    \subfigure[Fully Cooperative]{%
        \includegraphics[width=0.24\textwidth]{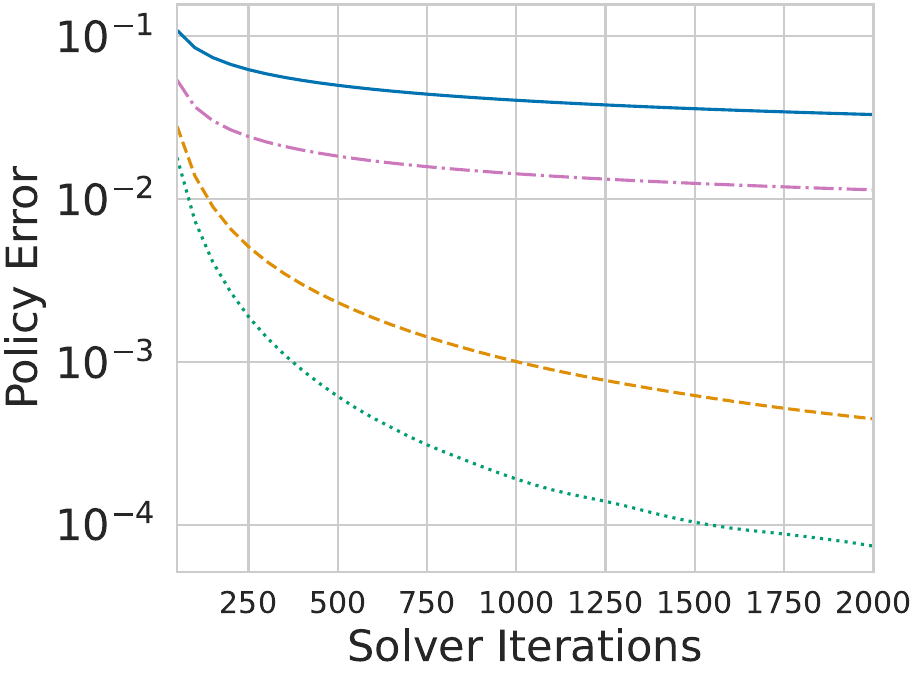}
    }
    \subfigure[Zero-Sum Policy Error]{%
        \includegraphics[width=0.24\textwidth]{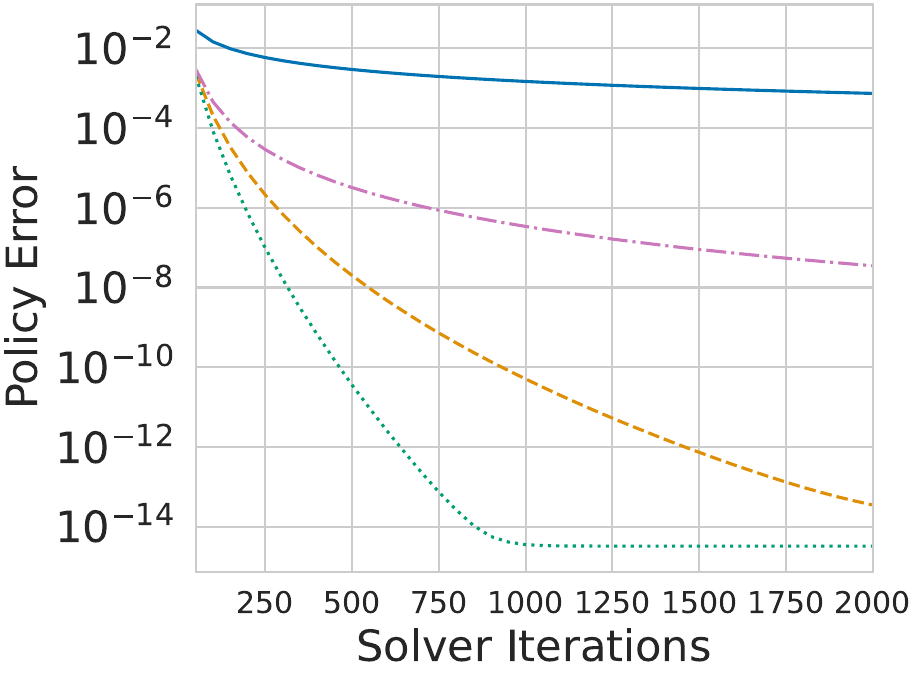}
    }
    \subfigure[Zero-Sum Value Error]{%
        \includegraphics[width=0.24\textwidth]{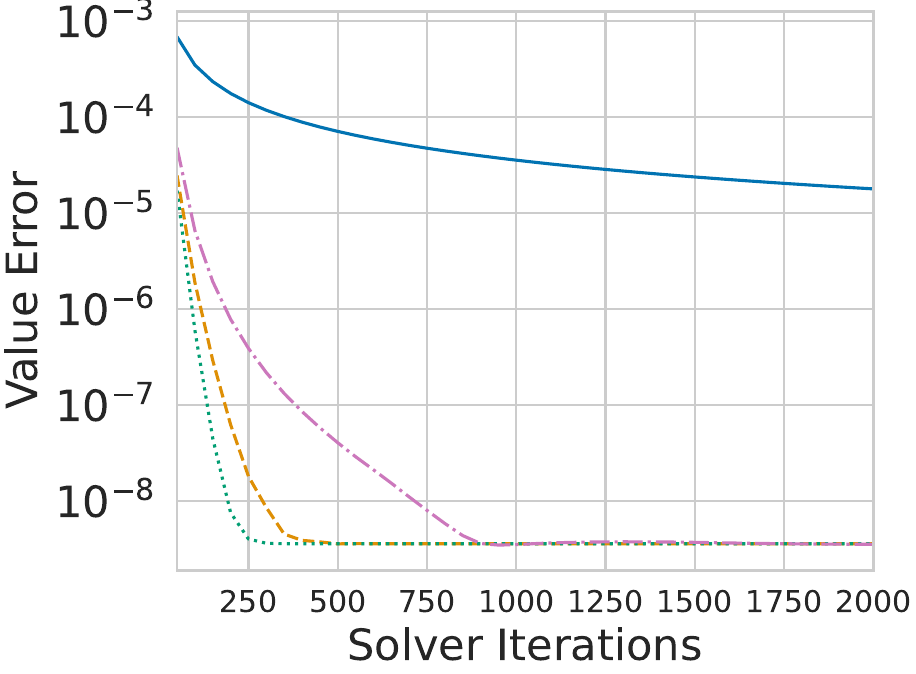}
    }
    \caption{
    Mean policy and value error of stochastic fictitious play using different learning rate schedules in $10^5$ randomly generated NFGs with different game theoretic properties.
    For the value error, ground truth values are computed using MSA with $10^7$ iterations.
    Temperatures of the Logit equilibria are randomly sampled in the interval $[0, 10]$.
    }
    \label{fig:logit_solver}
\end{center}
\end{figure}

In \cref{fig:logit_solver}, we test the different learning rate schedules in randomly generated NFGs.
For the random generation, we uniformly sample utilities of a 2-player NFG with 6 actions per agent.
To test different game theoretic properties, we additionally perform experiments with normalized utilities according to the fully cooperative or zero-sum property.
Then, we approximate the Logit equilibrium using a different budget of solver iterations.
The policy error is calculated as the absolute difference in policy between two steps of SFP, i.e $|\mathit{SBR}(\pi_{-i}^t, \tau) - \pi_i^t|$.
Additionally, we compute the value error in the zero-sum games, as they always have a unique Logit equilibrium \cite{global_convergence}.
In all experiments, the learning rate schedule of Nagurney and Zhang \yrcite{nagurney} performed best.

\begin{figure}[ht]
\begin{center}
    \subfigure[$\tau = 0$]{%
        \includegraphics[width=0.3\textwidth]{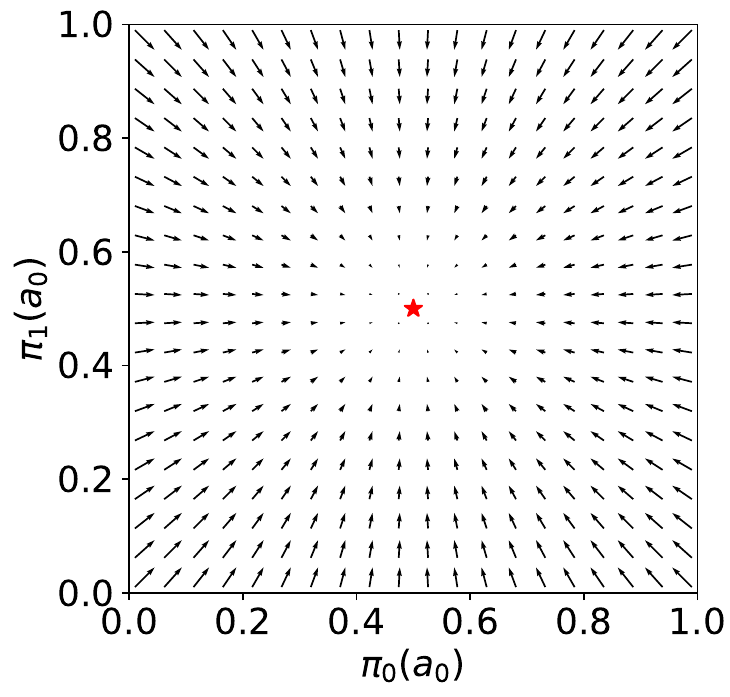}
    }
    \subfigure[$\tau = 1$]{%
        \includegraphics[width=0.3\textwidth]{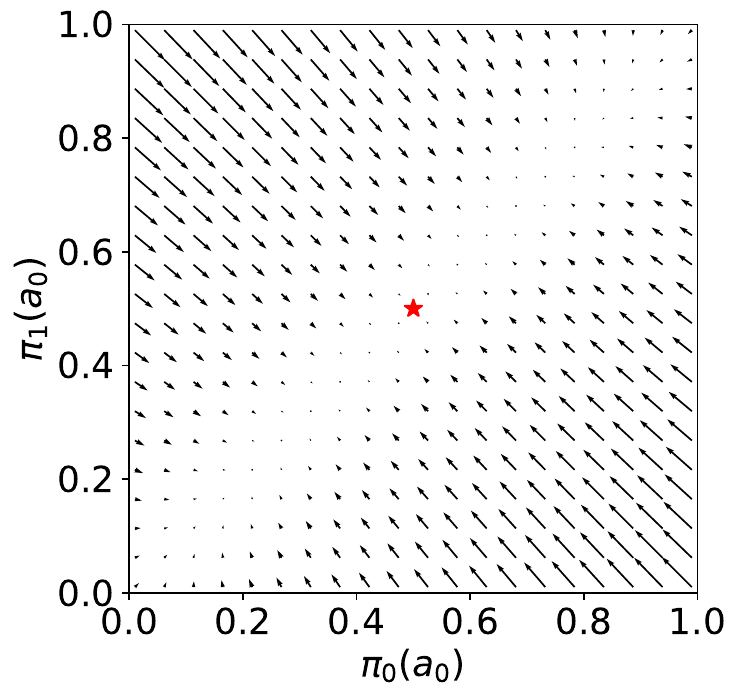}
    }
    \subfigure[$\tau = 10$]{%
        \includegraphics[width=0.3\textwidth]{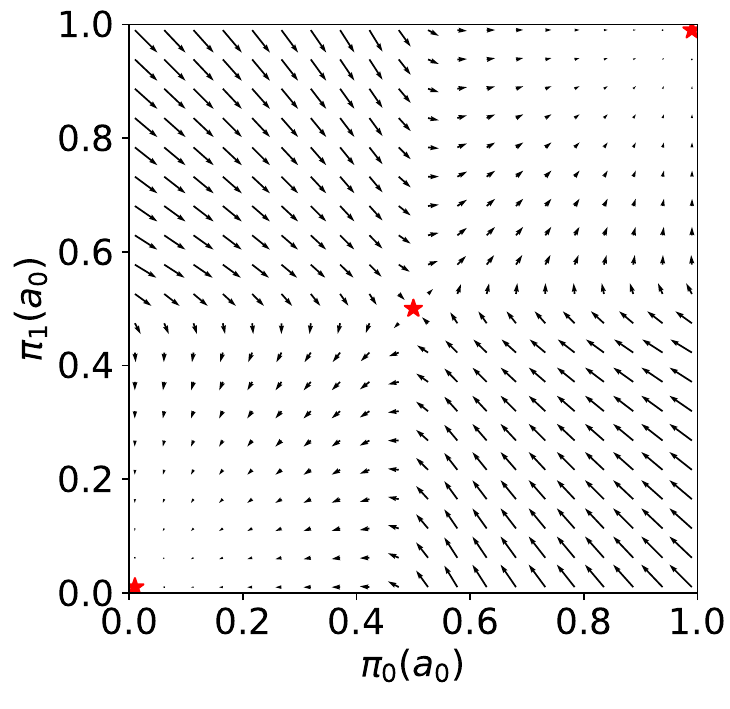}
    }
    \caption{
    Examples of the smooth best response dynamics at different temperatures for the matching pennies game. Both player have two actions and get a reward of +1 if they choose the same action and -1 otherwise. The Logit equilibria are marked in red.
    }
    \label{fig:sbr_dynamics_example}
\end{center}
\end{figure}

In \cref{fig:sbr_dynamics_example}, an example of the smooth best response dynamics is visualized at different temperatures for the matching pennies game.
In the matching pennies game, both player have two actions and get a reward of +1 if they choose the same action and -1 otherwise.
The Logit equilibrium is always unique for sufficiently small temperatures, but does not have to be for larger temperatures if multiple Nash equilibria exist \cite{qre}.
In this example, at $\tau=0$ and $\tau=1$, the LE is unique, but at $\tau=10$ multiple LE exist, which approximate the three Nash equilibria of the game.
Note that depending on the initialization, SFP will converge to different LE.

\subsection{Quantal Stackelberg equilibrium}
The computation requires finding a global optimum.
We utilize a Dinkelbach-Type algorithm\footnote{There are two errors in the original paper. In Algorithm 1, line 4 the subtraction needs to be an addition, and in line 7, the if and else cases need to be swapped.} \cite{dinkelbach}, which relies on fractional programming \cite{fractional_prog}.
The definition of a QSE can be reformulated as:
\begin{align*}
    \pi_i = \argmax_{\tilde{\pi}_i \in \Delta_i} \frac
    {
    \sum_{a \in A_{-i}} u_i(\tilde{\pi}_i, a) \exp(\tau\, u_i(\tilde{\pi}_i, a))
    }{
    \sum_{a \in A_{-i}} \exp(\tau \, u_i(\tilde{\pi}_i, a))
    }.
\end{align*}
The primary notion of fractional programming is the transformation of the problem $\max_x f(x) / g(x)$ into a different problem $F(p) = \max_x f(x) - p g(x)$, maximizing the original problem at the root $F(p^*) = 0$.
Since $p$ is a scalar and $F$ convex, one can find the global optimum using simple binary search.
In each iteration, the binary search solves the Dinkelbach subproblem
\begin{align*}
    \max_{\tilde{\pi}_i \in \Delta_i} \sum_{a \in A_{-i}} \big[ u_i(\tilde{\pi}_i, a) - p \big] \exp(\tau \, u_i(\tilde{\pi}_i, a)).
\end{align*}
Solving the subproblem requires finding a global optimum of a simpler problem than the original formulation, but the solution is still difficult to compute.
In our games with small action spaces, it was sufficient to approximate the global optimum using grid search.
Notably, there exist other methods using piece-wise linear approximation or gradient descent \cite{dinkelbach}.

\section{AlphaZero in Competitive Simultaneous Games}
\label{sec:az_sim}
The only component requiring change when adapting AlphaZero to simultaneous games is the tree search algorithm.
We evaluate three different tree search variants, namely Monte-Carlo Tree Search (MCTS) \cite{metropolis_monte_1949}, Counterfactual Regret Minimization (CFR) and fixed depth search.
MCTS is intended for sequential games, but there also exists an adaptation for simultaneous games, namely Simultaneous-Move Monte-Carlo Tree Search (SM-MCTS) \cite{sm_search}.
We refer to it simply as MCTS since only considering simultaneous games renders the distinction redundant.
For MCTS, we evaluate three different selection functions.

The \emph{Decoupled Upper Confidence bound for Trees (DUCT)} uses the standard Upper Confidence bound for trees like AlphaZero \cite{silver_alphazero}, but every player keeps an independent statistic of action-values and action-visits.
Specifically, each agent selects a move according to
\begin{align*}
    a^* = \argmax_{a_i \in A_i} \frac{w_{a_i}}{n_{a_i}} + c\, \pi_{\theta}(o_i, a_i) \frac{\sqrt{N}}{n_{a_i}},
\end{align*}
where $w_a$ is the sum of values propagated through the current node in the backup phase and $n_{a_i}$ the number of times agent $i$ selected action $a_i$.
$N$ denotes the total number of visits in the current node and c is a parameter balancing exploration and exploitation.
The policy $\pi_{\theta}(o_i, a_i)$ is used to guide the exploration of the tree search.
This adaptation is simple to implement and has been shown to work well in a variety of games \cite{lanctot13tron, sm_mcts_variants}.
However, it has also been shown that DUCT does not always converge to a Nash equilibrium \cite{Shafiei2009ComparingUV}.
This is, because the action selection is deterministic and independently of the other agents.
As a result, it is possible that all agents select the same actions over and over, leading to a cycle around the true Nash equilibrium.
This problem can be alleviated by using a random tie break between actions with the same upper confidence bound.
However, even with this extension, convergence cannot be guaranteed.

In contrast to DUCT, the \emph{Exponential Weight Algorithm for Exploration and Exploitation (EXP3)} \cite{Auer2002TheNM} is a stochastic selection algorithm.
Again, all players select their action independently from each other, but they sample their action from the following distribution
\begin{align*}
    \sigma(a_i) &= \frac{(1 - \gamma) \exp(\eta w_{a_i})}{\sum_{\tilde{a_i} \in A_i} \exp(\eta w_{\tilde{a_i}})} + \frac{\gamma}{|A_i|} \\
    &= \frac{(1 - \gamma)}{\sum_{\tilde{a_i} \in A_i} \exp(\eta (w_{a_i} - w_{\tilde{a_i}}))} + \frac{\gamma}{|A_i|},
\end{align*}
where $\eta = \gamma / |A|$ and $\gamma$ is an exploration parameter.
The second formula is numerically more stable, because it avoids the computation of large exponential terms.
In addition to the selection function, EXP3 also slightly alters the computation of the backup function:
\begin{align*}
   &w_{a_i} \leftarrow w_{a_i} + \frac{\Tilde{w}}{\sigma(a_i)}.
\end{align*}
The outcome of the current evaluation is scaled by the probability that an action is taken to account for very unlikely events.
The final resulting policy of EXP3 is the average of the sample probabilities over all iterations.
In contrast to DUCT, EXP3 has been proven to converge to a Nash equilibrium in Normal-form two-player zero-sum games \cite{Auer2002TheNM}.

In contrast to DUCT and EXP3, using \emph{Regret Matching} an action is not selected according to the expected outcome of that action, but proportional to the expected regret of not choosing that action.
Similar to EXP3, Regret Matching is a stochastic algorithm which samples an action from the distribution
\begin{align*}
    &\sigma(a) = \begin{cases}
    r^{+}_{a} / R^{+} &R^{+} > 0 \\
    1 / |A| &\text{else}
    \end{cases},
    &R^{+} = \sum_{\tilde{a} \in A} r^{+}_{\tilde{a}},
\end{align*}
where $(\cdot)^+ = \max(0, \cdot)$.
The regret values $r_a$ are computed during the backward pass for all actions that are not selected.
To compute the regret, the algorithm needs to keep track of the average outcome of all joint actions $a$, not just the individual actions $a_i$ like DUCT and EXP3.
Like Exp3, Regret Matching always converges to a Nash equilibrium in Normal-form two-player zero-sum games \cite{lisy_convergence_2013}.

Counterfactual Regret Minimization (CFR) \cite{deep_cfr} is a tree search for games with imperfect-information.
Since games with imperfect information are a superset of simultaneous perfect-information games, CFR can be applied.
However, CFR is unnecessarily complicated and inefficient, because the only imperfect information arises from the simultaneous move selection.
Specifically, a simultaneous perfect-information game can also be modeled as a sequential imperfect-information game, where agents do not know the selected move of the other agents.
This leads to the property, that the value of a node in the search tree only depends on the subtree below, but not on the previous game dynamics.
An algorithm exploiting this property is \emph{Simultaneous Move Online Outcome Sampling (SM-OOS)} \cite{sm_search}.
Because CFR relies on Regret Matching, SM-OOS is very similar to MCTS with Regret Matching as a selection function.
However, there are a few important differences.
In SM-OOS, only one player updates their regret values in each iteration, while in MCTS with Regret Matching all players update their regrets.
The updating agent explores the state space by choosing a random action with probability $\epsilon$ or playing on policy with probability $1 - \epsilon$.
The non-updating agents play on policy to ensure that the regret calculation of the updating player is correct.
Additionally, the updating agent keeps track of its tail and sampling probability.
The tail probability is the product of the policy action probabilities of all nodes from the current node to the leaf node on the path taken during action selection.
The sampling probability is very similar, but uses the action probabilities including exploration instead of the plain policy.
By weighting the regret updates with the ratio of tail and sampling probability, one can ensure that the computed regret accurately reflects the regret that would occur when the agent plays on policy.
Lastly, only the non-updating player adds their current policy estimates to the cumulative policy sum at each node to prevent a mixture with the exploration probability.
Even though the calculation with a single updating agent is more accurate than the computation in MCTS, it is also less efficient as less updates happen in the same computation time.

Fixed depth search, also called Backward induction \cite{Rust2018}, is an algorithm originally intended for solving a complete game tree, but it is also possible to use the algorithm on a truncated game tree with a heuristic evaluation in the leaf nodes \cite{sm_search}.
Firstly, the game tree is built up to a specific depth. 
Then, all leaf nodes are evaluated.
Lastly, the values of the leaf nodes are propagated upward the game tree to the root node using a backup function.
We test two backup functions, which are based on the idea of solving for an equilibrium.
Using the game theoretic algorithms presented in \cref{sec:game_theory_alg}, we test a backup function based on the Nash equilibrium as a Logit equilibrium with fixed temperature ($\tau = 10$).

\subsection{Baseline Agent}
\label{sec:baseline_agent}
In Battlesnake, area control is a standard heuristic for evaluating a game state \cite{niedersnake}.
In the stochastic game modes, area control implicitly incentivizes a snake to eat more food than the enemy, because it can reach more squares if it is able to win a head-to-head collision.
We compute area control using a flood-fill algorithm \cite{pavlidis}, which fills the board starting from the heads of all living snakes.
If two snakes are able to reach a grid square at the same time, we use the length as a tie break according to the head-to-head collision rule.
In the stochastic game modes, our variant of flood fill also dynamically deletes the current tail of all snakes in every iteration.
This simulates the game dynamics as every snake would move forward and leave the square occupied by its tail.
The only exception is a situation, where a snake has just eaten a food in the last turn.
Then, the tail of this snake stays on its square in the first iteration of flood fill and is only deleted in the second iteration.

As an additional improvement, we combine the area control of each snake with their relative health score to prevent them from starving.
Let $b^2$ be the total number of grid squares, $\tilde{N}$ the current agents alive, $h_{\mathit{max}}$ the maximum health, $h_i$ the health of agent $i$ and $\alpha_i$ their area control computed as described above.
Then, we evaluate a board position for player $i$ as
\begin{align*}
    \tilde{w}_i = \frac{1}{2}(\tilde{\alpha}_i + \tilde{h}_i - \frac{1}{|\tilde{N}|} \sum_{j \in \tilde{N}} \tilde{h}_j),
\end{align*}
where $\tilde{\alpha}_i$ is the area control advantage of agent $i$ relative to the board size and $\tilde{h}_i$ their relative health.
Specifically, these values are computed as
\begin{align*}
    \tilde{\alpha}_i &= \frac{1}{b^2}\big(\alpha_i - \frac{1}{|\tilde{N}|} \sum_{j \in \tilde{N}} \alpha_j\big),
    &\tilde{h}_i = \frac{h_i}{h_{\mathit{max}}}.
\end{align*}

For Tron, we omit the terms using a health score and simply evaluate a game state using $\tilde{\alpha}_i$.
To compute a policy from the heuristic value function, we utilize MCTS with DUCT as a selection function. 
For DUCT, we use the standard exploration bonus of $c = \sqrt{2}$.
Since the baseline agent does not have a trained policy model to guide the search, we omit the policy guidance term in DUCT.

\subsection{Evaluation}
We evaluate the different tree search variants in the three game modes of Battlesnake against the baseline agent using $2e3$ search iterations.
For all game modes, we train AlphaZero with the adapted tree search on five seeds.
The game of Tron has short episodes and a smaller state space, such that we limit training time to a single day.
In contrast, for the stochastic game modes, we train AlphaZero for two days.

\newcommand{\curfigwidth}{0.31\textwidth}
\begin{figure*}[ht]
    \begin{center}
        \subfigure[Tron]{%
            \includegraphics[width=\curfigwidth]{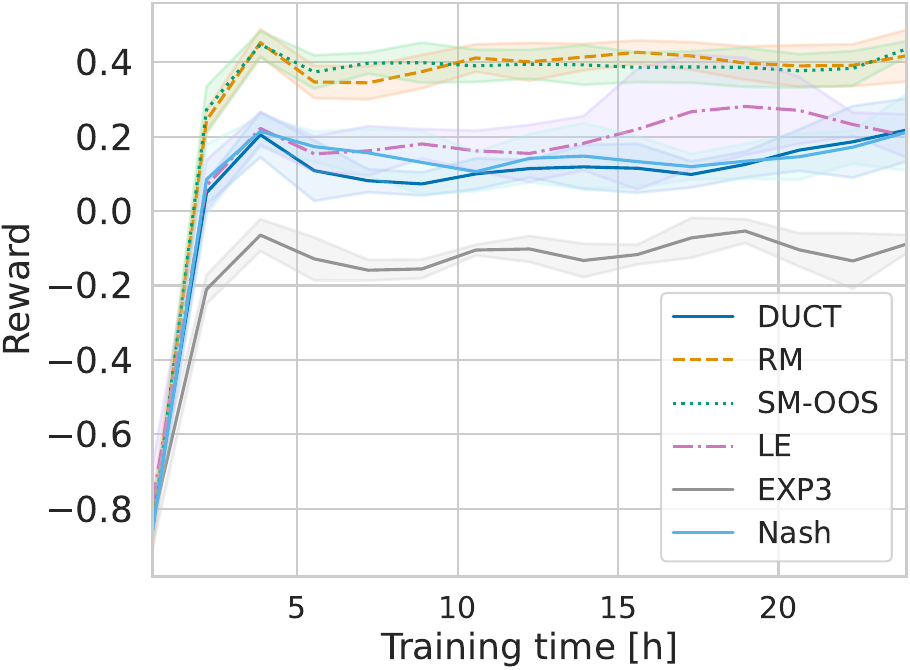}
        }
        \subfigure[Stochastic 2 Player]{%
            \includegraphics[width=\curfigwidth]{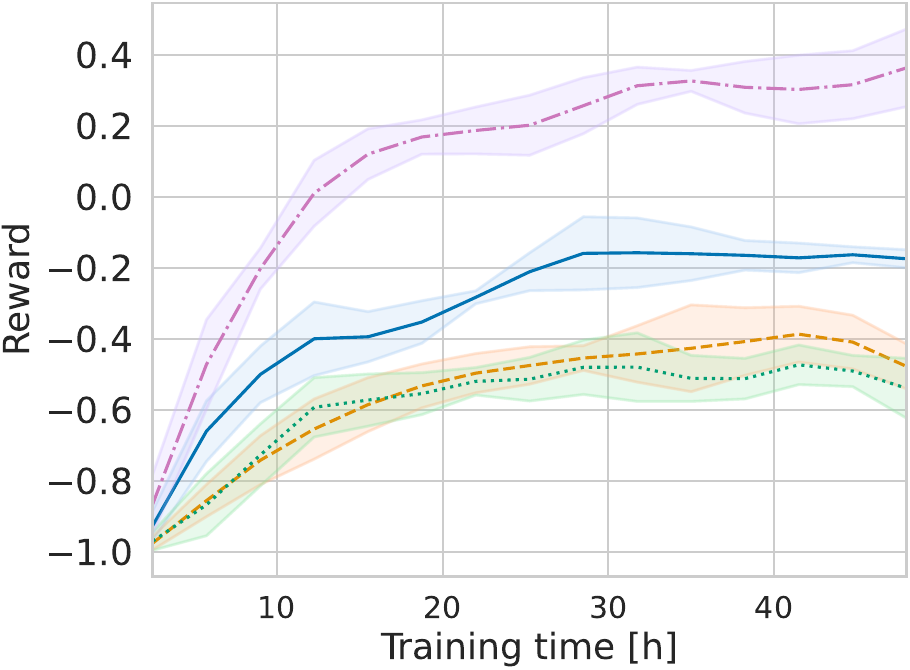}
        }
        \subfigure[Stochastic 4 Player]{%
            \includegraphics[width=\curfigwidth]{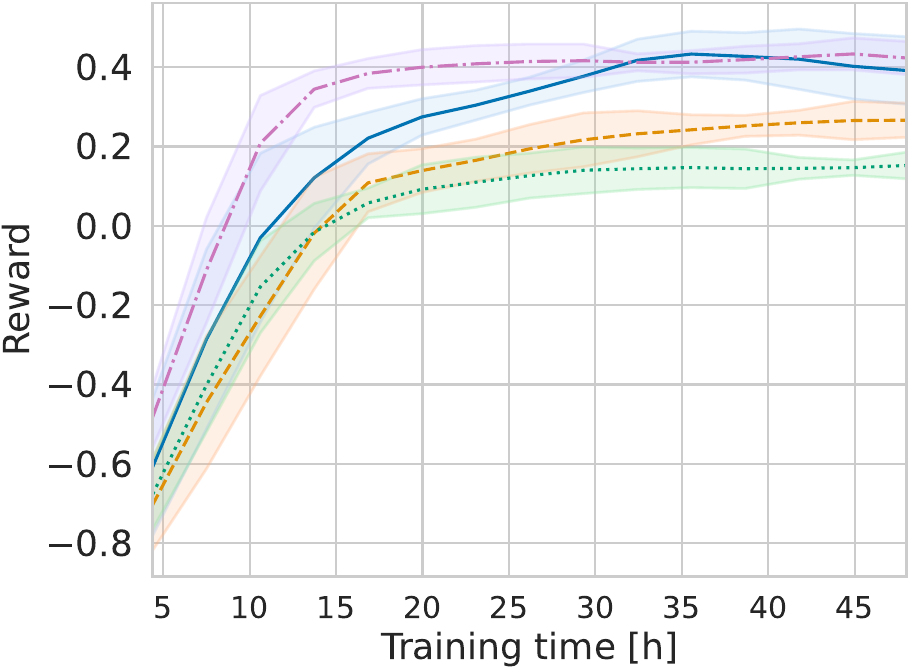}
        }
        \caption{Evaluation of AlphaZero with different Search variants in sim. games. Mean and standard deviation are computed over five seeds.}
        \label{fig:az_vs_baseline}
    \end{center}
\end{figure*}

In \cref{fig:az_vs_baseline}, the evaluation results are displayed.
Note that we excluded the worst performing variants, namely MCTS with Exp3 selection function and fixed depth search with Nash equilibrium backup, from the experiments of the stochastic modes to save resources.
In Tron, RM and SM-OOS perform best.
However, in the stochastic mode with two players, LE outperforms all other methods by a large margin.
In the mode with four players, LE achieves the best results as well, but closely followed by DUCT.
Overall, we conclude that fixed depth search with Logit equilibrium backup is the best adaptation of AlphaZero to simultaneous perfect-information games.

\section{Stability of a Unique Learning Target}
\label{subsec:nash_instability}
Even though the Logit equilibrium inherently includes an error probability, it may be preferable to a Nash equilibrium as it is more stable.
As an example, consider a situation where an agent has two possible actions.
If both actions have the same expected utility, selecting either action or any probability distribution over both actions is a Nash equilibrium.
However, if one of those actions has a marginally higher expected utility, this action is assigned a probability of one.
As a result, a Nash equilibrium may be sensitive to small approximation errors in the utility function.
Those errors, however, are always present when using neural networks as value function approximations.

\begin{figure}[ht]
\begin{center}
    \subfigure[Game Situation]{
        \includegraphics[height=0.24\columnwidth]{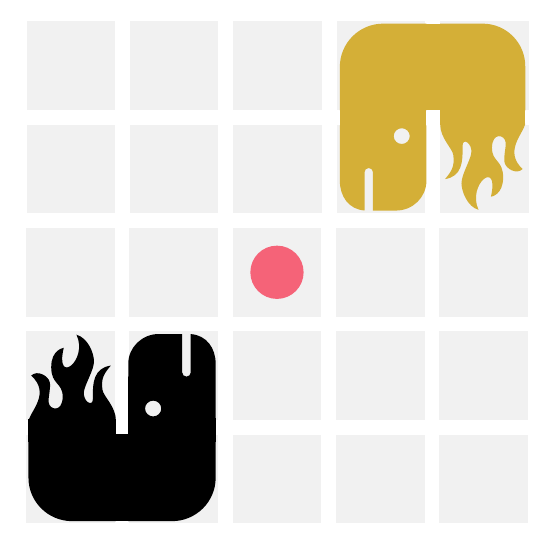}
    }
    \subfigure[Nash Equilibrium]{
        \includegraphics[height=0.24\columnwidth]{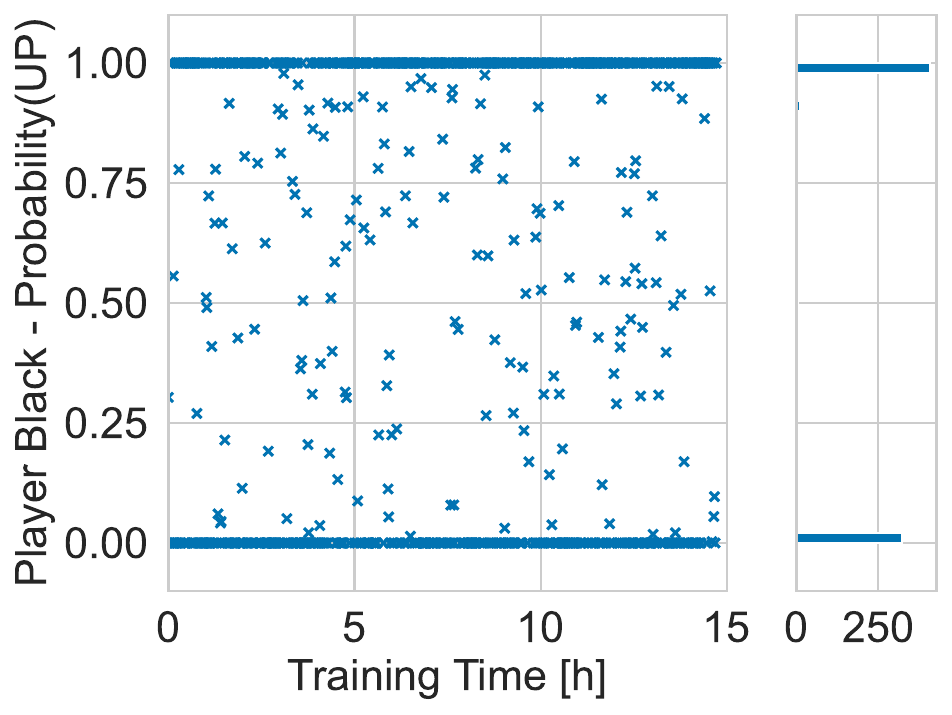}
    }
    \subfigure[Logit Equilibrium]{%
        \includegraphics[height=0.24\columnwidth]{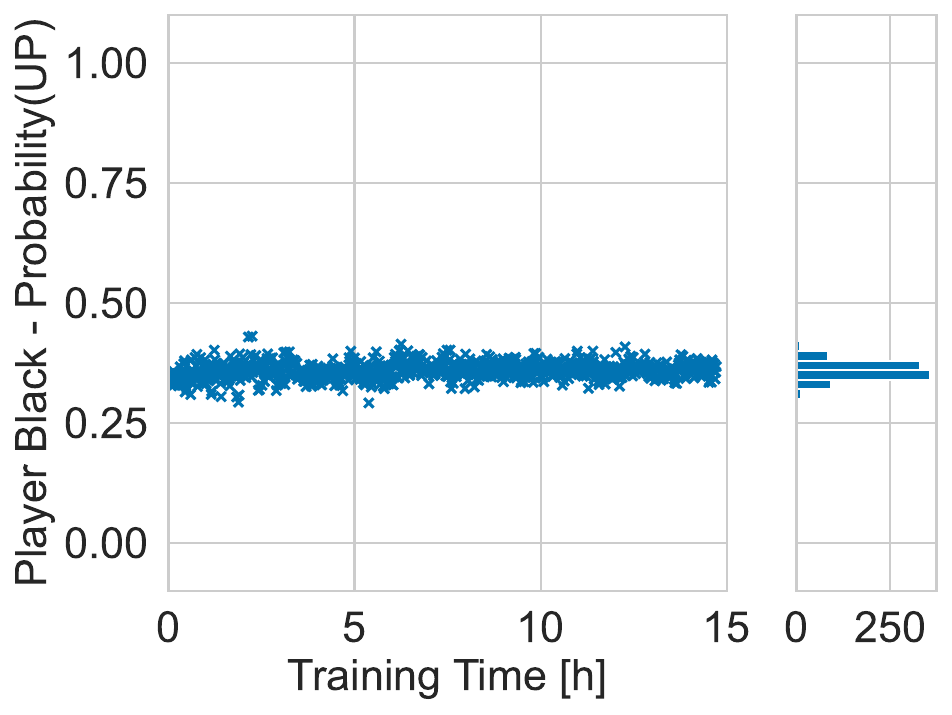}
    }
\caption{
Comparison of equilibria in a single game situation using multiple training checkpoints. The policy of the Nash- and Logit equilibrium for an NFG generated by the value function of a neural network at different training time is compared.
}
\label{fig:instability}
\end{center}
\end{figure}

We empirically investigate this property in the non-deterministic mode of Battlesnake with a board of size $5\times5$.
Using AlphaZero with DUCT, we train a neural network for 15 hours on a single RTX3090 GPU and 15 Intel Xeon Gold 6258R CPU cores.
During training, we save checkpoints of the model weights in regular time intervals.
Afterwards, we compute the Nash equilibrium and Logit equilibrium for a single fixed board position, which is displayed in Figure \ref{fig:instability}.
This board position has the property of inhibiting multiple actions with equal expected utility, i.e. the actions up and right for the black agent.
For the Logit equilibrium, we use a temperature of $\tau = 5$.
We plot the probability of the black agent playing action up in the Nash- or Logit equilibrium.
Additionally, we compute the histogram of probabilities for this action over the training time.
Evidently, the value function of the neural network changes during training, such that the probability oscillates between zero and one for the Nash equilibrium.
However, the action probabilities of a Logit equilibrium for the same neural network are much more stable.
This is due to the entropy based smoothing, which prevents a strong oscillation.
We believe this to be an indication Logit equilibria are a better target for updating a neural network than Nash equilibria.

\begin{figure}[H]
\vskip 0.2in
\begin{center}
\centerline{\includegraphics[width=0.24\columnwidth]{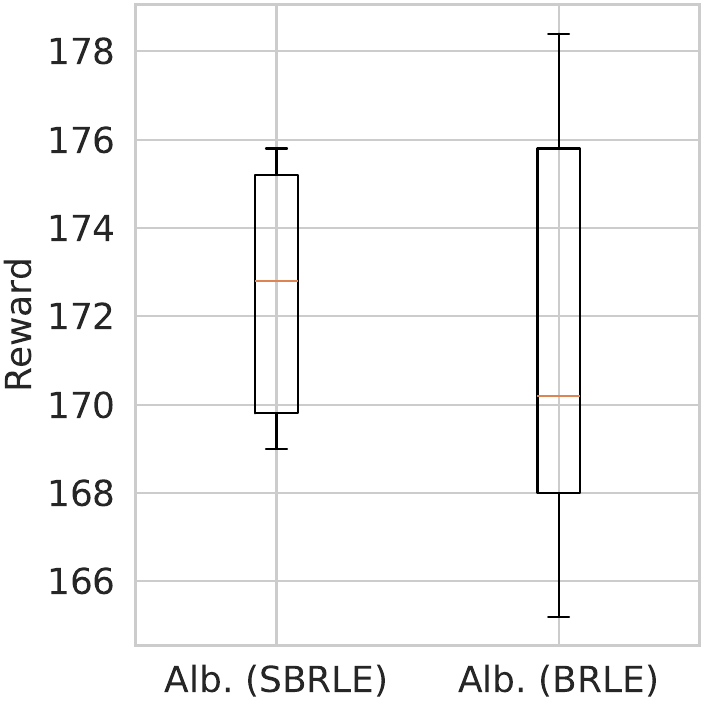}}
\caption{
Comparison of Albatross trained with SBRLE and BRLE with the behavior cloning agent in the Cramped Room layout.
Starting from the same proxy model, we trained response models using BRLE and SBRLE on five different seeds.
}
\label{fig:brle_oc_cr}
\end{center}
\end{figure}

This stability of Logit equilibria also translates to a greater stability of SBRLE than BRLE.
As an ablation study, we train the response model of Albatross using BRLE in the Cramped Room layout of Overcooked.
Event though there is only a small difference, Albatross with a SBRLE performs better than with BRLE and has less variation in the results.

\section{Maximization of Transformed Utilities with Shannon Entropy}
\label{sec:softmax_proof}
We present a mathematical proof that the softmax function maximizes smooth best responses with entropy regularization.
This is a well known fact in literature and easy to verify, but to the best of our knowledge the proof has never been published.

\begin{theorem}
\label{theorem:softmax_sbr}
The transformed utilities $\tilde{u}_i(\pi)$ defined as $\tilde{u}_i(\pi) = u_i(\pi) + \frac{1}{\tau} \psi(\pi_i)$ using Shannon entropy $\psi(\pi_i) = \sum_{a_i \in A_i} \pi_i(a_i) \log(\pi_i(a_i))$ as a smoothing function are maximized by the softmax function $\pi_i = \mathit{SBR}(\pi_{-i}, \tau) \propto \exp (\tau \, u_i(\,\cdot\,, \pi_{-i}))$.
\end{theorem}

\begin{proof}
We define the theorem as a maximization problem and solve it with the method of Lagrange multipliers.
The objective function $f$ is 
\begin{align*}
    f(\pi_i) &= u_i(\pi_i, \pi_{-i}) + \frac{1}{\tau} \psi(\pi_i) \\
           &= \Bigg[\sum_{a_i \in A_i} \pi_i(a_i)\, u_i(a_i, \pi_{-i}) \Bigg] - \frac{1}{\tau} \sum_{a_i \in A_i} \pi_i(a_i) \log(\pi_i(a_i)),
\end{align*}
where the policy $\pi_{-i}$ of the other agents is fixed.
We use $u_i(a_i, \pi_{-i})$ as a shorthand for $u_i(\hat{\pi}_{a_i}, \pi_{-i})$ with $\hat{\pi}_{a_i}$ being the policy that assigns probability one to action $a_i$ and zero otherwise.
The policy $\pi_i$ has to be a valid probability distribution, which yields the constraint
\begin{align*}
    g(\pi_i) = -1 + \sum_{a_i \in A_i} \pi_i(a_i) = 0.
\end{align*}
The resulting Lagrangian function is defined as $L(\pi_i, \lambda) = f(\pi_i) + \lambda g(\pi_i)$.
Setting the derivative of the Lagrangian function $L$ with regards to the policy of an action $\pi_i(a_i)$ to zero results in
\begin{align*}
     0 &= \frac{\partial L}{\partial\, \pi_i(a_i)} \\ 
     &= \frac{\partial f}{\partial\, \pi_i(a_i)} + \lambda \frac{\partial g}{\partial\, \pi_i(a_i)} \\
      &= u_i(a_i, \pi_{-i}) - \frac{1}{\tau} \Big(1 + \log\big(\pi_i(a_i)\big) \Big) - \lambda.
\end{align*}
We can directly solve this equation for the policy $\pi_i(a_i)$, which yields
\begin{align*}
\pi_i(a_i) &= \exp \big( \tau\, u_i(a_i) - \tau \lambda - 1\big)  \\
                       &= C \,\exp \big(\tau\, u_i(a_i, \pi_{-i}) \big).
\end{align*}
Therefore, it holds that $\pi_i(a_i) \propto \exp \big(\tau \, u_i(a_i, \pi_{-i}) \big)$, since $C = \exp( -\tau \lambda - 1)$ is a constant.
\end{proof}

\section{Concavity of Likelihood Function}
\label{subsec:concavity_proof}
We prove the concavity of the likelihood function by showing that the second derivative is always less or equal to zero.
For an alternative proof, we refer to the work of McFadden \yrcite{McFa73}.

\begin{theorem}
\label{theorem:mle_concave}
Given $K$ observations of actions $(a_i^1, \ldots, a_i^K)$ and corresponding optimal policies $(\pi_{-j}^1, \ldots, \pi_{-j}^K)$ of the other agents, the log-likelihood function $l$ for the temperature $\tau$ of agent $i$ is concave.
\end{theorem}
\begin{proof}
The second derivative of the likelihood function is:

\begin{align*}
    \frac{\partial^2 l}{\partial^2 \tau} &= - \sum_{k=1}^K \Bigg[
    \frac{
    \Big(\sum_{a_i\in A_i}  \big(u^k_i(a_i, \pi_{-i}^k)\big)^2 \exp\big(\tau\, u^k_i(a_i, \pi_{-i}^k)\big)\Big)
    \Big(\sum_{a_i\in A_i} \exp\big(\tau\, u^k_i(a_i, \pi_{-i}^k)\big)\Big)
    }{
    \Big(\sum_{a_i\in A_i} \exp\big(\tau\, u^k_i(a_i, \pi_{-i}^k)\big)\Big)^2
    } \\ &\qquad - \frac{
    \Big(\sum_{a_i\in A_i} u^k_i(a_i, \pi_{-i}^k) \exp\big(\tau\, u^k_i(a_i, \pi_{-i}^k)\big)\Big)^2
    }{
    \Big(\sum_{a_i\in A_i} \exp\big(\tau\, u^k_i(a_i, \pi_{-i}^k)\big)\Big)^2
    }
    \Bigg].
\end{align*}

We continue to show that the term in square brackets is greater equal zero for all $k \in \{1, \ldots, K\}$.
The denominator of the term is always positive and can therefore be ignored in our analysis.

\begin{align*}
    0 &\leq \Big(\sum_{a_i\in A_i}  \big(u^k_i(a_i, \pi_{-i}^k)\big)^2 \exp\big(\tau\, u^k_i(a_i, \pi_{-i}^k)\big)\Big) \Big(\sum_{a_i\in A_i} \exp\big(\tau\, u^k_i(a_i, \pi_{-i}^k)\big)\Big) \\ 
    &\qquad - \Big(\sum_{a_i\in A_i} u^k_i(a_i, \pi_{-i}^k) \exp\big(\tau\, u^k_i(a_i, \pi_{-i}^k)\big)\Big)^2 \\
    \llap{$\Leftrightarrow$ \qquad} 0 &\leq \sum_{a_i\in A_i}\sum_{b \in A_i} \big(u^k_i(a_i, \pi_{-i}^k)\big)^2 \exp\big(\tau\, u^k_i(a_i, \pi_{-i}^k)\big) \exp\big(\tau\, u^k_i(b_i, \pi_{-i}^k)\big) \\
    &\qquad - \sum_{a_i\in A_i}\sum_{b_i\in A_i} u^k_i(a_i, \pi_{-i}^k) u^k_i(b_i, \pi_{-i}^k) \exp\big(\tau\, u^k_i(a_i, \pi_{-i}^k)\big) \exp\big(\tau\, u^k_i(b_i, \pi_{-i}^k)\big) \\
    \llap{$\Leftrightarrow$ \qquad}0 &\leq \sum_{a_i\in A_i}\sum_{b_i\in A_i} \big(u^k_i(a_i, \pi_{-i}^k)\big)^2 - u^k_i(a_i, \pi_{-i}^k)\, u^k_i(b_i, \pi_{-i}^k) \\
    \llap{$\Leftrightarrow$ \qquad}0 &\leq 2 \sum_{a_i\in A_i}\sum_{b_i\in A_i} \big(u^k_i(a_i, \pi_{-i}^k)\big)^2 - u^k_i(a_i, \pi_{-i}^k)\, u^k_i(b_i, \pi_{-i}^k) \\
    \llap{$\Leftrightarrow$ \qquad}0 &\leq \sum_{a_i\in A_i}\sum_{b_i\in A_i} \big(u^k_i(a_i, \pi_{-i}^k)\big)^2 + \big(u^k_i(b_i, \pi_{-i}^k)\big)^2 - 2\, u^k_i(a_i, \pi_{-i}^k)\, u^k_i(b_i, \pi_{-i}^k) \\
    \llap{$\Leftrightarrow$ \qquad}0 &\leq \sum_{a_i\in A_i}\sum_{b_i\in A_i} \big( u^k_i(a_i, \pi_{-i}^k) - u^k_i(b_i, \pi_{-i}^k) \big)^2.
\end{align*}
As all terms of the sum are greater equal zero, the sum itself is greater equal zero.
Therefore, the second derivative of the log-likelihood function is negative.
Consequently, the log-likelihood function is concave.
\end{proof}

The line search algorithm for computing the maximum likelihood for the temperature of weak agent $j$ is displayed in \cref{alg:mle}.

\begin{algorithm}[ht]
   \caption{Maximum Likelihood Estimation of Temperature $\tau_j$}
   \label{alg:mle}
\begin{algorithmic}
    \STATE {\bfseries Input:} Observations $(a_j^1, \ldots, a_j^K)$, ground truth policies $(\pi_{-j}^1, \ldots, \pi_{-j}^K)$, valid temperature interval $[\tau_{\mathit{max}}, \tau_{\mathit{min}}]$, iteration number $m$
    \FOR{$i=0$ {\bfseries to} $m-1$}
        \STATE $\tau \gets \frac{1}{2}(\tau_{\mathit{max}} + \tau_{\mathit{min}})$
        \IF{$\frac{\partial l}{\partial \tau} > 0$}
            \STATE $\tau_{\mathit{min}} \gets \tau$
        \ELSE
            \STATE $\tau_{\mathit{max}} \gets \tau$
        \ENDIF
    \ENDFOR
    \STATE {\bfseries Output:} $\tau_j = \frac{1}{2}(\tau_{\mathit{max}} + \tau_{\mathit{min}})$
\end{algorithmic}
\end{algorithm}

The algorithm iteratively reduces the interval of possible temperature estimations by computing the sign of the gradient at the midpoint of the interval.
If the gradient is positive, the maximum likelihood estimate has to be in the upper half of the interval, otherwise the lower half.

For \cref{alg:mle}, we assume that optimal ground truth policies $(\pi_{-j}^1, \ldots, \pi_{-j}^K)$ for agents $-j$ are given.
If there exists a unique Nash equilibrium, then the optimal policies are well defined as the equilibrium policies.
However, computing a Nash equilibrium in games with many agents is difficult and most games do not have a unique Nash equilibrium.
As a solution, one can substitute the ground truth policies with the policy of the Logit equilibrium of temperature $\tau_{\mathit{max}}$ learned by the proxy model, i.e $\pi_{\theta_{P}}(o_i, \tau_{\mathit{max}})$.
Consequently, the bounded rationality of the weak agents is no longer modeled according to globally optimal play, but rather relative to the rationality acquired during training.
Additionally, this implicitly solves the equilibrium selection process, because rationality is modeled relative to the learned equilibrium during training.

\section{The Game of Overcooked}
\label{sec:overcooked}

Cooperation performance in the game of Overcooked is evaluated in five different layouts, introduced by Carrol et al. \yrcite{overcooked}.
The layouts are displayed in \cref{fig:oc_layouts}.
In the cramped room layout, agents have little space to move around and need to focus on good movement coordination. 
In the asymmetric advantage layout, one agent has a much shorter path between onion dispenser and cooking pot, while the other agent has a shorter path between serving location and cooking pot. 
Therefore, agents should distribute tasks in a way that takes into account these advantages. 
The distribution of tasks can be seen as high-level coordination, while the coordination of movement in the cramped room layout is low-level coordination. 
In the coordination ring layout, agents have a single circle of free squares available. 
They have to coordinate the direction of movement, such that they do no block each others way. 
In the forced coordination layout, only one agent has access to onion and dish dispenser, while the other agent has access to cooking pots and serving locations. 
Consequently, the first agent has to pass onions and dishes over a counter to the second agent. 
This makes training much more difficult, because the first agent only receives a reward signal if the second agent is sufficiently trained to use the tools or ingredients received. 
But, to train the second agent, the first agent needs to pass them ingredients and tools, which rarely happens through random play. 
In the counter circuit layout, the agents have a single circuit available for movement, similar to the coordination ring layout. 
However, the circle is longer, which forces the agents to pass onions over the middle counter to achieve a perfect score. 
All layouts have a fixed starting position, which are displayed in \cref{fig:oc_layouts}.
For evaluation, we swap starting positions every other episode.

\begin{figure*}[!ht]
\begin{center}
    \subfigure[Cramped Room]{%
        \includegraphics[height=0.13\textwidth]{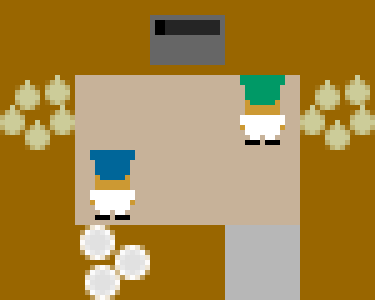}
    }
    \subfigure[Asymmetric Advantage]{%
        \includegraphics[height=0.13\textwidth]{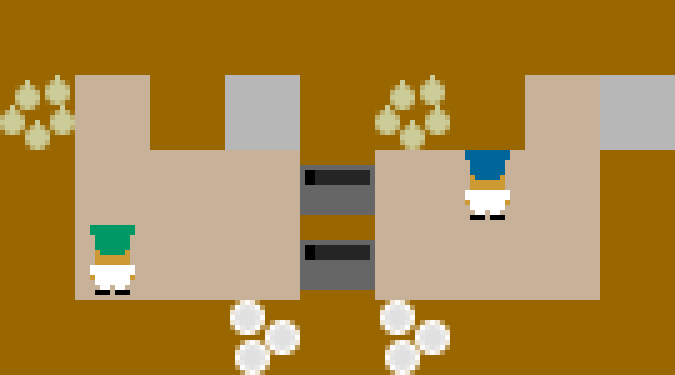}
    }
    \subfigure[Coord. Ring]{%
        \includegraphics[height=0.13\textwidth]{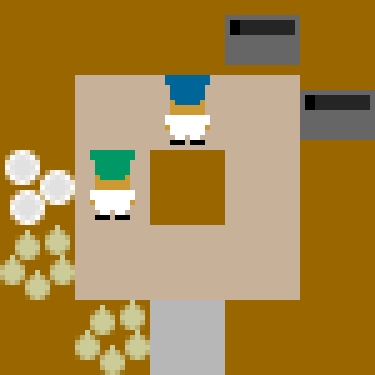}
    }
    \subfigure[Forced Coord.]{%
        \includegraphics[height=0.13\textwidth]{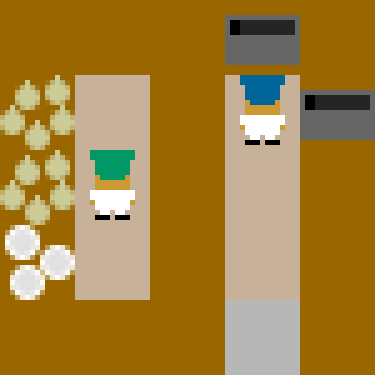}
    }
    \subfigure[Counter Circuit]{%
        \includegraphics[height=0.13\textwidth]{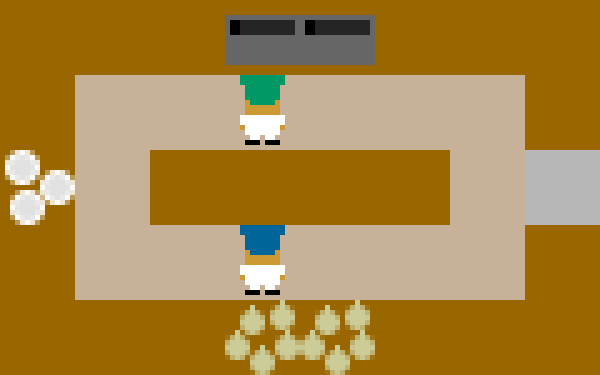}
    }
    \caption{
    Different kitchen layouts in the cooperative game of Overcooked as introduced by Carrol et al. \yrcite{overcooked}.
    The Cramped Room and Coordination Ring layout test low-level coordination since the agents should not block each other.
    In the Asymmetric Advantage and Counter Circuit layout high-level coordination, i.e. distribution of tasks, is tested.
    In the Forced Coordination layout, cooperation performance is limited by the weakest of the two agents since both need to perform distinct tasks.
    }
    \label{fig:oc_layouts}
\end{center}
\end{figure*}

In Overcooked, multiple implementations or versions with vastly different levels of difficulty exists.
For example, as of January 2024, the newest version of the implementation of Carrol et al. \yrcite{overcooked} allows agents to start cooking a soup even if not all ingredients are present in the soup.
This produces a soup, which does not yield any rewards when serving.
However, previous versions did not allow this behavior, which made solving the environment vastly easier.
Additionally, very early versions of Overcooked had an option to use a dense shaped reward based on the distance to pots or dispensers.
Newer versions still allow to specify these rewards, but this specification is internally ignored.
Since these versions developed over time, it is difficult to determine which version was used in the experiments of different researchers.
For our experiments, we opted to use the game dynamics used by Lou et al. \yrcite{pecan} as their training scheme PECAN was previously state of the art in Overcooked.
That is, we disallowed the premature cooking of a soup with missing ingredients and did not use any distance-based rewards.
To increase the training speed, we reimplemented the game of Overcooked in C++, which also provides an environment with fixed game dynamics for reproducibility in future work.

Another challenge when comparing results is the usage of a behavior cloning agent as a proxy for human behavior.
These agents are trained on a small dataset of human play.
However, we hypothesize that during recording the human actions, the time between environment steps was too short for the human contestants to react.
That is, because most actions (58\%) in the human dataset are the "do nothing" action.
Consequently, the trained behavior cloning agents exhibit a behavior of staying at a single point most of the time.
In literature, different workarounds for this issue were used.
Carrol et al. \yrcite{overcooked} used the agents as is, but implemented a dynamic into the game of Overcooked which forces agents to take a random action if they have been stuck at one position for more than three time steps.
This dynamic intends to free agents if they are stuck at a position and block each other.
In addition to this dynamic, Lou et al. \yrcite{pecan} filtered the probability distribution predicted by the trained agent, completely removing the stay action.
In our implementation, we again used the same setup as Lou et al. \yrcite{pecan}.
To facilitate reproduction of our results, we publish our trained human behavior cloning agents.

\section{Hyperparameters}
\label{sec:hyperparams}

To support reproducibility of our results, we report all hyperparameters used in our experiments.
In this section, we list common hyperparameters used across all experiments.
Hyperparameters, which differ between experiments are listed in \cref{tbl:hyperparams}.
For gradient updates, we use the AdamW optimizer \cite{adamw} with a weight decay factor of $1\mathrm{e}{-5}$.
We anneal the learning rate using a cosine decay from $1\mathrm{e}{-3}$ to $1\mathrm{e}{-6}$.
The loss function consists of a value and policy term, which are summed.
The value term is the mean squared error between prediction and target, while the policy term is the cross entropy between predicted and target policy.
For the neural network, we use the MobileNetV3 architecture \cite{Howard2019SearchingFM}.
During backup of tree search, a Logit equilibrium is solved at every non-leaf node.
To solve the Logit equilibrium, we perform 150 iterations of Stochastic Fictitious Play \cite{global_convergence} with the step size annealing of Nagurney and Zhang \yrcite{nagurney}, as discussed in \cref{subsec:le_details}.
During training, we balance exploration and exploitation using Boltzmann-Exploration \cite{CesaBianchi2017BoltzmannED} with a probability of $0.5$ or sample from the LE-policy at the root node with equal probability of $0.5$.
In all experiments, we trained our models on five seeds.
Unless specified otherwise, we used Nvidia RTX3090 GPU and 14 Intel Xeon Gold 6258R CPU for each GPU.
Those numbers were chosen to optimally saturate the compute cluster used.

\begin{table}[ht]
\caption{Hyperparameters of Albatross in Overcooked and the different game modes of Battlesnake. Only hyperparameters that differ between modes are listed here.}
\vskip 0.15in
\begin{center}
\begin{small}
\begin{sc}
\begin{tabular}{lccccc}
\toprule
Hyperparameter & Overcooked & Tron & stoch. 2p. BS & stoch. 4p. BS & 4p. Coop. Tron \\
\midrule
Buffer Size             & $5\mathrm{e}{5}$  & $1\mathrm{e}{5}$  &$2\mathrm{e}{6}$& $2\mathrm{e}{6}$ & $2\mathrm{e}{6}$\\
Number GPU (RTX 3090)   & 2  & 3 & 3 & 2 & 1 \\
Search Depth            & 1  & 3 & 3 & 1 & 1\\
Total Training Time     & 48h  & 24h & 96h & 48h & 24h \\
Discount Factor         & 0.9  & 0.99 & 0.99 & 0.99 & 0.97 \\
Batch Size              & 15000  & 2000 & 2000 & 2000 & 12000 \\
\bottomrule
\end{tabular}
\end{sc}
\end{small}
\end{center}
\vskip -0.1in
\label{tbl:hyperparams}
\end{table}

\section{Additional Analysis of Albatross}
We perform additional experiments to analyze the behavior and hyperparameters of Albatross.
Firstly, we test the effect of different temperature distributions during training.
Then, we perform additional analysis on the effect of the temperature when cooperating with different opponents.

\subsection{Temperature Training Distribution}
\label{subsec:temp_dist}
During training, we randomly sample a temperature from a fixed interval $[\tau_{\mathit{min}}, \tau_{\mathit{max}}]$.
This sampling has the effect that the entire range of temperatures are learned, as well as all combinations of different temperatures.
However, when using a neural network for function approximation, its predictions may degrade at the boundary of the input space, i.e. the interval of temperatures.
Therefore, we bias the random sampling toward the upper and lower boundary of the interval of valid temperatures.
Specifically, we use a cosine function to sample more often near the boundaries.

\begin{figure*}[!ht]
    \begin{center}
        \subfigure[Policy Function]{%
            \includegraphics[width=0.3\textwidth]{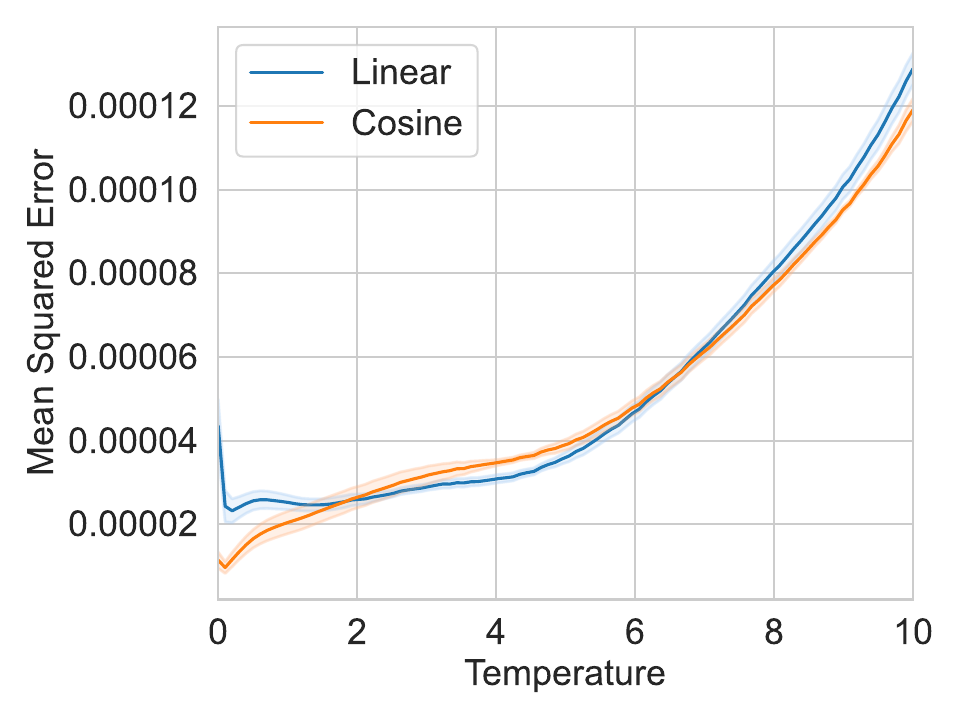}
        }
        \subfigure[Value Function]{%
            \includegraphics[width=0.3\textwidth]{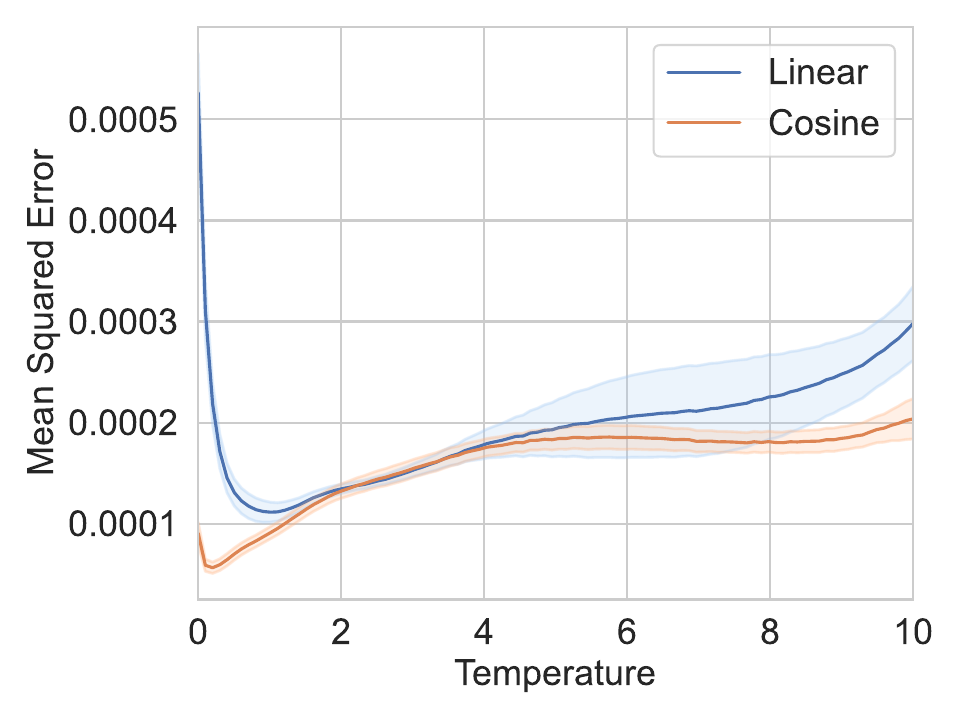}
        }
        \caption{
        Mean squared error of the policy and value function of the proxy model in deterministic 2-player Battlesnake with a board size of $5\times5$.
        Since the game is zero-sum, it has a unique Logit equilibrium for all temperatures, which can be computed exactly due to the small state space.
        }
        \label{fig:d5_proxy_error}
    \end{center}
\end{figure*}

We compare the cosine based sampling against the uniform (linear) sampling in the deterministic zero-sum Battlesnake game mode with a board of size $5\times5$.
Using this smaller board size, it is possible to compute the ground truth Logit equilibrium.
This ground truth Logit equilibrium is unique for all temperatures, since the game is a two-player zero-sum game.
In \cref{fig:d5_proxy_error}, the mean squared error of the policy and value function is displayed.
For most parts of the temperature input space, the cosine sampling produces a lower error than the linear sampling.
Especially for small temperatures near zero, the linear sampling produces larger error spikes than cosine sampling.
Consequently, we use the cosine function in all of our experiments.

\subsection{Temperature Estimation}

Following the experiments presented in \cref{sec:eval}, we present the full experimental data on all layouts of Overcooked and in all game modes of Battlesnake.

\begin{figure*}[!ht]
\begin{center}
    \subfigure[Cramped Room]{%
        \includegraphics[width=0.19\textwidth]{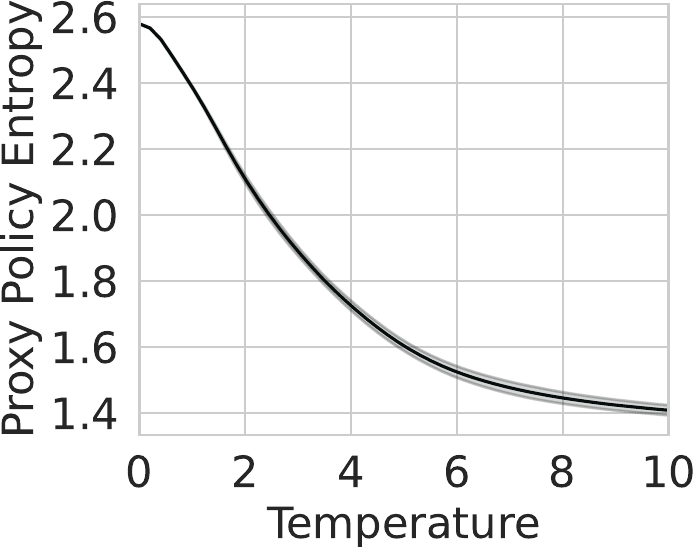}
    }
    \subfigure[Asym. Advantage]{%
        \includegraphics[width=0.19\textwidth]{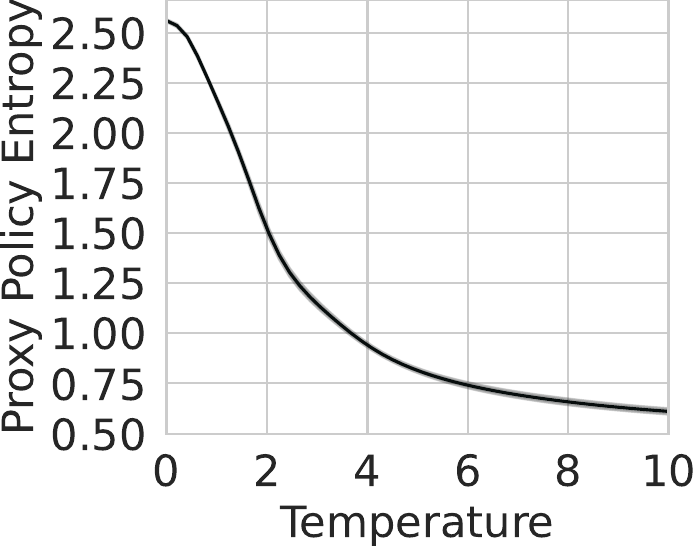}
    }
    \subfigure[Coord. Ring]{%
        \includegraphics[width=0.19\textwidth]{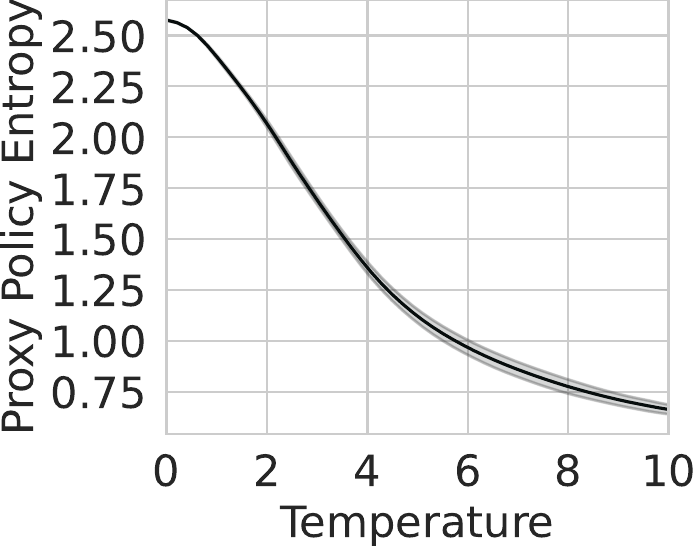}
    }
    \subfigure[Forced Coord.]{%
        \includegraphics[width=0.19\textwidth]{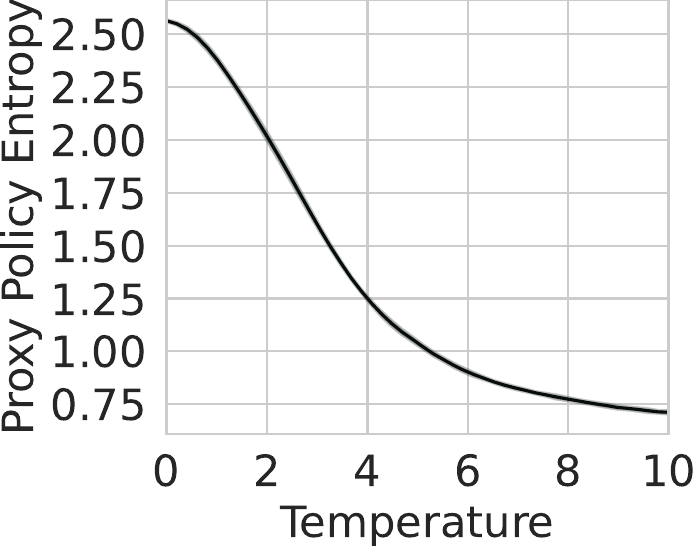}
    }
    \subfigure[Counter Circuit]{%
        \includegraphics[width=0.19\textwidth]{img/proxy_entropy_notitle_cc.pdf}
    }
    \caption{
    Entropy of the proxy model policy at different temperatures for all five layouts in the game of Overcooked.
    }
    \label{fig:all_proxy_entropy}
\end{center}
\end{figure*}

In \cref{fig:all_proxy_entropy}, the entropy of the proxy policy at different temperatures in displayed for the different layouts in Overcooked.
In all layouts, for temperature $\tau = 0$, the entropy starts at about $\log_2(6) = 2.58$, which is the entropy of a uniformly random policy with six discrete actions.
The entropy decreases at higher temperatures and converges to the entropy of the learned Logit equilibrium at highest temperature $\tau = 10$, which may be different for every layout.

\begin{figure*}[!ht]
\begin{center}
    \subfigure[Cramped Room]{%
        \includegraphics[width=0.19\textwidth]{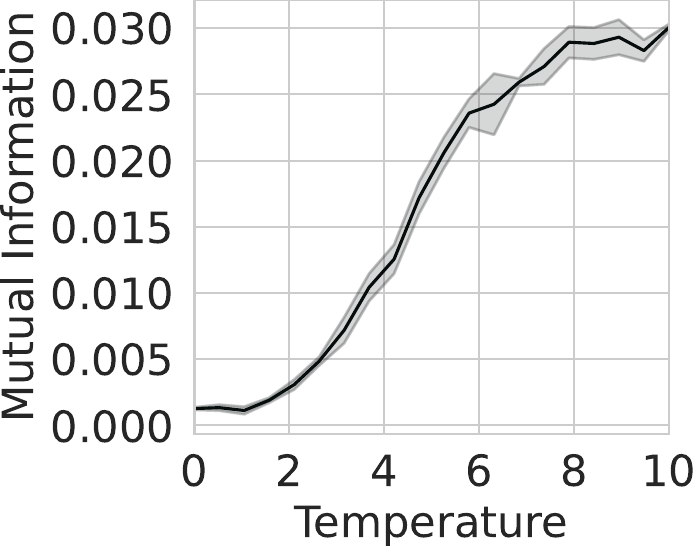}
    }
    \subfigure[Asym. Advantage]{%
        \includegraphics[width=0.19\textwidth]{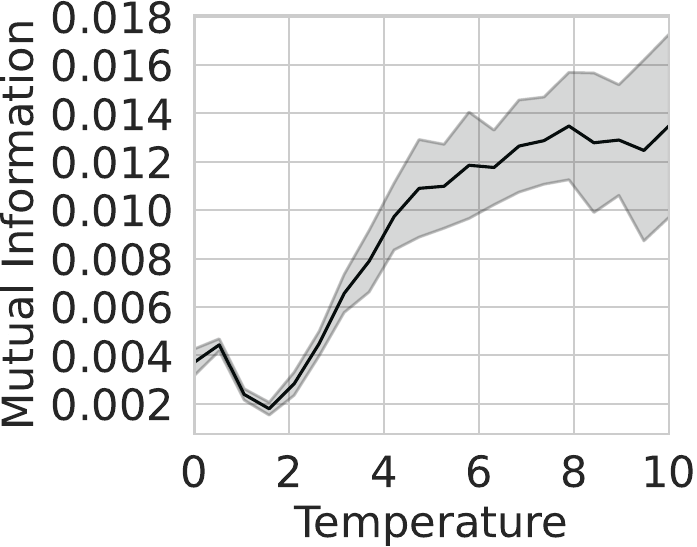}
    }
    \subfigure[Coord. Ring]{%
        \includegraphics[width=0.19\textwidth]{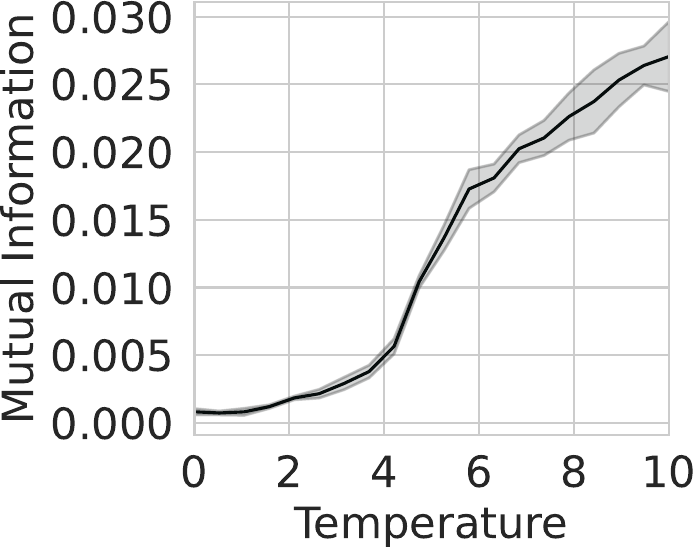}
    }
    \subfigure[Forced Coord.]{%
        \includegraphics[width=0.19\textwidth]{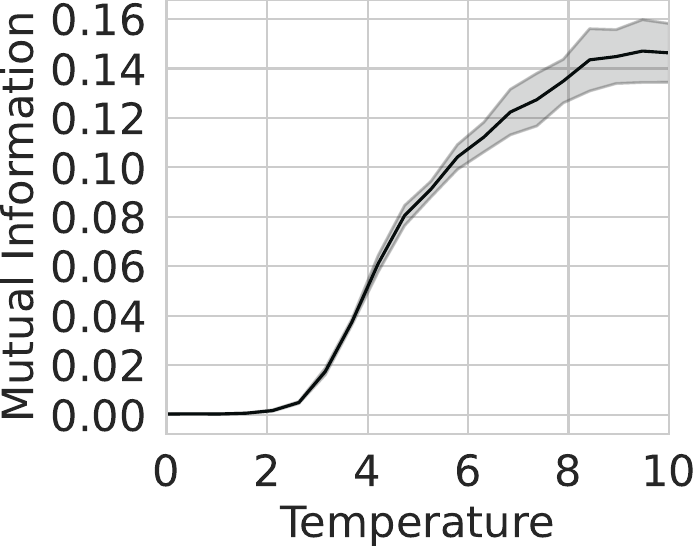}
    }
    \subfigure[Counter Circuit]{%
        \includegraphics[width=0.19\textwidth]{img/mi_resp_proxy_notitle_cc.pdf}
    }
    \caption{
    Mutual information between Albatross response and proxy model at different temperatures for all five layouts in the game of Overcooked.
    }
    \label{fig:mi_alb_proxy}
\end{center}
\end{figure*}

In \cref{fig:mi_alb_proxy}, the mutual information between proxy and response model at different temperatures are displayed.
For a temperature of $\tau = 0$, the mutual information in all game modes is close to zero, indicating that the agents do not cooperate.
This is expected, since the proxy model at this temperature only plays a uniformly random policy.
At high temperatures, the mutual information between the agents rises as they start to cooperate.
The mutual information can also be seen as a measure of trust from the response model, that the proxy model fulfills their expected task in cooperation.

\begin{figure*}[!ht]
\begin{center}
    \subfigure[Cramped Room]{%
        \includegraphics[width=0.19\textwidth]{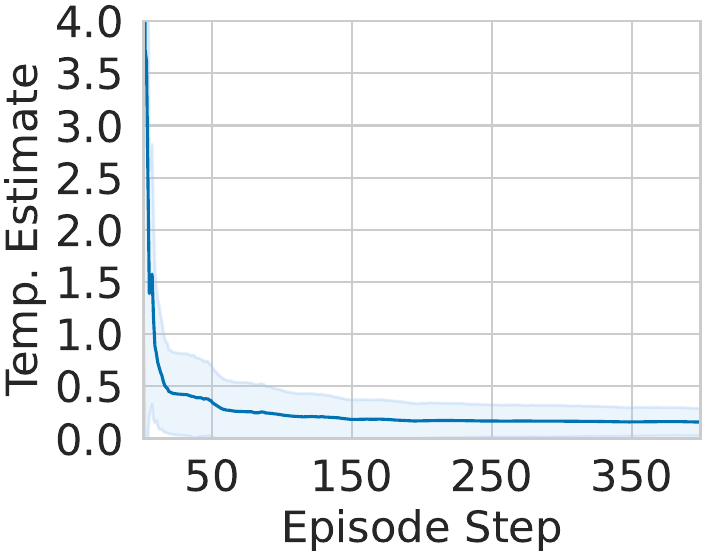}
    }
    \subfigure[Asym. Advantage]{%
        \includegraphics[width=0.19\textwidth]{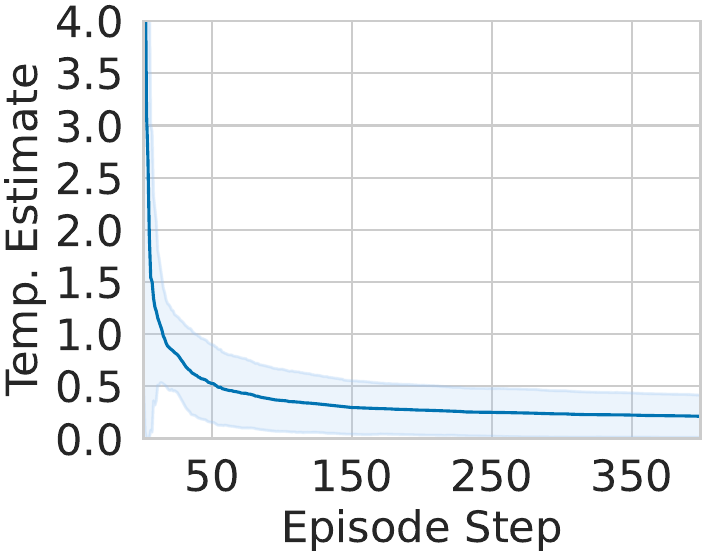}
    }
    \subfigure[Coord. Ring]{%
        \includegraphics[width=0.19\textwidth]{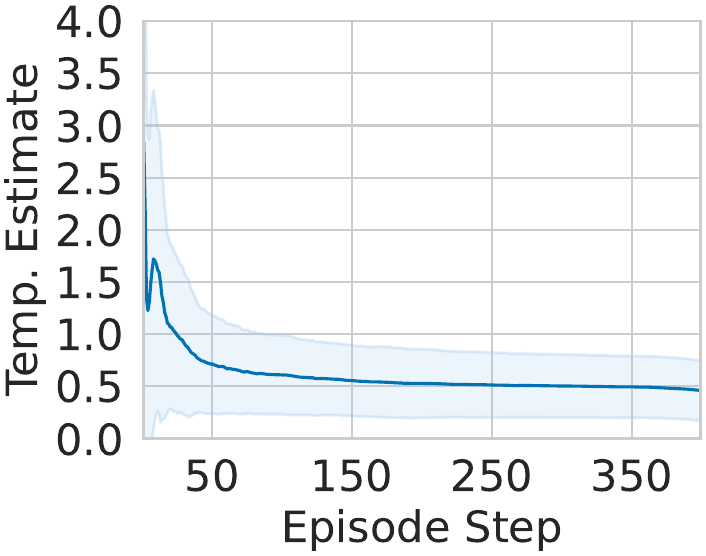}
    }
    \subfigure[Forced Coord.]{%
        \includegraphics[width=0.19\textwidth]{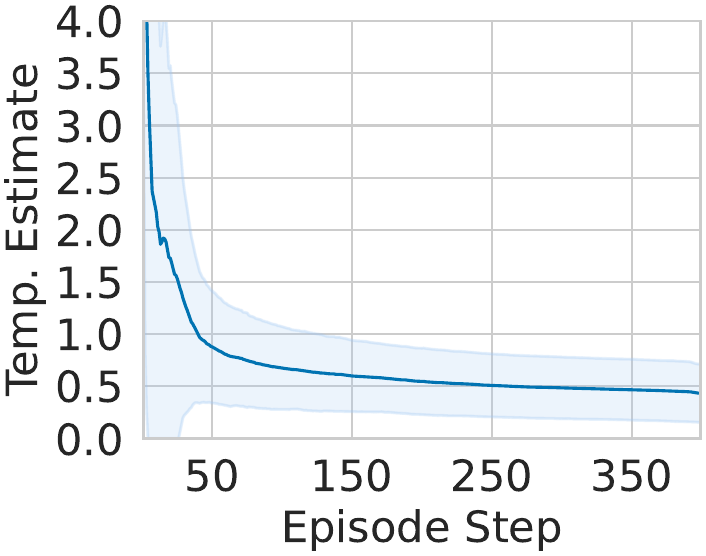}
    }
    \subfigure[Counter Circuit]{%
        \includegraphics[width=0.19\textwidth]{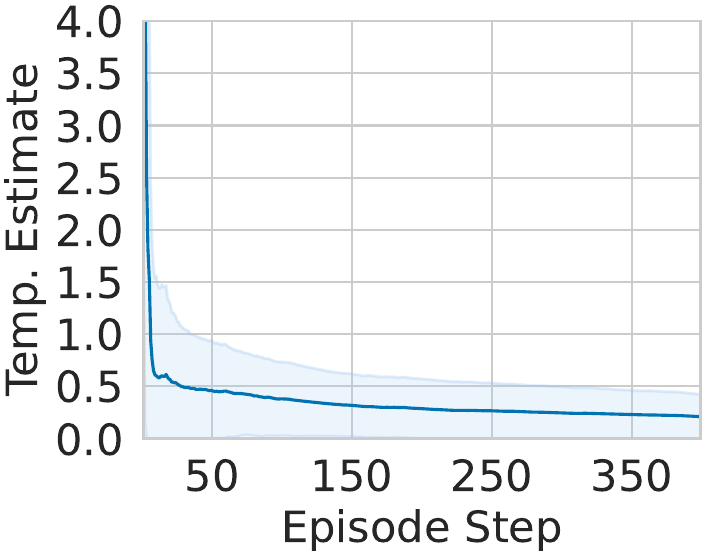}
    }
    \caption{
    Temperature estimation of Albatross for the human behavior cloning agent per time step in the different layouts of Overcooked.
    }
    \label{fig:bc_strength_all}
\end{center}
\end{figure*}

In \cref{fig:bc_strength_all}, we display the result of the Maximum Likelihood estimation of Albatross when playing the human behavior cloning agent.
For details regarding the MLE, we refer to \cref{alg:mle}.
In all layouts, after only a few episode steps, the temperature estimation converges toward the true temperature.
This indicates, that it is possible to estimate the rationality of an agent within a single episode of Overcooked.
However, as discussed in \cref{sec:limits}, we see that the MLE needs between 10 and 30 episode steps to reach a stable estimation, which may not be feasible in very short games.

\begin{figure*}[!ht]
\begin{center}
    \subfigure[Cramped Room]{%
        \includegraphics[width=0.19\textwidth]{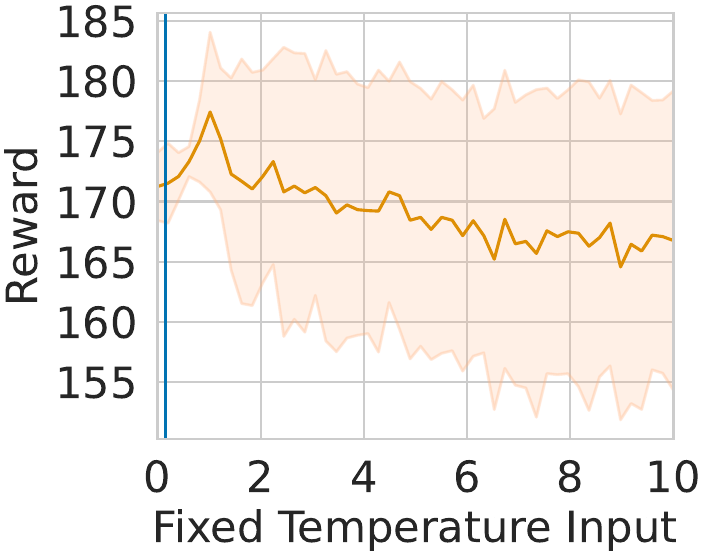}
    }
    \subfigure[Asym. Advantage]{%
        \includegraphics[width=0.19\textwidth]{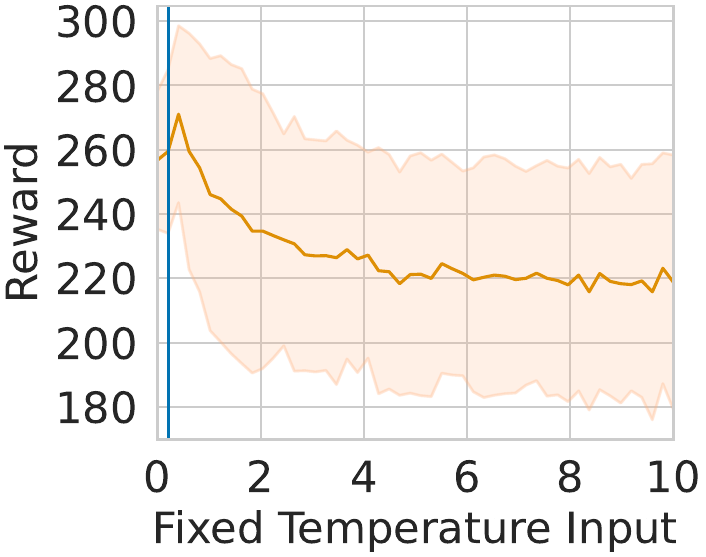}
    }
    \subfigure[Coord. Ring]{%
        \includegraphics[width=0.19\textwidth]{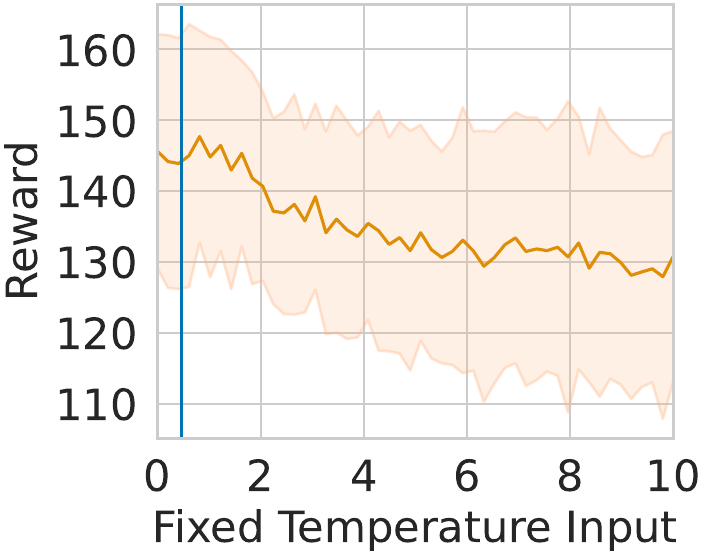}
    }
    \subfigure[Forced Coord.]{%
        \includegraphics[width=0.19\textwidth]{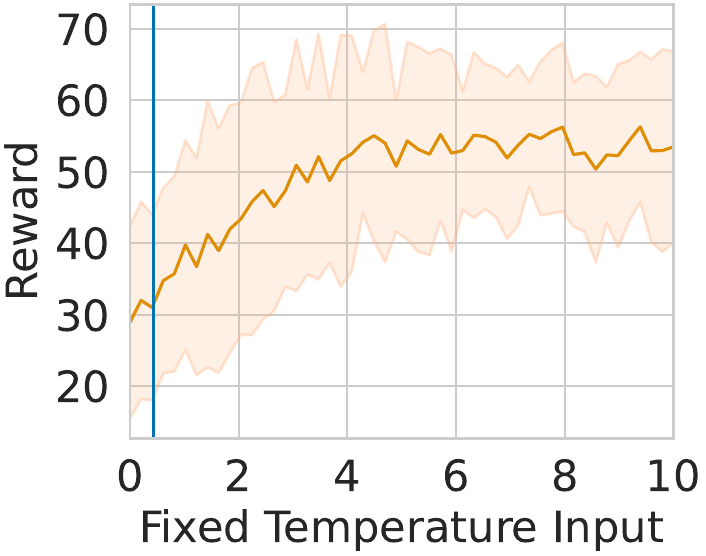}
    }
    \subfigure[Counter Circuit]{%
        \includegraphics[width=0.19\textwidth]{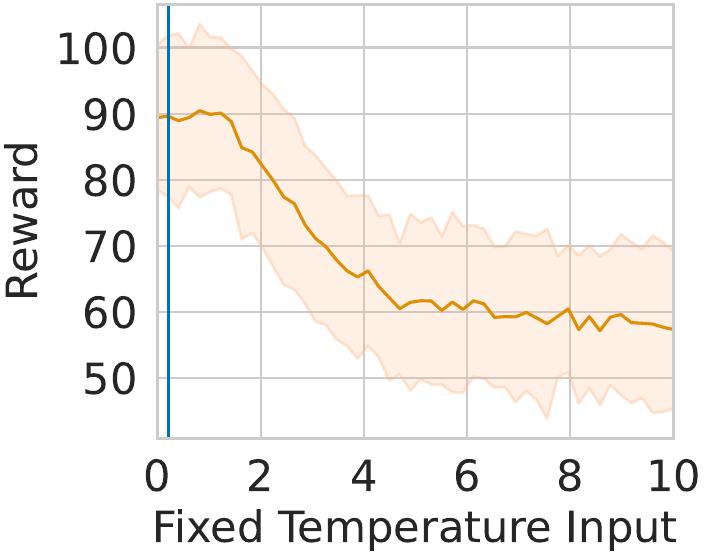}
    }
    \caption{
    Analysis of the robustness of Albatross to a wrong temperature estimation in all layouts of Overcooked.
    }
    \label{fig:albfix_all}
\end{center}
\end{figure*}

In \cref{fig:albfix_all}, we test the robustness of Albatross to wrong temperature estimations in the different layouts of Overcooked.
To this end, we evaluate Albatross with a fixed temperature input against the human behavior cloning agent.
For all layouts, except Forced Coordination, the highest reward is achieved when using a temperature close to the true temperature.
In those layouts, the achieved reward drops when using temperature estimations far from the ground truth temperature.
In the Forced Coordination layout, the training process at low temperatures is much more difficult, because only sparse rewards resulting from the nearly random proxy policy are observed.
Therefore, Albatross performs better at higher temperature estimations.
Note that the reported reward does not necessarily align perfectly with the reward achieved by MLE, because in some situations the adaptive estimation of MLE may be advantageous.
For example, the human behavior cloning agent performs better in some parts of the state space and as a consequence the best temperature estimation differs between episodes.

\begin{figure*}[!ht]
\begin{center}
    \subfigure[Tron]{%
        \includegraphics[width=0.25\textwidth]{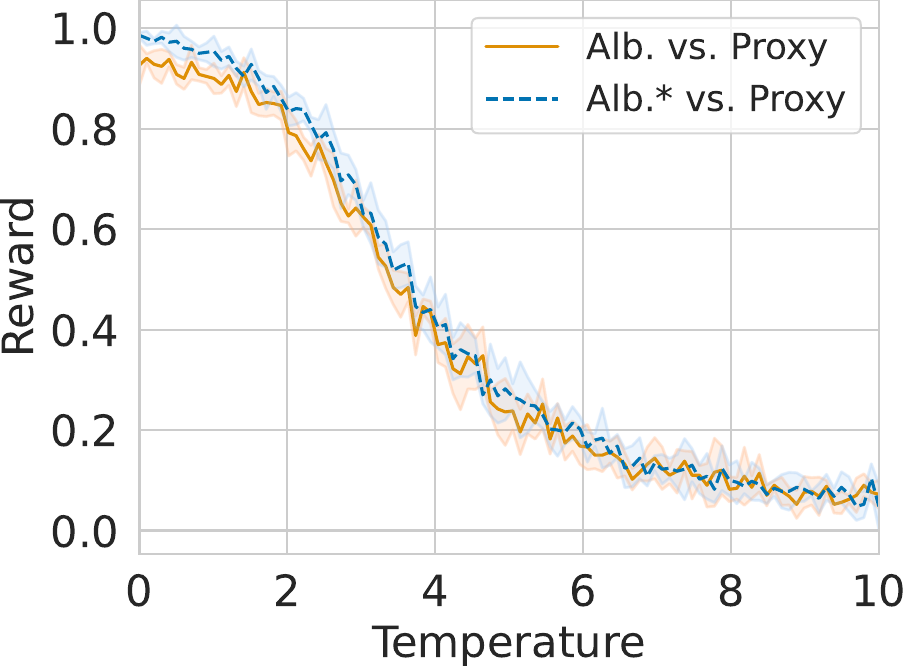}
    }
    \subfigure[Stochastic 2 Player]{%
        \includegraphics[width=0.25\textwidth]{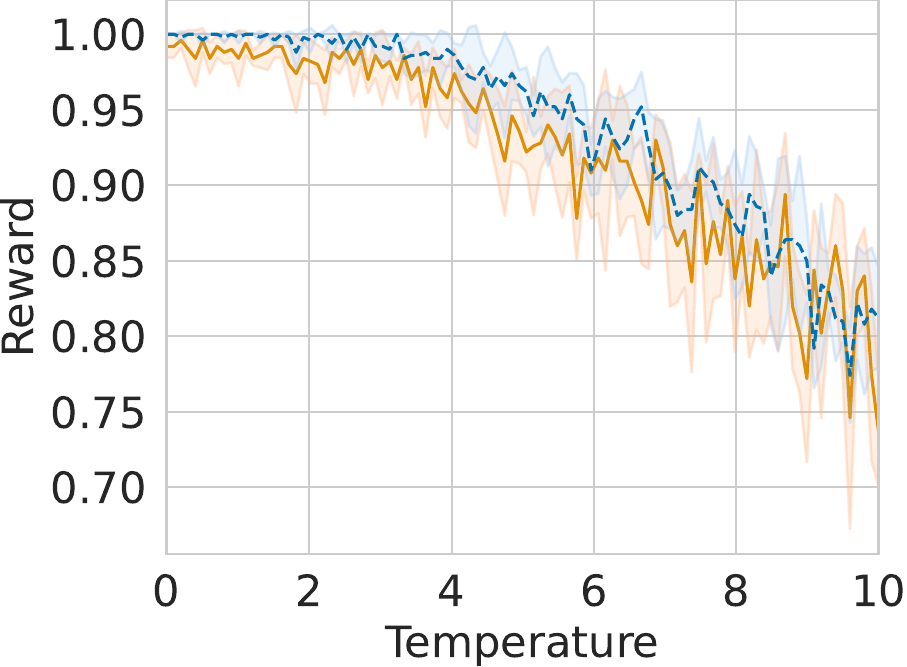}
    }
    \subfigure[Stochastic 4 Player]{%
        \includegraphics[width=0.25\textwidth]{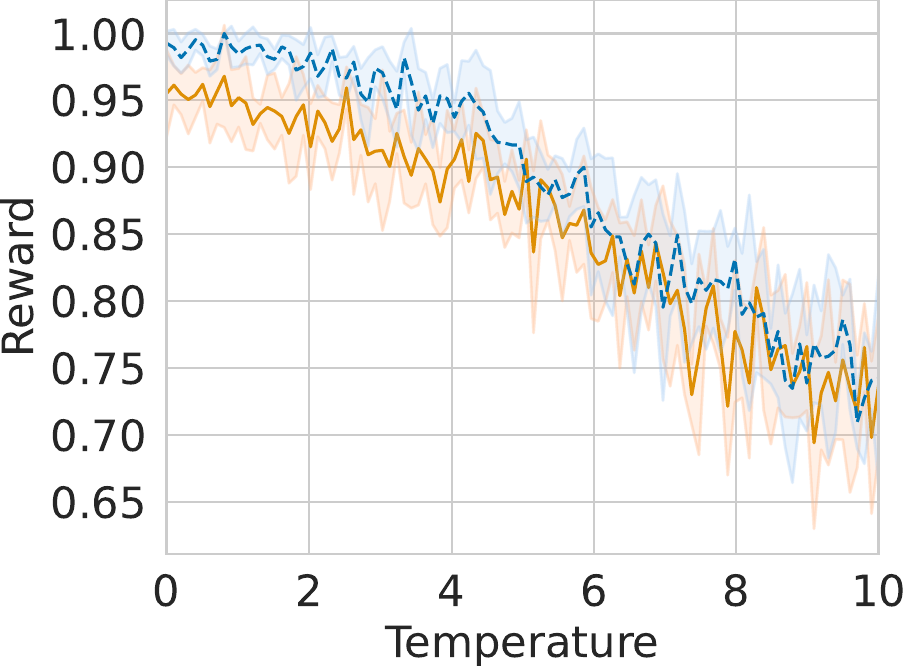}
    }
    \caption{
    Evaluation of Albatross versus the proxy model at different temperatures.
    }
    \label{fig:bs_proxy_temps}
\end{center}
\end{figure*}

In \cref{fig:bs_proxy_temps}, we test the performance of Albatross against the proxy model in the game of Battlesnake at different temperatures, akin to \cref{fig:proxy_temps}.
At low temperatures, the proxy model plays nearly random and consequently the response model wins nearly all games.
In contrast, at high temperatures the proxy model plays rationally and the reward achieved by the response model drops.
We denote Albatross with fixed ground truth temperature input as Albatross*, and compare its performance to standard Albatross with MLE.
In the game of Battlesnake, the reward difference between MLE and given ground truth is small, indicating that MLE is able to accurately estimate rationality.

\subsection{Behavior of Albatross in Repeated Matrix Games}
The behavior of Albatross and (S)BRLE as well as the temperature estimation can be explained using simple repeated matrix games.
For simplicity we use the BRLE, because SBRLE would require the computation of a softmax that only obfuscates the explanation.
Depending on the response temperature used, the results may vary, but the main points of this explanations are also applicable to the SBRLE.
Let's suppose we have the matrix of a Normal form game, which is given in \cref{tbl:nfg-example}.

\begin{table}[ht]
\caption{Example Normal form game.}
\label{tbl:nfg-example}
\vskip 0.15in
\begin{center}
\begin{small}
\begin{sc}
\begin{tabular}{lccccc}
\toprule
 & P2, A1 & P2, A2 \\
\midrule
P1, A1  & 4, 4  & 0, 0 \\
P1, A2   & 1, 1 & 2, 2 \\
\bottomrule
\end{tabular}
\end{sc}
\end{small}
\end{center}
\vskip -0.1in
\end{table}

Pareto optimal play would be the Nash equilibrium at (A1, A1) with utility of 4 for both players. 
Suppose P1 plays according to a BRLE and estimates the temperature of P2, who plays according to some specific pattern. 
For example, if P2 always plays A1, then their rationality is estimated as positive infinity (which we clip to some maximum temperature), since they always play the best action. 
Hence, the best response to a Logit equilibrium of maximum temperature would be A1 for P1. 
This dynamic is illustrated in the left part of \cref{tbl:behavior_optimal}.


\begin{table}[ht]
\caption{Behavior of Albatross with Pareto optimal player (left) and Tit-for-Tat agent (right)}
\label{tbl:behavior_optimal}
\vskip 0.15in
\begin{center}
\begin{small}
\begin{sc}
\begin{tabular}{lcccc|ccccccc}
\toprule
Time Step & 1 & 2 & 3 & ... & 1 & 2 & 3 & 4 & 5 & ... & k\\
\midrule
P1, Action  & 1 & 1 & 1 & ...& 1 & 2 & 2 & 2 & 2 & ... & 2\\
P2, Action   & 1 & 1 & 1 & ...& 2 & 1 & 2 & 2 & 2 & ... & 2\\
Temp. Estimate   & - & Max & Max & ... & - & Min & 0 & 0.17 & 0.27 & ...  & Max\\
Utility & 4 & 4 & 4 & ...& 0 & 1 & 4 & 4 & 4 & ... & 4 \\
\bottomrule
\end{tabular}
\end{sc}
\end{small}
\end{center}
\vskip -0.1in
\end{table}

In contrast, if P2 always plays A2, then the temperature estimate is negative infinite (clipped to some minimum temperature). 
Therefore, both players would continue to play A2, which is the other pure Nash equilibrium (A2, A2) of the game.

Now suppose P2 plays an adaptive "Tit for Tat" strategy, which always copies the last action of P1. 
Depending on the initial actions chosen by both players, different behaviors emerge. We use an optimistic initialization for BRLE, choosing A1 (corresponding to high rationality estimation of P2) at first. 
If P2 also plays A1 at first, then both players will continue to play A1 forever (see first table). 
In contrast, if P2 starts with A2, then the initial rationality estimate would be the minimum clipped temperature. 
Then, P2 copies A1 from the first turn. 
Having played A1 and A2 both once, the rationality estimate is zero as it corresponds to uniform random play. 
Then, both agents continue to play A2 and the temperature estimate rises to the maximum clipped temperature. 
This behavior can be seen in the right part of \cref{tbl:behavior_optimal}.
We should note, that Albatross was not designed to play adaptive agents. 
Instead, we make the standard assumption of opponent modelling that the other agents play a static strategy only depending on the environment state. 



\end{document}